\RequirePackage{fix-cm}
\documentclass[smallextended]{svjour3}       

\smartqed

\usepackage[T1]{fontenc}
\usepackage[utf8]{inputenc}
\usepackage{amsmath,amssymb}
\usepackage{accents,dsfont}
\usepackage{enumitem}
\usepackage{graphicx,subcaption}
\usepackage{xcolor}
\usepackage{cite}
\usepackage[hidelinks]{hyperref}
\usepackage{listings}

\newtheorem{assum}{Assumption}
\newtheorem{defn}{Definition}
\newtheorem{thm}{Theorem}
\newtheorem{prop}{Proposition}
\newtheorem{cor}{Corollary}
\newtheorem{lem}{Lemma}
\newtheorem{rem}{Remark}

\DeclareMathOperator{\ind}{\mathds{1}}
\DeclareMathOperator{\supp}{\mathrm{supp}}

\newcommand{\bzero}{\boldsymbol{0}}
\newcommand{\bone}{\boldsymbol{1}}

\newcommand{\bh}{\boldsymbol{h}}
\newcommand{\bF}{\boldsymbol{F}}
\newcommand{\bH}{\boldsymbol{H}}
\newcommand{\bk}{\boldsymbol{k}}
\newcommand{\bN}{\boldsymbol{N}}
\newcommand{\bs}{\boldsymbol{s}}
\newcommand{\bu}{\boldsymbol{u}}
\newcommand{\bv}{\boldsymbol{v}}
\newcommand{\bV}{\boldsymbol{V}}
\newcommand{\balpha}{\boldsymbol{\alpha}}
\newcommand{\blambda}{\boldsymbol{\lambda}}
\newcommand{\bLambda}{\boldsymbol{\Lambda}}
\newcommand{\bnu}{\boldsymbol{\nu}}

\newcommand{\cD}{\mathcal{D}}
\newcommand{\cE}{\mathcal{E}}
\newcommand{\cF}{\mathcal{F}}
\newcommand{\cI}{\mathcal{I}}
\newcommand{\cL}{\mathcal{L}}
\newcommand{\cJ}{\mathcal{J}}
\newcommand{\cM}{\mathcal{M}}
\newcommand{\cN}{\mathcal{N}}
\newcommand{\cR}{\mathcal{R}}
\newcommand{\cS}{\mathcal{S}}

\newcommand{\cZ}{\mathcal{Z}}

\newcommand{\bcD}{\boldsymbol{\cD}}
\newcommand{\bcS}{\boldsymbol{\cS}}

\newcommand{\bbN}{\mathbb{N}}
\newcommand{\bbP}{\mathbb{P}}
\newcommand{\bbR}{\mathbb{R}}

\newcommand{\bbZ}{\mathbb{Z}}

\newcommand{\dd}{\mathrm{d}}
\newcommand{\hqed}{\hfill~\qed}
\newcommand{\sqed}{\hfill~$\star$}
\newcommand{\T}{^\top\hspace{-1mm}}
\newcommand{\floor}[1]{\lfloor #1 \rfloor}
\newcommand{\ceil}[1]{\lceil #1 \rceil}
\newcommand{\ubar}[1]{\underaccent{\bar}{#1}}
\newcommand{\utilde}[1]{\underaccent{\tilde}{#1}}
\newcommand{\im}{i=1,\dots,m}
\newcommand{\sumim}{\sum_{i=1}^m}
\newcommand{\prodim}{\prod_{i=1}^m}
\newcommand{\ms}{\cD_1,\dots,\cD_m}
\newcommand{\Nlb}{\ubar{N}}
\newcommand{\Nub}{\bar{N}}

\newcommand{\bbPN}{\bbP^{N}}
\newcommand{\bbPbN}{\bbP^{|\bN|}}
\newcommand{\scen}[2]{\delta_{#1}^{(#2)}}
\newcommand{\sceni}[1]{\scen{i}{#1}}
\newcommand{\scenij}{\sceni{j}}

\newcommand{\zstar}{z_N^\star}
\newcommand{\sstar}{s_N^\star}
\newcommand{\zsN}{z_{\bN}^{\star}}
\newcommand{\sstari}[1][i]{s_{#1,\bN}^\star}
\newcommand{\bsstar}{\bs^{\star}_{\bN}}
\newcommand{\Ski}[1]{\cS_{#1,k_{#1}}}
\newcommand{\tSki}[1]{\widetilde{\cS}_{#1,k_{#1}}}
\newcommand{\bSk}{\bcS_{\bk}}
\newcommand{\tcD}{\widetilde{\cD}}

\newcommand{\knorm}{K}				
\newcommand{\ktot}{k_{\text{tot}}}	
\newcommand{\bkub}{\bar{\bk}}		
\newcommand{\bklb}{\ubar{\bk}}		

\newcommand{\tbub}{\bar{t}}
\newcommand{\tblb}{\ubar{t}}
\newcommand{\tktot}{\tbub_{\ktot}}

\newcommand{\epstubs}{\tilde{\epsilon}}
\newcommand{\epstlbs}{\utilde{\epsilon}}
\newcommand{\epstub}{\tilde{\varepsilon}}
\newcommand{\epstlb}{\utilde{\varepsilon}}
\newcommand{\epsub}{\bar{\varepsilon}}
\newcommand{\epslb}{\ubar{\varepsilon}}
\newcommand{\epstubap}{\hat{\epstub}}
\newcommand{\epsubap}{\hat{\epsub}}
\newcommand{\epslbap}{\hat{\epslb}}
\newcommand{\eps}{\varepsilon}

\definecolor{lcolor}{RGB}{135 0 223}
\definecolor{mygray}{RGB}{160 160 160}

\begin{document}

\title{Scenario theory for multi-criteria data-driven decision making%
\thanks{This work has been supported by the PRIN PNRR project P2022NB77E ``A data-driven cooperative framework for the management of distributed energy and water resources'' (CUP: D53D23016100001) funded by the NextGeneration EU program (Mission 4, Component 2, Investment 1.1), by the PRIN 2022 project ``The Scenario Approach for Control and Non-Convex Design'' (project number D53D23001450006), by the FAIR (Future Artificial Intelligence Research) project, funded by the NextGenerationEU program within the PNRR-PE-AI scheme (M4C2, Investment 1.3, Line on Artificial Intelligence), and by the Italian Ministry of Enterprises and Made in Italy in the framework of the project 4DDS (4D Drone Swarms) under grant no. F/310097/01-04/X56.}
}


\author{Simone~Garatti \and Alessadro~Falsone \and Lucrezia~Manieri \and Algo~Car\`{e} \and Marco~C.~Campi \and Maria~Prandini}


\institute{
	S. Garatti, A. Falsone, L. Manieri, and M. Prandini \at
	Dipartimento di Elettronica Informazione e Bioingegneria, Politecnico di Milano, Italy \\
	\email{\{simone.garatti,alessandro.falsone,lucrezia.manieri,maria.prandini\}@polimi.it}
	\and
	A. Car\`{e} and M.C. Campi \at
	Dipartimento di Ingegneria dell'Informazione, Universit\`a degli studi di Brescia, Italy \\
	\email{\{algo.care,marco.campi\}@unibs.it}
}

\date{Received: date / Accepted: date}

\maketitle

\begin{abstract}
	The scenario approach provides a powerful data-driven framework for designing solutions under uncertainty with rigorous probabilistic robustness guarantees. Existing theory, however, primarily addresses assessing robustness with respect to a single appropriateness criterion for the solution based on a dataset, whereas many practical applications -- including multi-agent decision problems -- require the simultaneous consideration of multiple criteria and the assessment of their robustness based on multiple datasets, one per criterion. This paper develops a general scenario theory for multi-criteria data-driven decision making. A central innovation lies in the collective treatment of the risks associated with violations of individual criteria, which yields substantially more accurate robustness certificates than those derived from a naive application of standard results. In turn, this approach enables a sharper quantification of the robustness level with which all criteria are simultaneously satisfied, and it is shown that informative guarantees can be delivered also when the number of criteria tends to infinity. The proposed framework applies broadly to multi-criteria data-driven decision problems, providing a principled, scalable, and theoretically grounded methodology for design under uncertainty.
	\keywords{data-driven decision-making \and non-convex optimization \and multi-agent systems \and duality-based methods}
\end{abstract}

\section{Introduction} \label{sec:introduction}

Design in the presence of uncertainty is paramount in science and engineering, where systems are routinely required to operate reliably despite incomplete knowledge and unpredictable disturbances. As the complexity of the problems addressed continues to grow, domain-specific knowledge is often insufficient, or even unavailable, to support traditional model-based approaches. In this context, data-driven methods are emerging as a dominant paradigm, leveraging historical realizations of uncertainty to inform and robustify optimization and decision-making directly from data, \cite{BertsTh:2006,CampiCareGaratti2021,Dorfler2023-A, Dorfler2023-B}. By extracting structural information and variability patterns from observed samples, these methods enable the construction of solutions that remain reliable under unseen conditions. A broad spectrum of approaches has been developed to address this challenge, \cite{1R_BertGuptKall_2017b,2R_BertGuptKall_2017a,SutterVanParysKuhn2024,KannanBayraksanLuedtke2025,BennounaVanParys2026}, with data-driven distributionally robust optimization emerging as one of the most prominent and extensively investigated frameworks, \cite{3R_CalElGa_2006,4R_DelYe_2010,22K_Goh_Sim_2010,51K_WKS_2014,BayraksanLove2015,8R_JiangGuan_2016,5R_PeyKuhn_2017,lam2019recovering,Kuhnetal2019,GamboaValadaoStreetHomemDeMello2021,vanparys2021data,GaoKleywegt2023,ChenKuhnWiesemann2024}.

Among data-driven techniques, the scenario approach represents a powerful framework for certifying the robustness of designed solutions. The scenario approach was originally introduced for convex optimization in \cite{CalCam:05, calafiore_ScenarioApproachRobust_2006}, and its fundamental probabilistic guarantees were further established in \cite{campi_ExactFeasibilityRandomized_2008}. While \cite{campi_ScenarioApproachSystems_2009} consolidated the framework in systems and control, subsequent work refined and extended the theory, addressing, among other aspects, structural properties of specific problems, connections with compression learning, sample complexity and sequential optimization schemes, constraint removal, and objective-value guarantees \cite{SchiFaMo:13, MarPraLyl:14, Alamoetal2015, ChaDaTeVeWa2016, ShangYou2020, Romao_etal_2023a, ROMAOetal_2023b, wangJungers2023objective}. The framework was progressively extended beyond convex optimization \cite{Peyman_Sutter_Llygeros2015, GrammaticoEtal:14, campi_GeneralScenarioTheory_2018, Assif_Chatterjee_Banavar_2020, garatti_NonconvexScenarioOptimization_2025}, and the latest developments have highlighted the central role of the observed complexity of data-driven decisions in characterizing their out-of-sample reliability \cite{campi_WaitandjudgeScenarioOptimization_2018, garatti_RiskComplexityScenario_2022, GarCarCam2022, CamGa2023, garatti_NonconvexScenarioOptimization_2025}. The key feature of the modern scenario approach is that probabilistic guarantees are derived from the very same dataset used for design, thereby fully exploiting the informational content of the data while rigorously quantifying out-of-sample performance. This dual role of data — both to shape the solution and to assess its reliability — ensures an efficient and principled use of available information.

So far, the scenario approach has primarily been developed to quantify solution robustness with respect to a single prescribed appropriateness criterion. In many practical applications, however, multiple criteria must be considered simultaneously, often with differing levels of importance. Such situations arise, for example, in multi-agent optimization and federated learning \cite{nedic2020distributed, aledhari2020federated}, where several agents collaborate to solve a common decision problem, each introducing its own requirement on the solution. Each of these requirements naturally defines a distinct appropriateness criterion. When multiple appropriateness criteria are present, a unified framework enabling simultaneous design and certification across all relevant robustness requirements becomes necessary. In particular, a theory is needed that can certify the robustness level with which the data-driven solution satisfies a specific criterion among the various ones (this is captured by the notion of \emph{individual risks} introduced later in this paper) as well as the robustness with which the solution simultaneously satisfies all criteria (\emph{joint risk}). The latter concept is specific to multi-criteria settings and is crucial in applications where the violation of even a single requirement may have disruptive consequences.

Most contributions within the scenario approach literature address problems involving a single appropriateness criterion and are therefore not directly applicable to this more general setting. An initial attempt to handle multiple individual risks was presented in \cite{SchiFaMo:13}; however, partly due to the absence of more recent theoretical developments at the time, the resulting guarantees were relatively loose and restricted to problems with specific structural properties. More recently, \cite{falsone_ScenariobasedApproachMultiagent_2020, margellos2025encyclopedia} made a significant step forward by extending the results of \cite{campi_GeneralScenarioTheory_2018, campi_WaitandjudgeScenarioOptimization_2018} to convex robust multi-agent optimization problems in which different agents impose distinct constraints, interpretable as different appropriateness criteria. In those works, bounds for both the individual and joint risk are derived, but the analysis strongly relies on the specific structure of the considered setup and remains conservative, as individual risks are treated independently in order to leverage results from the single-criterion framework of \cite{campi_GeneralScenarioTheory_2018, campi_WaitandjudgeScenarioOptimization_2018}. In \cite{manieri_ProbabilisticFeasibilityDatadriven_2023}, a further attempt was made by providing guarantees for a specific non-convex multi-agent setting using statistical learning arguments; however, the resulting bounds are also conservative.

The main contribution of this article is to fully extend the theory of \cite{garatti_RiskComplexityScenario_2022, garatti_NonconvexScenarioOptimization_2025} -- currently representing the most advanced and accurate results within the scenario approach -- to a multi-criteria setting. For a broad class of decision schemes, we show that accurate \emph{a posteriori} certifications of both the individual and joint risk can be obtained based on a suitably defined observable quantity -- namely, a statistic of the data -- that naturally extends the notion of \emph{complexity} introduced in \cite{garatti_RiskComplexityScenario_2022}. We also establish guarantees that hold independently of the observed data realization and, hence, come handy for \emph{a priori} considerations: bounds that explicitly depend on the number of appropriateness criteria adopted, say $m$, as well as (more conservative) bounds that hold uniformly with respect to $m$.

Beyond covering a significantly broader class of multi-criteria data-driven decision schemes and leveraging the most recent advances in scenario theory, the key distinction between the results of the present paper and those of \cite{falsone_ScenariobasedApproachMultiagent_2020, manieri_ProbabilisticFeasibilityDatadriven_2023} lies in the collective treatment of individual risks, as opposed to their independent evaluation. This collective analysis yields risk certificates that rule out certain worst-case configurations of the individual risks that would otherwise be considered admissible under independent evaluations. As a result, the obtained collective characterization of the individual risks is significantly tighter than those that can be obtained using the results in \cite{falsone_ScenariobasedApproachMultiagent_2020, manieri_ProbabilisticFeasibilityDatadriven_2023}, with consequent dramatic improvements in joint risk certification.

\subsection*{Structure of the paper}
The remainder of the paper is organized as follows. After introducing the relevant notation, Section~\ref{sec:problem_setting} formally presents the considered multi-criteria decision framework along with the key notions and assumptions. Section~\ref{sec:individual_risk} develops new theoretical results for the collective evaluation of individual risks, highlighting the differences with independent evaluations previously proposed in the literature. Section~\ref{sec:joint_risk} leverages the results of Section~\ref{sec:individual_risk} to address joint risk certification, obtaining for the first time guarantees that scale favorably with the number of appropriateness criteria. Final remarks are provided in Section~\ref{sec:conclusions}. For readability, the (rather lengthy) proofs of all theoretical results, together with some supporting derivations, are deferred to the Appendix.

\subsection*{Mathematical notation}
We denote with $\bbZ$ the set of integer numbers and with $\bbN\subset\bbZ$ its subset of positive integers. We indicate the set of non-negative integers as $\bbN_0=\bbN\cup\{0\}$.

We use bold symbols for vectors, e.g. $\bv=[v_1 \,\cdots\,v_m]\T\in\bbR^m$, and $|\bv| = \sumim |v_i|$ denotes their 1-norm. We call \textit{multi-index} a vector $\balpha= [\alpha_1\,\cdots\,\alpha_m]\T \in\bbZ^m$ with integers components. We denote with $\bzero$ and $\bone$, respectively, a vector of ones and zeros of appropriate dimension.

For vectors $\bu,\bv\in\bbR^m$, inequalities of the form $\bu\le\bv$ ($\bu\ge\bv$) must be intended component-wise, i.e $u_i\le v_i$ ($u_i\ge v_i$), for all $\im$. Products or ratios between vectors also has to be intended as component-wise, i.e., $\bu \bv = [u_1 v_1 \,\cdots\, u_m v_m]\T$ and $\frac{\bu}{\bv} = [\frac{u_1}{v_1} \, \cdots \, \frac{u_m}{v_m} ]\T$. Given two multi-indices $\ubar{\balpha},\bar{\balpha} \in\bbZ^m$ with $\ubar{\balpha} \leq \bar{\balpha}$, the compact notation $\balpha=\ubar{\balpha},\dots,\bar{\balpha}$ denotes that multi-index $\balpha$ ranges in the set $\{ \balpha \in \bbZ^m :~ \ubar{\alpha}_i \le \alpha_i \le \bar{\alpha}_i,~ \im \}$.

The $m$-dimensional object $\{ x_{\balpha} \}_{\balpha=\ubar{\balpha}}^{\bar{\balpha}}$ is the collection of elements $x_{\balpha} = x_{\alpha_1,\dots,\alpha_m}$, indexed by the multi-index $\balpha \in \bbZ^m$, when $\balpha=\ubar{\balpha},\dots,\bar{\balpha}$. In addition, we use the compact notation $ \sum_{\balpha=\ubar{\balpha}}^{\bar{\balpha}} x_{\balpha}$ to denote the nested summation $\sum_{\alpha_1=\ubar{\alpha}_1}^{\bar{\alpha}_1}\cdots\sum_{\alpha_m=\ubar{\alpha}_m}^{\bar{\alpha}_m} x_{\alpha_1,\dots,\alpha_m}$. Given multi-indices $\balpha,\bnu\in\bbZ^m$ and a vector $\bv\in\bbR^m$, we use the compact notation $\binom{\balpha}{\bnu}$ as a short-hand for the product of binomial coefficients $\prodim\binom{\alpha_i}{\nu_i}$, and $\bv^{\balpha}$ to denote the product of powers $\prodim v_i^{\alpha_i}$. Note, in particular, that $\bv^{\bone}$ will be useful as a shorthand for $\prodim v_i$.

With few exceptions, we use calligraphic symbols for collections of elements. We use curly braces $\{ \cdot \}$ for sets and round braces $(\cdot)$ for lists. A list $\cE=( e_1,\dots,e_m)$ is an ordered sequence of elements $e_i$, $\im$, from any given generic space $E$ where multiple occurrences of the same instance are allowed. We introduce a multiplicity function $\mu_{\cE}(e): E \to \bbN_0$ to count the number of occurrences of $e$ in $\cE$ (equal to 0 if $e$ does not appear in $\cE$). With a slight abuse of notation, we use the same operator $|\cdot|$ to denote the number of elements in a set or a list, counting repetitions for lists (we also recall that the same symbol denotes the 1-norm when applied to a vector; the intended meaning will be clear from the context).

We extend set operators to lists as follows. We write $e\in\cE$ if $\mu_{\cE}(e)>0$, and $e \notin \cE$ otherwise. Given two lists $\cE = (e_1,\dots,e_m)$ and $\cE' = (e'_1,\dots,e'_n)$, we denote by $\cE \cup \cE'$ the list $(e_1,\dots,e_m,e'_1,\dots,e'_n)$, obtained by appending all elements in $\cE'$ to $\cE$, preserving their order and multiplicity (the multiplicity function associated with $\cE \cup \cE'$ is given by $\mu_{\cE \cup \cE'}(e) = \mu_{\cE}(e)+\mu_{\cE'}(e)$ for all $e \in E$). We denote by $\cE \setminus \cE'$ the list obtained by removing the first $k = \min\{\mu_{\cE}(e),\mu_{\cE'}(e)\}$ occurrences of $e$ from the list $\cE$, for each distinct element $e \in \cE'$ (the multiplicity function associated with $\cE \setminus \cE'$ is given by $\mu_{\cE \setminus \cE'}(e) = \max\{\mu_{\cE}(e)- \mu_{\cE'}(e),0\}$ for all $e \in E$).

We use $\ind_{V}(\bv)$ to denote the indicator function over the set $V$, that is equal to $1$ if $\bv \in V$ and $0$ otherwise. We adopt the usual convention that $0^0=1$.

\section{Problem formulation and key definitions} \label{sec:problem_setting}

\subsection{Mathematical setup}
Throughout, let $z \in \mathcal{Z}$ denote the decision variable belonging to a general (possibly infinite-dimensional) decision space $\mathcal{Z}$. The symbol $\delta$ instead represents an uncertain parameter collecting all the uncertain elements of the environment in which the decision will be deployed. The realization of $\delta$ is modeled as an outcome from a probability space $(\Delta,\mathfrak{D},\bbP)$. Importantly, while we assume $(\Delta,\mathfrak{D},\bbP)$ to exist, no precise knowledge of it is required, reflecting a fundamental ignorance of the mechanism through which uncertainty is generated, as is typical of complex decision problems.

We consider a decision framework where there are $m \in \bbN$ different criteria to judge the appropriateness of a decision $z \in \cZ$. Here, appropriateness captures the interaction of the decision with the environment and, therefore, a decision may or may not be appropriate with respect to a given criterion depending on the realization of $\delta$. For $\im$, we denote by $\cZ_i(\delta) \subseteq \cZ$ the set of decisions which are appropriate according to the $i$-th criterion under realization $\delta$. As is clear, $\delta$ may consist of multiple components, which may be independent or dependent under $\bbP$. Each criterion may depend only on some components of $\delta$. The components involved across criteria may be either disjoint or overlapping, without any restriction.

Adopting a data-driven perspective, we assume that for each criterion $i$, $\im$, one can collect $N_i \in \bbN$ observations of $\delta$ called scenarios. To streamline the developments that follow, we augment each collected scenario with a label indicating the criterion to which it refers, thus forming the pair $(\delta,i)$, where $i$ denotes the corresponding criterion index. The $N_i$ scenarios labeled by the criterion index $i$ are then gathered in the list $\cD_i=( (\scen{}{1},i),\dots,(\scen{}{N_i},i))$. For brevity, we write each pair $(\scen{}{j},i)$ as $\sceni{j}$, and we denote by $\Delta_i$ the corresponding domain $ \Delta \times \{i\}$. When no ambiguity arises, we use the term scenario to refer not only to the $\delta$-component of $\delta_i^{(j)}$, which is the scenario in the strict sense, but also to $\delta_i^{(j)}$ itself. The $\delta$-components of all scenarios $\sceni{j}$, $j=1,\ldots,N_i$, $\im$ -- i.e., within and across datasets -- are modeled as independent outcomes from the same probability space $(\Delta,\mathfrak{D},\bbP)$ (i.i.d. draws). Then, a data-driven decision scheme is any algorithm that uses the information contained in the datasets list $\bcD = (\ms)$, including labels, to construct a decision. Accordingly, such a scheme can be represented as a mapping
\begin{equation} \label{eq:decision_map}
	\cM :~ \bigcup_{\bN\in\bbN^m} \Delta_1^{N_1} \times \cdots \times \Delta_m^{N_m}\to\cZ,
\end{equation}
with $\bN = [N_1 \,\cdots\, N_m]\T$. When there is no ambiguity regarding the datasets used, the returned decision $\cM(\bcD)$ will be often denoted by $\zsN$.

A typical application of this framework is multi-agent decision making, where $m$ agents cooperate to agree on a decision that is satisfactory to all of them, each agent being associated with its own appropriateness criterion. In such settings, agents often collect their own (private) datasets independently, and a common decision is reached while each agent relies on its local data to enforce its own requirements.\footnote{This can, for example, be achieved through consensus algorithms that avoid sharing sensitive information. See, e.g., \cite{NedicOzdaglar2009subgradient, ZhuMartinez2012, di2016next, nedic2017achieving, FALSONE_etal_2017, Margellos_etal_2018, qu2017harnessing, li2019decentralized, falsone2022aut, falsone2023aut_ALT}.} In this context, labels indicate which agent collected the scenario. Besides multi-agent problems, the proposed framework encompasses a large variety of multi-criteria settings. In particular, it is worth noting that it can also accommodate the common and realistic situation in which only partial observations of uncertainty are collected for each criterion. Specifically, for each $\sceni{j}$, only the components of $\delta$ required to evaluate the $i$-th criterion can be observed. From an abstract viewpoint, this is equivalent to a decision scheme that is provided with the entire realization $\sceni{j}$, as in~\eqref{eq:decision_map}, but the definition of $\cM$ incorporates the discarding of those parts of $\sceni{j}$ that are not relevant to the $i$-th criterion, an operation made possible by the presence of the label $i$ in $\sceni{j}$.

In this paper, we do not focus on a specific decision scheme $\mathcal{M}$. Instead, we establish results that apply to a broad class of schemes, allowing considerable flexibility in how the scenarios are used to construct a decision. The admissible schemes are characterized by the properties stated in the following Assumption~\ref{ass:property}, which extend to the multi-criteria setting the analogous conditions introduced in~\cite{garatti_RiskComplexityScenario_2022,garatti_NonconvexScenarioOptimization_2025} for the single-criterion case.

\begin{assum} \label{ass:property}
	For any number of scenarios $N_i\in\bbN$ and any dataset $\cD_i = ( \sceni{1},\dots,\sceni{N_i} )$, $\im$, the map $\cM$ satisfies the following conditions.
	\begin{enumerate}[label=\roman*., ref=\roman*] \setlength\itemsep{1em}
		\item \emph{Permutation invariance}: \label{property-permutation invariance}
		For any list $\cD_{i}^\pi$ obtained from $\cD_i$ by permuting its elements, $\im$, we have $\cM (\cD^\pi_1,\dots,\cD_m^\pi) = \cM (\cD_1,\dots,\cD_m)$.
	\end{enumerate}
	For all $\im$, let $\widetilde\cD_i$ be a list of additional $n_i \in \bbN_0$ scenarios for criterion $i$ and define ${\cD}_i^+ =\cD_i\cup \widetilde\cD_i$.
	\begin{enumerate}[label=\roman*., ref=\roman*] \setcounter{enumi}{1}
		\item\emph{Stability in case of confirmation}: \label{property-stability}
		If $\cM(\cD_1,\dots,\cD_m)\in\cZ_i(\delta)$ for all $\delta \in \widetilde{\cD}_i$ and for all $\im$, then $\cM(\cD_1^+,\dots,\cD_m^+) = \cM({\cD}_1,\dots,{\cD}_m)$.
		\item\emph{Responsiveness to contradiction}: \label{property-responsiveness}
		If there exists $i \in \{1,\dots,m\}$ and $\delta \in \widetilde{\cD}_i$ such that $\cM(\cD_1,\dots,\cD_m) \not\in \cZ_i(\delta)$, then $\cM({\cD}_1^+,\dots,{\cD}_m^+) \neq \cM(\cD_1,\dots,\cD_m)$.
	\end{enumerate}
Conditions~\ref{property-stability} and~\ref{property-responsiveness} are called \emph{consistency} conditions.~\sqed
\end{assum}
Permutation invariance in \emph{\ref{property-permutation invariance}} requires that the decision does not depend on the ordering of scenarios within each dataset. The consistency properties \emph{\ref{property-stability}} and \emph{\ref{property-responsiveness}} instead govern how the decision reacts to additional data. Stability under confirmation requires that the decision remain unchanged when it is already appropriate for the newly collected observations with respect to the criterion these scenarios were collected for. Responsiveness to contradiction, on the other hand, requires the decision to change whenever it is not appropriate according to at least one criterion for one or more newly added observations for that same criterion.

For the sake of concreteness, we present here an example of a decision scheme satisfying Assumption~\ref{ass:property}, the so-called robust scenario optimization framework. In this case, $\cZ = \bbR^d$ and $\zsN$ is defined as the solution of the following optimization problem:
\begin{align}
	\min_{z \in \bar{\cZ} \subseteq \bbR^d} \quad
	&c(z) \nonumber \\
	\text{s.t.} \quad
	&z \in \bigcap_{i=1}^{m} \left[ \bigcap_{j=1}^{N_i} \cZ_i(\sceni{j}) \right]. \label{eq:multiagent-opt}
\end{align}
In this problem, each criteria is associated with a constraint set that depend on the uncertain parameter $\delta$ and the objective is to minimize the cost $c$, while safeguarding against the worst for all criteria as assessed based on the available scenarios. This motivates imposing the constraint $z \in \bigcap_{i=1}^{m} \left[ \bigcap_{j=1}^{N_i} \cZ_i(\sceni{j}) \right]$. Instead, $\bar{\cZ}$ can be used to enforce additional conditions on $z$, e.g., that some components can only take integer values. The stringent requirement that the solution be appropriate for all scenarios and criteria makes it straightforward to verify that Assumption~\ref{ass:property} holds for this scheme.

Problems of the form~\eqref{eq:multiagent-opt} frequently arise in multi-agent optimization, where $\cZ_i(\sceni{j})$, $j=1,\ldots,N_i$, represent data-based constraints imposed by agent $i$. When $z$ is a common decision vector and $c(z) = \sum_{i=1}^m c_i(z)$, \eqref{eq:multiagent-opt} is called robust multi-agent decision-coupled optimization problem (cf., e.g.,~\cite{falsone2022aut}); if $z = [z_1 \, \cdots \, z_m] \in \bbR^{d_1} \times \cdots\times \bbR^{d_m} = \bbR^d$, $c(z) = \sum_{i=1}^m c_i(z_i)$, $\bar{\cZ} = \{z: \sum_{i=1}^m g_i(z_i) \le 0 \}$, and $\cZ_i(\delta)$ is constraining $z_i$ only, the problem is called constrained-coupled (cf., e.g.,~\cite{falsone2023aut_ALT}). Both instances fall within our framework.

We conclude by emphasizing that~\eqref{eq:multiagent-opt} is only a particular instance of the broader class of decision schemes satisfying Assumption~\ref{ass:property}. We refer the reader to~\cite{garatti_RiskComplexityScenario_2022,garatti_NonconvexScenarioOptimization_2025} for further discussion of the generality of this framework.

\subsection{The risk certification problem}
Once the decision $\cM(\bcD)$ is computed, assessing its robustness is essential for reliable use. On one hand, we are interested in estimating the robustness level of the solution with respect to each individual criterion of appropriateness. To formalize this, we introduce the following definition.
\begin{defn}[Individual risks]
	Given a decision $z \in \cZ$, the individual risk with respect to the $i$-th criterion of appropriateness is defined as
	\begin{equation} \label{eq:local-risks}
		V_i(z) = \bbP\{\delta\in\Delta :~ z \notin \cZ_i(\delta)\},
	\end{equation}
	i.e., it is the probability that $z$ will not be appropriate according to the \emph{specific} criterion $i$ for a realization $\delta$ extracted at random according to $\bbP$. We also denote by $\bV(z) = [V_1(z) \,\cdots\, V_m(z)]\T$ the vector collecting the individual risks for all criteria. \sqed
\end{defn}
In many problems, the various appropriateness criteria may carry different priorities, which motivates the use of individual risks. However, in situations where being inappropriate for \emph{any} criterion can have detrimental consequences, it is necessary to monitor the robustness of the decision with respect to all criteria simultaneously. This motivates the following notion of \emph{joint risk}.
\begin{defn}[Joint risk]
	Given a decision $z \in \cZ$, the joint risk is defined as
	\begin{equation} \label{eq:global-risk}
		V(z) = \bbP\{\delta\in\Delta :~ z \notin \cZ_1(\delta) \lor \cdots \lor z \notin \cZ_m(\delta) \},
	\end{equation}
	i.e., it is the probability that $z$ will not be appropriate according to \emph{some} criterion $i$, $\im$, for a realization $\delta$ extracted at random according to $\bbP$. \sqed
\end{defn}

Given these two definitions, the objective of the paper can be summarized as evaluating the individual and joint risks $\bV(\zsN)$ and $V(\zsN)$ of the data-driven decision $\zsN = \cM(\bcD)$. In particular, we aim to provide a collective characterization of the individual risks that can capture any possible coupling among them; coupling is clearly possible as the risks refer to the same decision $\zsN$, which is obtained from the entire dataset. The joint risk $V(\zsN)$ will be then deduced as a by-product of this characterization.

In a data-driven setting, $\bbP$ is not known to the user, so $\bV(\zsN)$ and $V(\zsN)$ cannot be computed directly. Following the scenario approach, our goal is to show that these quantities can be accurately estimated from the \emph{same} data used to compute $\zsN$, thereby enabling the efficient use of all available information for both the design and the certification of the solution. Specifically, we aim to provide results of the form
\begin{equation*}
	\bbPbN \{ \bV(\zstar) \in \text{region computed from } \bcD \} \ge 1-\beta,
\end{equation*}
for a given (small) $\beta\in (0,1)$ confidence parameter, and similarly for $V(\zsN)$. Here, $\bbPbN$ refers to $\bcD$ and denotes the product probability $\bbP \times \cdots \times \bbP$ (repeated $|\bN|$ times), thus reflecting the fact that the dataset is formed by independent draws from $(\Delta,\mathfrak{D},\bbP)$. The interpretation of these results is that, with high confidence with respect to the variability of $\bcD$, the solution risk vector lie within a region computed from the observed data. This region can then be safely used for reliability assessments. Evidently, the main challenge is to identify regions that are informative and yet tractable. The following definition is central to this purpose.
\begin{defn}[Support scenarios and Complexity] \label{def:support-scenario}
	Given a vector $\bN = [N_1 \,\cdots\, N_m]\T$ and a realization of $\bcD = (\ms)$, an element $\scenij \in \cD_i$, $j \in \{1,\ldots,N_i\}$, $i \in \{1,\ldots,m\}$, is said to be a support scenario if its removal while maintaining all the others changes the decision, i.e.,
	\begin{equation*}
		\cM (\cD_1,\dots,\cD_i\setminus (\scenij),\dots,\cD_m ) \neq \cM(\ms).
	\end{equation*}
	Let $\cS_i \subseteq \cD_i$ be the sub-list of support scenarios in $\cD_i$ and let $\sstari = |\cS_i| \in \bbN_0$ be the number of support scenarios in $\cD_i$, $\im$. The \emph{complexity} is the vector $\bsstar = [\sstari[1] \,\cdots\, \sstari[m]]\T$ containing the number of support scenarios for all criteria.	\sqed
\end{defn}
Definition~\ref{def:support-scenario} extends to the multi-criteria setting the notions of support scenarios and complexity introduced in~\cite{garatti_RiskComplexityScenario_2022} for a single appropriateness criterion. Plainly, $\bsstar$ is a statistic of the data $\bcD$. In~\cite{garatti_RiskComplexityScenario_2022}, as recalled in the next section, it is shown that, in the single-criterion case, the complexity is closely related to the risk of $\zsN = \cM(\bcD)$, making it a key observable for risk estimation. In this paper, we show that the extended notion of complexity in Definition~\ref{def:support-scenario} retains the same property in a multi-criteria setting, both for estimating individual risks and the joint risk.

Following~\cite{campi_WaitandjudgeScenarioOptimization_2018, garatti_RiskComplexityScenario_2022}, all results in this article are derived under the following assumption, referred to as \emph{non-degeneracy}.
\begin{assum}[Non-degeneracy] \label{ass:non-degeneracy}
	For every $\bN \in \bbN^m$ and for every $\bcD$, let $\bcS = (\cS_1,\dots,\cS_m)$ be the list of lists of support scenarios for all criteria. Then, $\cM(\bcS) = \zsN = \cM(\bcD)$, with probability one with respect to $\bbPbN$.
	\sqed
\end{assum}
Assumption~\ref{ass:non-degeneracy} requires that the support scenarios capture all the information needed to reconstruct the data-driven decision $\zsN$, except for zero-probability cases. For this reason, in the following, we will also refer to $\bsstar$ as the \emph{complexity of $\zsN$}. Non-degeneracy is relatively mild in some problems. For instance, when decisions are obtained via convex optimization, it often reduces to a simple condition ensuring that the probability does not exhibit concentrated mass (see~\cite{garatti_RiskComplexityScenario_2022,garatti_NonconvexScenarioOptimization_2025} for more detailed discussion). Nevertheless, this assumption restricts the applicability of the results. In~\cite{garatti_NonconvexScenarioOptimization_2025}, it was shown -- with significantly more complex derivations -- that some results from~\cite{garatti_RiskComplexityScenario_2022} remain valid even without assuming non-degeneracy. While we are confident that the techniques of~\cite{garatti_NonconvexScenarioOptimization_2025} can be extended to the present multi-criteria setting to relax the non-degeneracy assumption, this extension is left for future work.

\section{Individual risks assessment} \label{sec:individual_risk}
In this section, we devote our attention to providing a collective estimate of the individual risks $\bV(\zsN)$ associated with the scenario decision $\zsN = \cM(\bcD)$. We first show in Section~\ref{sec:independent_individual_bounds} how the standard theory of the scenario approach can be used to derive individual estimates, which however do not capture the interrelationship that may exist among different risks. This will also provide some background and a term of comparison. Then, in Section~\ref{sec:joint_individual_bounds} we present a new scenario theory to derive tighter estimates that truly hold collectively. This constitutes the main result of the paper.

\subsection{Independent assessment of individual risks} \label{sec:independent_individual_bounds}
Suppose first that $m = 1$, in which case there is only one appropriateness criterion and the theory of \cite{garatti_RiskComplexityScenario_2022} can be directly applied, as recalled in the next Theorem~\ref{thm:one_criterion_bound}. Since in this case, every quantity is scalar, we use a non-bold notation.
\begin{thm} \label{thm:one_criterion_bound}
	Fix $m = 1$, let $N \in \bbN$ be any integer, and let $H \ge N$. For any $k \in \{0,1,\ldots,N \}$, let $[\epstlbs(k),\epstubs(k)] \subseteq [0,1]$ be an interval whose extrema are parameterized in $k$. Under Assumptions~\ref{ass:property} and~\ref{ass:non-degeneracy}, $\zstar = \cM(\cD)$ satisfies
	\begin{equation} \label{eq:one_criterion_bound}
		\bbP^{N} \{ V(\zstar) \in [\epstlbs(\sstar),\epstubs(\sstar)] \} \ge \gamma^\star,
	\end{equation}
	where $\sstar \in \bbN_0$ is the complexity of $\zstar$ and $\gamma^\star$ is the optimal cost of the following problem
	\begin{subequations} \label{eq:dual-standard-scenario}
	\begin{align}
		\gamma^\star = \sup_{\{\lambda_{h}\}_{h=0}^{H}} \quad
		&\sum_{h=0}^{H} \lambda_{h} \label{eq:dual-standard-scenario-cost} \\
		\text{s.t.} \quad
		&\begin{aligned}
			&\sum_{h=k}^{H} \lambda_{h} \frac{\binom{h}{k}}{\binom{N}{k}}(1-v)^{h-N} \le \ind_{[\epstlbs(k),\epstubs(k)]}(v) \\
			&\qquad \forall v \in [0,1] : \; (1-v)^{N-k} > 0, \quad k = 0,\dots,N
		\end{aligned} \label{eq:dual-standard-scenario-fcn} \\
		&\begin{aligned}
			&\lambda_{h} \le 0, \quad h \neq N \\
			&\lambda_{h} \le 1, \quad h = N.
		\end{aligned} \label{eq:dual-standard-scenario-fcn-2}
	\end{align}
	\end{subequations}
	\sqed
\end{thm}
Theorem~\ref{thm:one_criterion_bound} is essentially derived within the proof of~\cite[Theorem~2]{garatti_RiskComplexityScenario_2022} -- see~\cite[Section 6.1, equations (39)-(40)]{garatti_RiskComplexityScenario_2022}.\footnote{Problem~\eqref{eq:dual-standard-scenario} is a, provably equivalent, reformulation of \cite[equation (40)]{garatti_RiskComplexityScenario_2022}.}

Theorem~\ref{thm:one_criterion_bound} can also be used conversely, to compute lower- and upper-bounds to the risk that hold with a user-decided confidence $1-\beta$, where $\beta \in (0,1)$. To this purpose, one starts from a tentative solution to~\eqref{eq:dual-standard-scenario} that makes $\gamma^\star \ge 1 -\beta$ and infer $\epstlbs(k)$ and $\epstubs(k)$ from~\eqref{eq:dual-standard-scenario-fcn}. Precisely, one selects $H \ge N$ and choose the variables $\{ \lambda_h \}_{h = 0}^H$ such that they satisfy~\eqref{eq:dual-standard-scenario-fcn-2} and their sum is equal to $1-\beta$. Then, for each $k$, $\epstlbs(k)$ and $\epstubs(k)$ are selected by ``encapsulating'' the left-hand side of~\eqref{eq:dual-standard-scenario-fcn} below the indicator function $\ind_{[\epstlbs(k),\epstubs(k)]}(v)$ with $[\epstlbs(k),\epstubs(k)]$ having the lowest possible range.\footnote{$\epstlbs(k)$ and $\epstubs(k)$ can always be found because~\eqref{eq:dual-standard-scenario-fcn-2} implies that the left-hand side of~\eqref{eq:dual-standard-scenario-fcn} is no bigger than $1$ for all $v \in [0,1]$.} In this way, the chosen $\{ \lambda_h \}_{h = 0}^H$ are feasible for~\eqref{eq:dual-standard-scenario-fcn} and~\eqref{eq:dual-standard-scenario-fcn-2}, and therefore the optimal solution to~\eqref{eq:dual-standard-scenario} with the obtained $[\epstlbs(k),\epstubs(k)]$ must result in $\gamma^\star \ge 1-\beta$. An application of Theorem~\ref{thm:one_criterion_bound} then gives
\begin{equation*}
	\bbP^{N} \{ V(\zstar) \in [\epstlbs(\sstar),\epstubs(\sstar)] \} \ge 1-\beta,
\end{equation*}
i.e., the risk of the scenario decision can be assessed with high confidence by evaluating the found $\epstlbs(k)$ and $\epstubs(k)$ at the complexity $\sstar$. \cite[Theorem~2]{garatti_RiskComplexityScenario_2022} is obtained as a noteworthy application of the above reasoning, selecting $H = 4N$ and the following candidate solution for~\eqref{eq:dual-standard-scenario}:
\begin{align*}
	&\lambda_h = -{\beta}/{2N}		&&\hspace{-5em} h = 0,\dots,N-1, \\
	&\lambda_h = 1 					&&\hspace{-5em} h = N, \\
	&\lambda_h = -\beta/{6N}		&&\hspace{-5em} h = N+1,\dots,4N.
\end{align*}
With such a choice, it is straightforward to verify that all required conditions are met; moreover, the left-hand side of~\eqref{eq:dual-standard-scenario-fcn} is a well-behaved function (it is negative at $v=0$; it is first increasing reaching a positive value and then decreasing tending to $-\infty$ for $v \to 1$) so that finding $\epstlbs(k)$ and $\epstubs(k)$ amounts to compute its zeros in $[0,1]$.\footnote{This is true $\forall k < N$; for $k=N$, the function is only increasing in $[0,1]$ and hence we must set $\epstubs(N) = 1$.} Other choices of $\{ \lambda_h \}_{h = 0}^H$ are possible. Another one that is worth mentioning consists in taking $H = N$ and $\lambda_h = -\beta/N$, $h = 0,\ldots,N-1$, and $\lambda_N = 1$. In this case, the left-hand side of~\eqref{eq:dual-standard-scenario-fcn} is positive for $v=0$ and it is a decreasing function tending to $-\infty$ for $v \to 1$. This means that $\epstlbs(k) = 0$, $\forall k$, while $\epstubs(k)$ is computed as the unique zero of the function. This choice is meaningful when one is interested in computing only an upper bound on the risk.

When $m > 1$, Theorem~\ref{thm:one_criterion_bound} can still be used to obtain independent evaluations of the individual risks. In fact, similarly to~\cite{falsone_ScenariobasedApproachMultiagent_2020} and owing to independence, one can focus on a single appropriateness criterion, say the $i$-th, and work conditionally to the scenarios in $\cD_j$ with $j \in \{1,\dots,m\} \setminus \{ i \}$. For every ``frozen'' value of these $\cD_j$'s, and considering the variability of $\cD_i$ only, $\cM$ can be seen as a map from $\cD_i$ to $\zsN$ with $m=1$, which satisfies Assumption~\ref{ass:property} as per the appropriateness criterion $\cZ_i(\delta)$ and whose complexity is $\sstari$. An application of Theorem~\ref{thm:one_criterion_bound} gives $\bbP^{N_i} \{ V_i(\zsN) \in[\epstlb_i(\sstari),\epstub_i(\sstari)] \}\ge 1-\beta_i$, where $\epstlb_i(\cdot),\epstub_i(\cdot)$ are obtained as discussed before with $N_i$ in place of $N$ and $\beta_i$ of $\beta$, and $\bbP^{N_i}$ represents the distribution of $\cD_i$. By virtue of independence, it holds that $\bbP^{|\bN|} \{ V_i(\zsN) \in[\epstlb_i(\sstari),\epstub_i(\sstari)] \; | \; \cD_j, \; j \in \{1,\dots,m\} \setminus \{ i \} \}=\bbP^{N_i} \{ V_i(\zsN) \in[\epstlb_i(\sstari),\epstub_i(\sstari)] \} $, so that integrating with respect to the distribution of $\cD_j$, $j \in \{1,\dots,m\} \setminus \{ i \}$, we obtain
\begin{equation} \label{eq:i-independent-bound-trivial}
	\bbPbN\{ V_i(\zsN) \in[\epstlb_i(\sstari),\epstub_i(\sstari)] \}\ge 1-\beta_i,
\end{equation}
for all $i \in \{ 1,\ldots,m\}$. Using the sub-additivity of $\bbPbN$, this eventually yields
\begin{equation} \label{eq:individual-independent-bound}
	\bbPbN \Big\{ \bV(\zsN) \in \prod_{i=1}^m [\epstlb_i(\sstari),\epstub_i(\sstari)] \Big\} \ge 1 - \beta,
\end{equation}
with $\beta = \sumim \beta_i$, which is a probabilistic bound for the vector of individual risks $\bV(\zsN)$.

\subsection{Collective assessment of individual risks} \label{sec:joint_individual_bounds}
The result in~\eqref{eq:individual-independent-bound} states that the vector of individual risks $\bV(\zsN)$ lies, with confidence $1-\beta$, within a box whose edge along the $i$-th dimension is given by the interval $[\epstlb_i(\sstari), \epstub_i(\sstari)]$. This assessment, however, was derived by treating the individual risks separately, which introduces a degree of conservatism. In particular, the box does not rule out the possibility that all risks simultaneously attain high values, each close to the corresponding $\epstub_i(\sstari)$. However, if the risks were in fact independent, the probability that they would all be large at the same time would be negligible. By contrast, in settings where the risks are intertwined, because the decision is formed by accommodating partially overlapping notions of appropriateness, the information conveyed by the scenarios $\cD_j$, for $j \in \{1,\dots,m\} \setminus \{ i \}$, contributes to keeping the $i$-th risk low. As a result, such jointly high values are also avoided. This informal reasoning suggests that bounds sharper than~\eqref{eq:individual-independent-bound} may be attainable. Remarkably, this is indeed the case at the level of generality of the present setup, without exploiting any specific structural properties of the decision scheme $\cM$ beyond Assumption~\ref{ass:property}.

The result is formalized in the next Theorem~\ref{thm:region_bound} and Corollary~\ref{corol:region_bound}, which provide a genuinely \emph{collective} assessment of the individual risks and constitutes the first main contribution of this paper.
\begin{thm} \label{thm:region_bound}
	Let $\bN$ be any given multi-index in $\bbN^m$ and let $\bH \ge \bN$. For $\bk = \bzero,\ldots, \bN$, let $\cR(\bk) \subseteq [0,1]^m$ be a region parameterized in $\bk \in \bbN_0^m$. Under Assumptions~\ref{ass:property} and~\ref{ass:non-degeneracy}, $\zsN = \cM(\bcD)$ satisfies
	\begin{equation} \label{eq:region_bound}
		\bbPbN\{ \bV(\zsN) \in \cR(\bsstar) \} \ge \gamma^\star,
	\end{equation}
	where $\bsstar \in \bbN_0^m$ is the complexity of $\zsN$ and $\gamma^\star$ is the optimal value of the following problem
	\begin{subequations} \label{eq:dual-lambda}
	\begin{align}
		\gamma^\star = \sup_{\{\lambda_{\bh}\}_{\bh=\bzero}^{\bH}} \quad
		&\sum_{\bh=\bzero}^{\bH} \lambda_{\bh} \label{eq:dual-lambda-cost} \\
		\text{s.t.} \quad
		&\begin{aligned}
			&\sum_{\bh=\bk}^{\bH} \lambda_{\bh} \frac{\binom{\bh}{\bk}}{\binom{\bN}{\bk}}(\bone-\bv)^{\bh-\bN} \le \ind_{\cR(\bk)}(\bv), \\
			&\qquad \forall \bv \in[0,1]^m : \; (\bone-\bv)^{\bN-\bk} > 0, \quad \bk = \bzero,\dots,\bN
		\end{aligned} \label{eq:dual-lambda-constraint} \\
		&\begin{aligned}
			&\lambda_{\bh} \le 0, \quad \bh \neq \bN \\
			&\lambda_{\bh} \le 1, \quad \bh = \bN
		\end{aligned} \label{eq:dual-lambda-lcond}
	\end{align}
	\end{subequations}
	(note that $\bk$ and $\bh$ are (running) multi-indices; for every $\bh$, $\lambda_{\bh}$ is a scalar variable).
	\sqed
\end{thm}
Theorem~\ref{thm:region_bound} states that, with probability at least $\gamma^\star$, the vector $\bV(\zsN)$ of individual risks belongs to the region $\cR(\bsstar) \subseteq [0,1]^m$, depending on the complexity for all criteria $\bsstar$. This result is a clear extension of Theorem~\ref{thm:one_criterion_bound}: for $m = 1$ and limiting $\cR(\bk)$ to the form of an interval $[\epstlbs(k),\epstubs(k)]$, Theorem~\ref{thm:region_bound} becomes identical to Theorem~\ref{thm:one_criterion_bound}.

As we did for Theorem~\ref{thm:one_criterion_bound}, we now discuss how Theorem~\ref{thm:region_bound} can be used conversely to provide region bounds for the risk vector, which hold with user-chosen confidence $1-\beta$. This is achieved by guessing a (suboptimal) $\{ \lambda_{\bh} \}_{\bh = \bzero}^{\bH}$ that gives $\gamma^\star \ge 1-\beta$ and then determining $\cR(\bk)$ from~\eqref{eq:dual-lambda-constraint}, without solving~\eqref{eq:dual-lambda} up to optimality, which may be computationally demanding. Precisely, after selecting a confidence parameter $\beta \in (0,1)$ and a multi-index $\bH \ge \bN$, one chooses the variables $\{ \lambda_{\bh} \}_{\bh = \bzero}^{\bH}$ such that they satisfy~\eqref{eq:dual-lambda-lcond} and their sum is equal to $1-\beta$, and compute regions $\cR(\bk)$ as the $0$-superlevel set of the left-hand side of~\eqref{eq:dual-lambda-constraint}, $\bk = \bzero,\dots,\bN$. In this way, the chosen $\{ \lambda_{\bh} \}_{\bh = \bzero}^{\bH}$ are feasible for~\eqref{eq:dual-lambda-constraint} and~\eqref{eq:dual-lambda-lcond}, and hence it holds that $\gamma^\star \ge 1-\beta$ for the so obtained $\cR(\bk)$. An application of Theorem~\ref{thm:region_bound} then gives the following result.
\begin{cor} \label{corol:region_bound}
	Let $\bN$ be any given multi-index in $\bbN^m$, $\bH \ge \bN$, and $\beta \in (0,1)$. Given $\{ \lambda_{\bh} \}_{\bh = \bzero}^{\bH}$ such that $\lambda_{\bh} \le 0$ for $\bh \neq \bN$, $\lambda_{\bN} \le 1$, and $\sum_{\bh =\bzero}^{\bH} \lambda_{\bh} = 1-\beta$, let
	$$
	\cR(\bk) = \left\{ v \in [0,1]^m : \; \sum_{\bh=\bk}^{\bH} \lambda_{\bh} \frac{\binom{\bh}{\bk}}{\binom{\bN}{\bk}}(\bone-\bv)^{\bh-\bN} \ge 0 \right\}
	$$
	for $\bk = \bzero,\ldots, \bN.$ Under Assumptions~\ref{ass:property} and~\ref{ass:non-degeneracy}, $\zsN = \cM(\bcD)$ satisfies
	$$
		\bbPbN\{ \bV(\zsN) \in \cR(\bsstar) \} \ge 1-\beta,
	$$
	where $\bsstar \in \bbN_0^m$ is the complexity of $\zsN$. \sqed
\end{cor}
Note that, in practice, the region bound for a decision $\zsN$ can be obtained simply evaluating $\cR(\bk)$ at $\bk = \bsstar$ only.

\begin{figure*}
	\centering
	\begin{subfigure}{0.328\linewidth}
		\includegraphics[width=\linewidth]{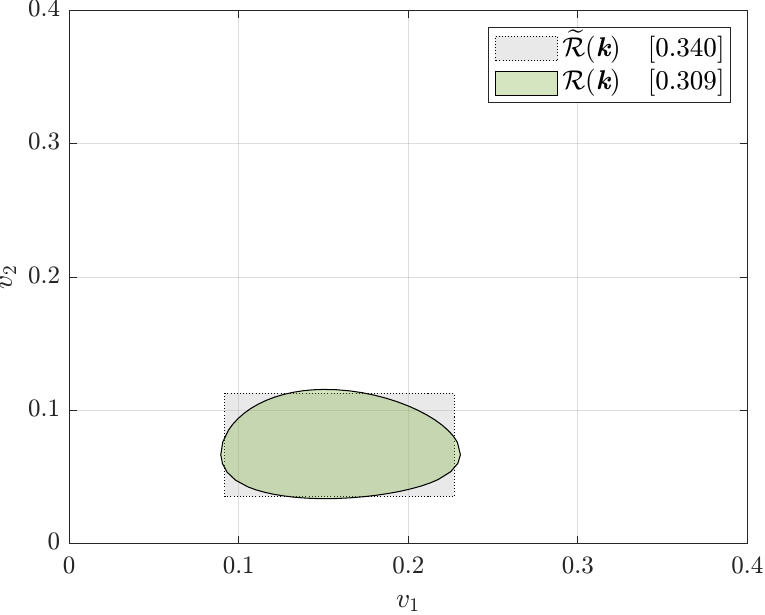}
		\caption{\centering Uniform allocation\\ $\bH = 2\bN$}
		\label{fig:ex-region-lambda-uniform}
	\end{subfigure}
	\hfill
	\begin{subfigure}{0.328\linewidth}
		\includegraphics[width=\linewidth]{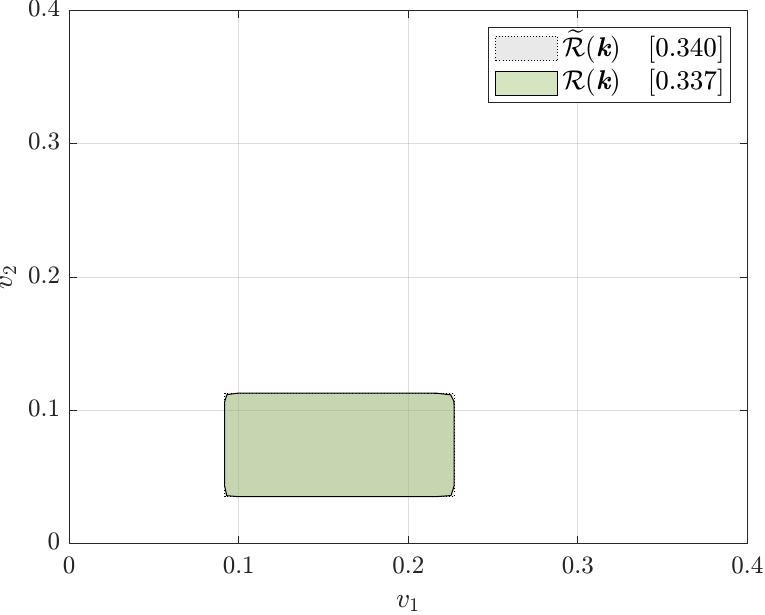}
		\caption{\centering Axial allocation\\ $\bH = 2\bN$}
		\label{fig:ex-region-lambda-axial}
	\end{subfigure}
	\hfill
	\begin{subfigure}{0.328\linewidth}
		\includegraphics[width=\linewidth]{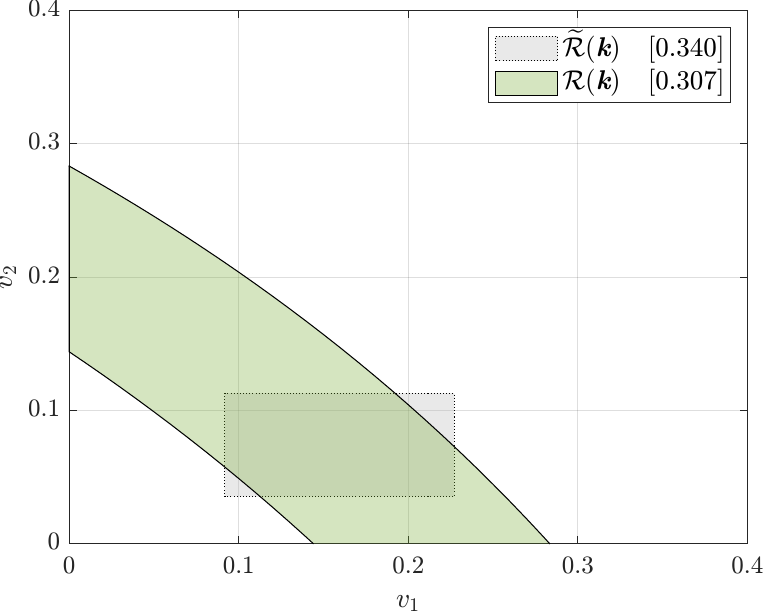}
		\caption{\centering Diagonal allocation\\ $\bH = 2\bN$}
		\label{fig:ex-region-lambda-diag}
	\end{subfigure}
	\\[1em]
	\begin{subfigure}{0.328\linewidth}
		\includegraphics[width=\linewidth]{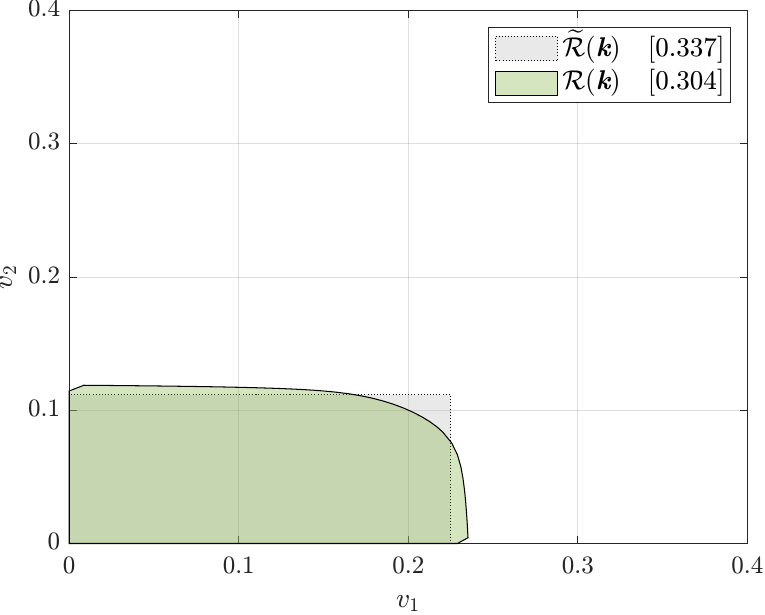}
		\caption{\centering Uniform allocation\\ $\bH = \bN$}
		\label{fig:ex-region-lambda-uniform-ub}
	\end{subfigure}
	\hfill
	\begin{subfigure}{0.328\linewidth}
		\includegraphics[width=\linewidth]{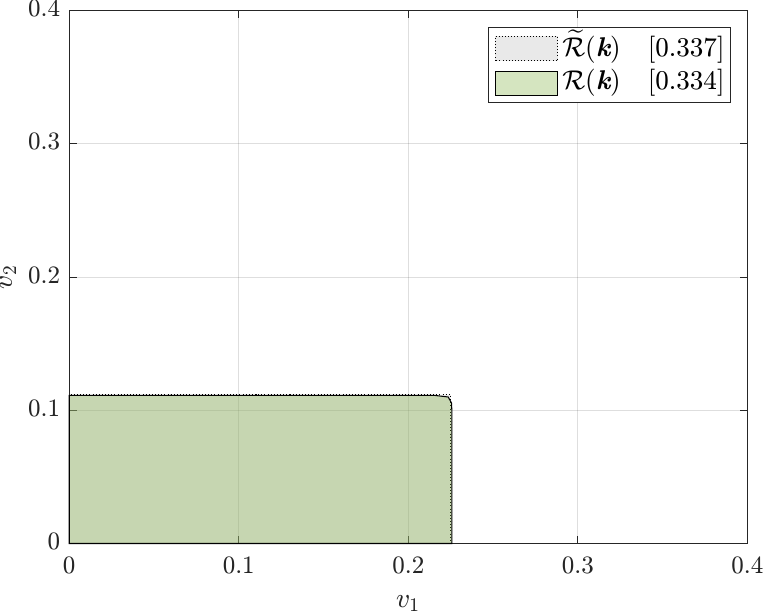}
		\caption{\centering Axial allocation\\ $\bH = \bN$}
		\label{fig:ex-region-lambda-axial-ub}
	\end{subfigure}
	\hfill
	\begin{subfigure}{0.328\linewidth}
		\includegraphics[width=\linewidth]{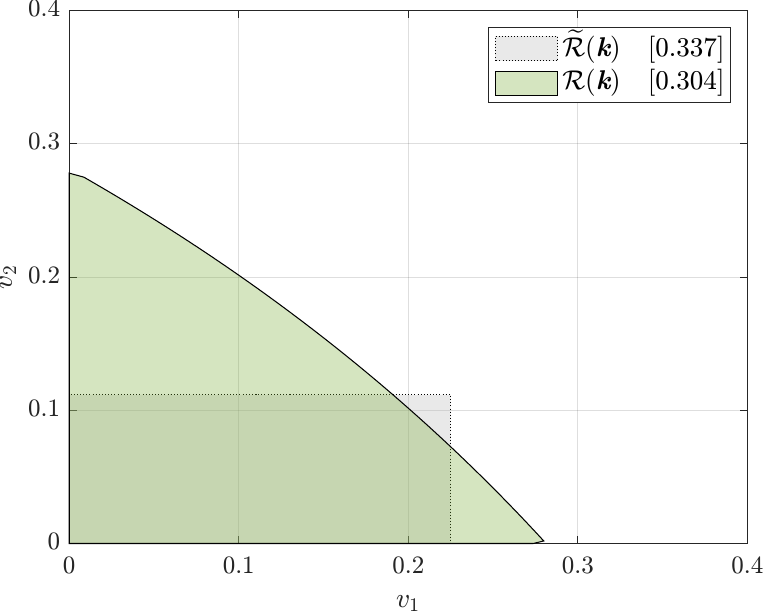}
		\caption{\centering Diagonal allocation\\ $\bH = \bN$}
		\label{fig:ex-region-lambda-diag-ub}
	\end{subfigure}
	\caption{Region $\cR(\bk)$ (in green) obtained from Corollary~\ref{corol:region_bound} for $m = 2$ appropriateness criteria, $\bN = [800 \; 1200]\T$, and different choices of $\{\lambda_{\bh} \}_{\bh = \bzero}^{\bH}$ (columns) and $\bH = 2\bN$ (top row) or $\bH = \bN$ (bottom row). Each plot also shows the box region $\tilde{\cR}(\bk)$ (in gray), obtained using the bound in~\eqref{eq:individual-independent-bound}. The maximum value of $|\bv|$ when $\bv$ belongs to each region is also reported in the legend within square brackets beside the corresponding entry.}
	\label{fig:ex-region-lambda}
\end{figure*}

Clearly, for a given $\bN$ and the confidence parameter $\beta \in (0,1)$, the shape of the region $\cR(\bk)$ depends on $\bk$, $\bH$, and on how the quantity $1-\beta$ is allocated among the variables $\{ \lambda_{\bh} \}_{\bh = \bzero}^{\bH}$. For the sake of illustration, we report in Figure~\ref{fig:ex-region-lambda} the different shapes of $\cR(\bk)$ (green regions with solid boundary) associated with a decision problem with $m = 2$ criteria, $\bN = [800 \; 1200]\T$, $\bk = [120 \; 80]\T$, and: $\bH = 2 \bN$ (top row), $\bH = \bN$ (bottom row). Results are presented for $\beta = 10^{-5}$, and for the following choices of $\{ \lambda_{\bh} \}_{\bh = \bzero}^{\bH}$ (each choice corresponds to a column in the figure).
\begin{description}
	\item[Uniform allocation] (left column):
	\begin{subequations} \label{eq:lambda-allocation-uniform}
	\begin{align}
		&\lambda_{\bh} = 1											&&\hspace{-5em} \bh = \bN \\
		&\lambda_{\bh} = -\frac{\beta}{(\bH + \bone)^{\bone}-1 } 	&&\hspace{-5em} \text{otherwise} 
	\end{align}
	\end{subequations}
	which ensures $\sum_{\bh = \bzero}^{\bH} \lambda_{\bh} = 1-\beta$ by distributing the $-\beta$ uniformly over all $\lambda_{\bh}$'s for which $\bh \neq \bN$. As can be seen from Figures~\ref{fig:ex-region-lambda-uniform} and~\ref{fig:ex-region-lambda-uniform-ub}, this choice of $\{ \lambda_{\bh} \}_{\bh = \bzero}^{\bH}$ produces rounded regions in the $\bv$-domain, either bounded both above and below along each coordinate axis when $\bH > \bN$ (cf. Figure~\ref{fig:ex-region-lambda-uniform}) or above only when $\bH = \bN$ (cf. Figure~\ref{fig:ex-region-lambda-uniform-ub}). \\
	\item[Axial allocation] (center column):
	\begin{subequations} \label{eq:lambda-allocation-axial}
	\begin{align}
		&\lambda_{\bh} = 1						&&\hspace{-5em} \bh = \bN \\
		&\lambda_{\bh} = -\frac{\beta}{|\bH|} 	&&\hspace{-5em} \bh \neq \bN \land \exists i : h_i = N_i \\
		&\lambda_{\bh} = 0 						&&\hspace{-5em} \text{otherwise} 
	\end{align}
	\end{subequations}
	which ensures $\sum_{\bh = \bzero}^{\bH} \lambda_{\bh} = 1-\beta$ by distributing the $-\beta$ only over the $\lambda_{\bh}$'s for which $h_i = N_i$ for some (but not all) $i$'s, and setting to zero the remaining ones. As can be seen from Figures~\ref{fig:ex-region-lambda-axial} and~\ref{fig:ex-region-lambda-axial-ub}, this choice of $\{ \lambda_{\bh} \}_{\bh = \bzero}^{\bH}$ produces rectangular regions with blunt corners in the $\bv$-domain, either bounded above and below along each coordinate axis when $\bH > \bN$ (cf. Figure~\ref{fig:ex-region-lambda-axial}) or above only for $\bH = \bN$ (cf. Figure~\ref{fig:ex-region-lambda-axial-ub}). \\
	\item[Diagonal allocation] (right column):
	\begin{subequations} \label{eq:lambda-allocation-diagonal}
	\begin{align}
		&\lambda_{\bh} = 1									&&\hspace{-2em} \bh = \bN \\
		&\lambda_{\bh} = -\frac{\beta}{|\cJ_{\bzero}|-1}	&&\hspace{-2em} \bh = \bN+j\bone,~ j \in \cJ_{\bzero} \setminus \{ 0 \} \\
		&\lambda_{\bh} = 0 									&&\hspace{-2em} \text{otherwise} 
	\end{align}
	\end{subequations}
	with $\cJ_{\bzero} = \{ j\in \bbZ :~ \max\{-\bN\} \le j \le \min\{\bH -\bN\} \}$, which ensures $\sum_{\bh = \bzero}^{\bH} \lambda_{\bh} = 1-\beta$ by distributing the $-\beta$ over the $\lambda_{\bh}$'s for which all components of $\bh$ differs from the corresponding components of $\bN$ by the same integer quantity $j \neq 0$, and setting to zero the remaining ones. As can be seen from Figures~\ref{fig:ex-region-lambda-diag} and~\ref{fig:ex-region-lambda-diag-ub}, this choice of $\{ \lambda_{\bh} \}_{\bh = \bzero}^{\bH}$ produces counter-diagonal banded regions in the $\bv$-domain, either limited both upward and downward along the diagonal direction for $\bH > \bN$ (cf. Figure~\ref{fig:ex-region-lambda-diag}) or upward only for $\bH = \bN$ (cf. Figure~\ref{fig:ex-region-lambda-diag-ub}).
\end{description}
Despite introduced for the case $m = 2$ only, the allocations in~\eqref{eq:lambda-allocation-uniform},~\eqref{eq:lambda-allocation-axial}, and~\eqref{eq:lambda-allocation-diagonal} are \emph{mutatis mutandis} valid for any number $m \in \bbN$ of appropriateness criteria, arguably generating similar features for $\cR(\bk)$ to those highlighted for $m=2$. The names of these three ways of allocating the confidence $1-\beta$ among the $\lambda_{\bh}$ variables refers to the sparsity pattern of the $m$-dimensional object $\{ \lambda_{\bh} \}_{\bh = \bzero}^{\bH}$. Indeed, for $m = 2$, the $\lambda_{\bh}$ variables can be arranged in a matrix $\bLambda$ with $H_1+1$ rows and $H_2+1$ columns, and the multi-index $\bh = [h_1 \; h_2]\T$ identifies the row $h_1+1$, column $h_2+1$ entry. In the \emph{uniform} allocation, all entries of $\bLambda$ are non zero, while in the \emph{axial} allocation only the entries in row $h_1 = N_1$ and in column $h_2 = N_2$ are non zero (they align with the horizontal/vertical axes of the matrix); in the \emph{diagonal} allocation only the entries along the diagonal passing through element $(h_1,h_2) = (N_1,N_2)$ are non zero.

\begin{figure*}
	\centering
	\begin{subfigure}{0.328\linewidth}
		\includegraphics[width=\linewidth]{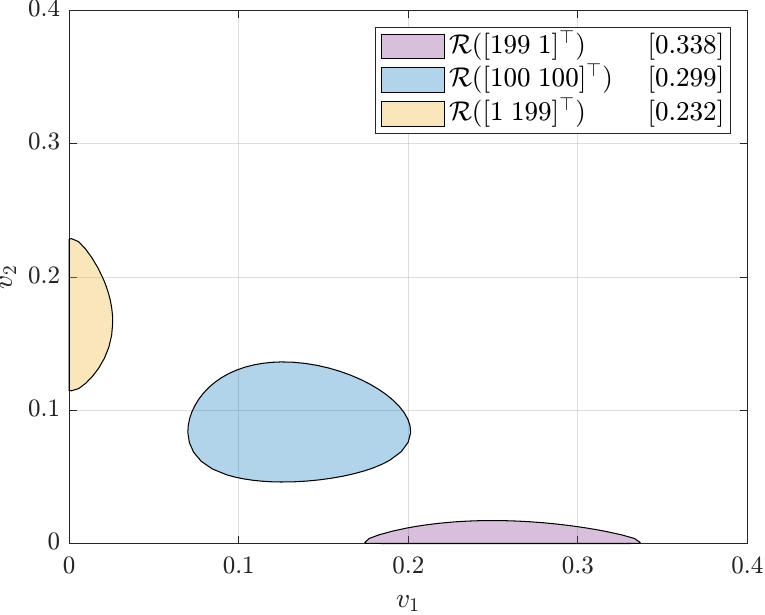}
		\caption{\centering Uniform allocation\\ $\bH = 2\bN$}
		\label{fig:ex-region-lambda-uniform-ki}
	\end{subfigure}
	\hfill
	\begin{subfigure}{0.328\linewidth}
		\includegraphics[width=\linewidth]{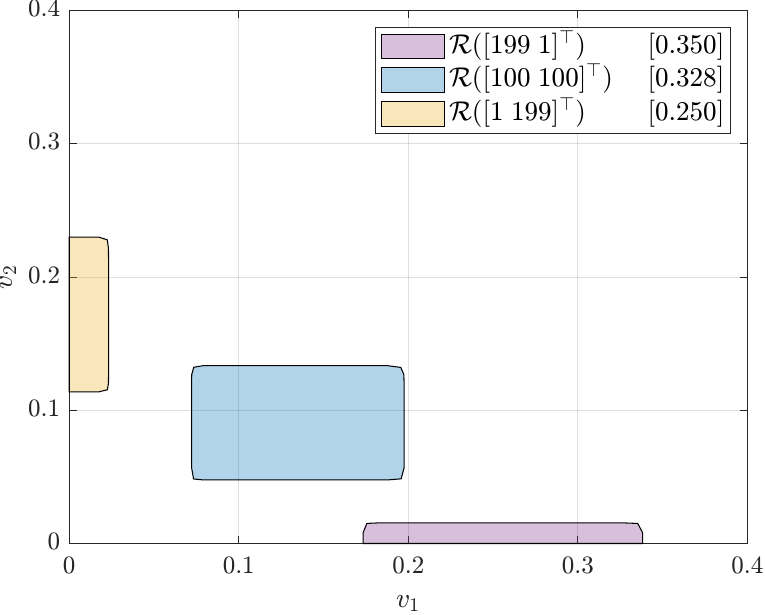}
		\caption{\centering Axial allocation\\ $\bH = 2\bN$}
		\label{fig:ex-region-lambda-axial-ki}
	\end{subfigure}
	\hfill
	\begin{subfigure}{0.328\linewidth}
		\includegraphics[width=\linewidth]{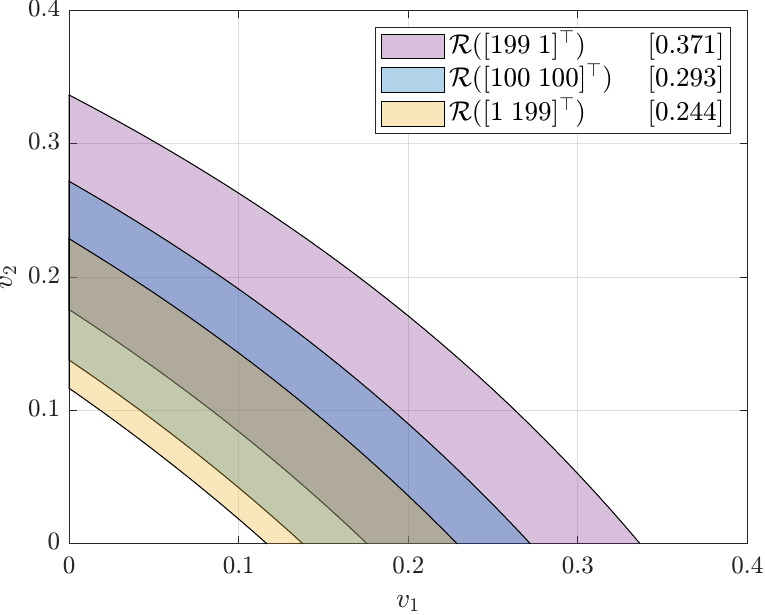}
		\caption{\centering Diagonal allocation\\ $\bH = 2\bN$}
		\label{fig:ex-region-lambda-diag-ki}
	\end{subfigure}
	\\[1em]
	\begin{subfigure}{0.328\linewidth}
		\includegraphics[width=\linewidth]{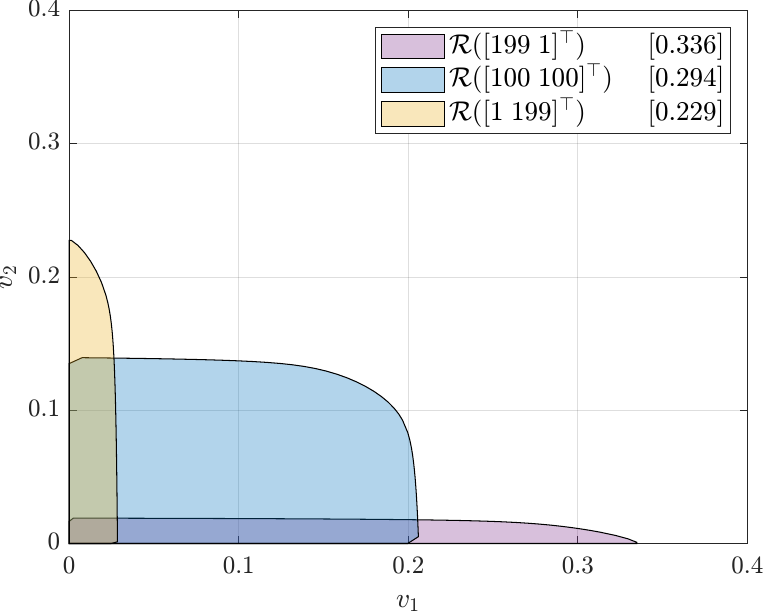}
		\caption{\centering Uniform allocation\\ $\bH = \bN$}
		\label{fig:ex-region-lambda-uniform-ub-ki}
	\end{subfigure}
	\hfill
	\begin{subfigure}{0.328\linewidth}
		\includegraphics[width=\linewidth]{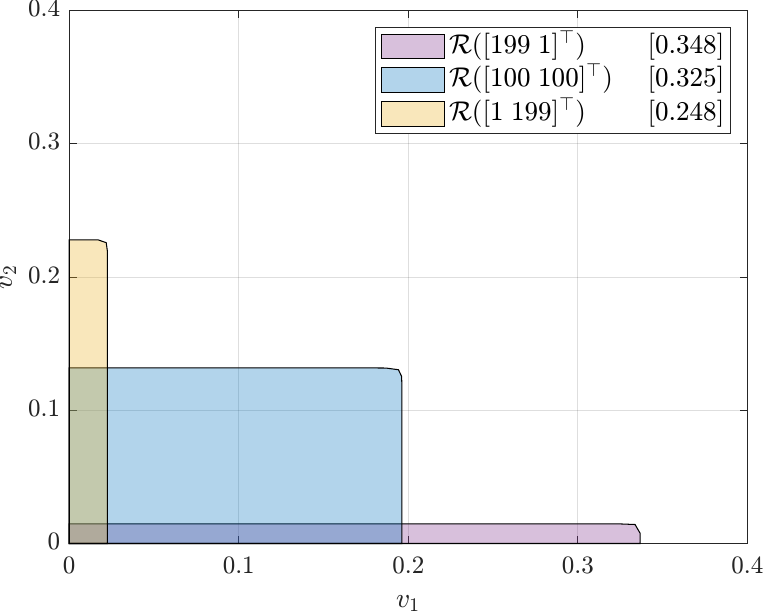}
		\caption{\centering Axial allocation\\ $\bH = \bN$}
		\label{fig:ex-region-lambda-axial-ub-ki}
	\end{subfigure}
	\hfill
	\begin{subfigure}{0.328\linewidth}
		\includegraphics[width=\linewidth]{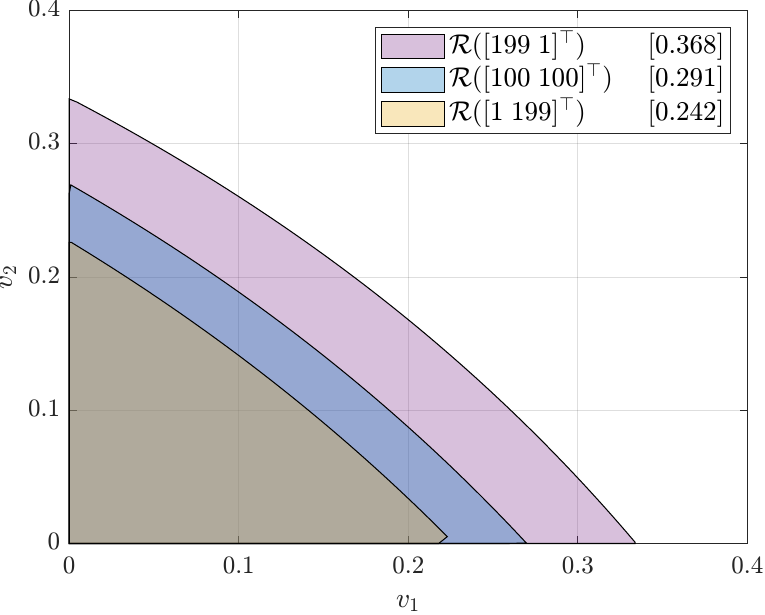}
		\caption{\centering Diagonal allocation\\ $\bH = \bN$}
		\label{fig:ex-region-lambda-diag-ub-ki}
	\end{subfigure}
	\caption{Region $\cR(\bk)$ obtained from Theorem~\ref{thm:region_bound} for $m = 2$ appropriateness criteria, $\bN = [800 \; 1200]\T$, and different choices of $\{\lambda_{\bh} \}_{\bh = \bzero}^{\bH}$ (columns), $\bH = 2\bN$ (top row) or $\bH = \bN$ (bottom row), and $\bk = [199 \; 1]\T$ (purple) or $\bk = [100 \; 100]\T$ (blue) or $\bk = [1 \; 199]\T$ (yellow). The maximum value of $|\bv|$ when $\bv$ belongs to each region is also reported in square brackets beside the corresponding legend entry.}
	\label{fig:ex-region-lambda-ki}
\end{figure*}

In Figure~\ref{fig:ex-region-lambda} we also report -- in gray, with dotted boundaries -- the box regions
\begin{equation} \label{eq:box-region}
	\tilde{\cR}(\bk) = \prod_{i=1}^m [\epstlb_i(\sstari),\epstub_i(\sstari)]
\end{equation}
that one obtains from~\eqref{eq:individual-independent-bound} using the independent assessment of the individual risks as discussed in Section~\ref{sec:independent_individual_bounds}. As can be seen from the comparison between the green and gray regions in Figure~\ref{fig:ex-region-lambda}, besides offering more flexibility in shaping the bounding regions, the proposed collective assessment of the individual risks in Corollary~\ref{corol:region_bound} is able to rule out some extreme risk vector values nearby the vertexes of $\tilde{\cR}(\bk)$, which the independent assessment considers instead possible within the same confidence level $1-\beta$.\footnote{This is also true for the axial allocation (center column) where the rectangular green region shows blunt corners, whereas the corners of the gray rectangle (almost completely hidden) are sharp; however, in this case the difference is minimal, and barely perceptible in the pictures.} This is a clear advantage of the collective bound in~\eqref{eq:region_bound} with respect to the independent bound in~\eqref{eq:individual-independent-bound}, and will be key in providing tighter estimates for the joint risk. To assess the advantage more quantitatively, the legend of each plot in Figure~\ref{fig:ex-region-lambda} also reports the quantities $\max_{\bv \in \cR(\bk)} |\bv|$ and $\max_{\bv \in \tilde{\cR}(\bk)} |\bv|$ as a representative worst-case. As it appears, the value corresponding to the collective assessment is always smaller than the value corresponding to the individual assessment, and the difference is significant for uniform and axial allocations; moreover, we will show that the gap grows with the number $m$ of appropriateness criteria.

We now discuss how the position and extent of the region are also affected by $\bk$. Recalling that, in the standard scenario theory of~\cite{garatti_RiskComplexityScenario_2022}, the ratio between complexity and number of scenarios is related to the risk magnitude, it is natural to expect that the component-wise ratio $\frac{\bk}{\bN}$ determines the position of the region within $[0,1]^m$; however, in the present setup $\bk$ may vary along various directions. To exemplify this, we consider again the set-up of Figure~\ref{fig:ex-region-lambda} and vary $\bk$ while keeping $|\bk| = 120+80 = 200$ fixed (this is to say that the overall number of support scenarios is fixed, but they are variously distributed between $\cD_1$ and $\cD_2$). Figure~\ref{fig:ex-region-lambda-ki} shows how the regions in Figure~\ref{fig:ex-region-lambda} change when $\bk = [199 \; 1]\T$ (purple); $\bk = [100 \; 100]\T$ (blue); and $\bk = [1 \; 199]\T$ (yellow). Even though we do not report a picture in the interest of space, it is worth mentioning that as $N_1$ and $N_2$ grows, while keeping the ratios $k_1 / N_1$ and $k_2 / N_2$ fixed, the regions shrinks in volume, making the assessment tighter and tighter. Moreover, when $k_1 = k_2$, the ratio between $N_1$ and $N_2$ determines the proportions of the region, which is more extended along the direction of the criteria with the smaller $N_i$ ($N_1$ in the case of Figure~\ref{fig:ex-region-lambda-ki}).

\subsection{Computational aspects}
The reduction in conservatism experienced using a collective assessment of the individual risks as opposed to an independent assessment of the individual risks comes at the price of an increased computational complexity. Indeed, to apply~\eqref{eq:individual-independent-bound}, one only has to find the two zeros of $m$ one-dimensional functions (cf. left hand side of~\eqref{eq:dual-standard-scenario-fcn} for each $\im$) to independently bound the individual risks, whereas Corollary~\ref{corol:region_bound} requires to find the 0-superlevel set of one $m$-dimensional function to obtain the region $\cR(\bk)$. The regions displayed in Figures~\ref{fig:ex-region-lambda}-\ref{fig:ex-region-lambda-ki} were brute-force computed, which however is not affordable for larger values of $m$, and how to compute efficiently $\cR(\bk)$ in the general case remains an open issue.

To partially overcome this limitation, the following Proposition~\ref{prop:region-bound-diag} shows that the computational complexity is greatly alleviated when $\{ \lambda_{\bh} \}_{\bh = \bzero}^{\bH}$ are set according to the diagonal allocation in~\eqref{eq:lambda-allocation-diagonal}, since in this case the corresponding region $\cR(\bk)$ admits an almost-explicit expression, regardless of $m$. Before presenting the result, we introduce some preliminary elements. Given $\bN$ and $\bH \ge \bN$, consider the family of scalar functions $\psi_{\bk,\bN,\bH}: \bbR \to \bbR$, $\bk = \bzero,\ldots,\bN$, as follows:\footnote{These functions are related to the left-hand side of~\eqref{eq:dual-lambda-constraint} when the $\lambda_{\bh}$'s are as in the diagonal allocation.}
\begin{equation} \label{eq:psi_def}
	\psi_{\bk,\bN,\bH}(t) = 1-\frac{\beta}{|\cJ_{\bzero}|-1}\sum_{j \in \cJ_{\bk} \setminus \{0\}} \frac{\binom{\bN+j\bone}{\bk}}{\binom{\bN}{\bk}} t^j,
\end{equation}
with $\cJ_{\bk} =\{\max\{\bk-\bN\}, \dots,\min\{\bH-\bN\}\} \subset \bbZ$. Then, for $\bk = \bzero,\ldots,\bN$, we can define the two scalars $\tbub_{\bk}$ and $\tblb_{\bk}$ as follows
\begin{subequations} \label{eq:psi_zeros}
	\begin{align}
		&\tbub_{\bk}
		=
		\begin{cases}
			\text{unique zero of $\psi_{\bk,\bN,\bH}$ in $(0,\hat{t})$ } \\
			0
		\end{cases}
		&&\hspace{-4em}
		\begin{aligned}
			&\bk < \bN \\
			&\text{otherwise}
		\end{aligned}
		\label{eq:psi_zero_ub} \\
		&\tblb_{\bk}
		=
		\begin{cases}
			\text{unique zero of $\psi_{\bk,\bN,\bH}$ in $(\hat{t},\infty)$ } \\
			1
		\end{cases}
		&&\hspace{-4em}
		\begin{aligned}
			&\bH > \bN \\
			&\text{otherwise}
		\end{aligned}
		\label{eq:psi_zero_lb}
	\end{align}
\end{subequations}
where $\hat{t} = (\bone-\frac{\bk}{\bN})^{\bone} \in [0,1]$. The fact that $\tbub_{\bk}$ and $\tblb_{\bk}$ are well-defined is shown in the auxiliary Lemmas~\ref{lem:psi-continous-concave-increasing} and~\ref{lem:zeros-psi} in Appendix~\ref{sec:proof-aux}, which study relevant features of the functions $\psi_{\bk,\bN,\bH}$.
\begin{prop} \label{prop:region-bound-diag}
	Let $\bN$ be any given multi-index in $\bbN^m$, $\bH \ge \bN$, and $\beta \in (0,1)$. Let
	\begin{equation} \label{eq:region_diag}
		\cR(\bk) = \{ \bv \in [0,1]^m :~ (\bone-\bv)^{\bone} \in [\tbub_{\bk}, \tblb_{\bk}] \}.
	\end{equation}
	with $\tbub_{\bk}$ and $\tblb_{\bk}$ defined in~\eqref{eq:psi_zeros}. Under Assumptions~\ref{ass:property} and~\ref{ass:non-degeneracy}, $\zsN = \cM(\bcD)$ satisfies
	\begin{equation} \label{eq:region_bound_diag}
		\bbPbN\{ \bV(\zsN) \in \cR(\bsstar) \} \ge 1 - \beta,
	\end{equation}
	where $\bsstar \in \bbN_0^m$ is the complexity of $\zsN$.
	\sqed
\end{prop}
Proposition~\ref{prop:region-bound-diag} is essentially obtained by showing that, under a diagonal allocation, $\cR(\bk)$ for each $\bk$ is the region in~\eqref{eq:region_diag} lying between two $m$-dimensional rectangular hyperbolas.
Note that we use the upper bar in $\tbub_{\bk}$ because, despite being the smaller zero of~\eqref{eq:psi_def}, it controls the upper hyperbola; similarly, we used $\tblb_{\bk}$ because it controls the lower one. Note also that, if $\tblb_{\bk} > 1$ for some $\bk$ when $\bH > \bN$, this is equivalent to setting $\tblb_{\bk} = 1$, since $(\bone-\bv)^{\bone} \le 1$ for all $\bv \in [0,1]^m$.

The computational benefit of Proposition~\ref{prop:region-bound-diag} with respect to the general case of Theorem~\ref{thm:region_bound} is evident, since all the effort lies in the computation of $\tbub_{\bk}$ and $\tblb_{\bk}$. These are zeros of a continuous (see Lemma~\ref{lem:psi-continous-concave-increasing} in Appendix~\ref{sec:proof-aux}) one-dimensional function, and can be efficiently obtained using any root-finding method for continuous functions; for instance, the bisection method, with $[0,\hat{t}]$ as starting interval for $\tbub_{\bk}$ and $[\hat{t},1]$ for $\tblb_{\bk}$. See Appendix~\ref{appendix:matlab} for a ready-to-use MATLAB code.

\section{Joint risk assessment} \label{sec:joint_risk}
We now turn our attention to estimating the joint risk $V(\zsN)$ of the scenario decision $\zsN = \cM(\bcD)$. We here focus on the more practical case in which the user is interested in deriving only an upper bound on the joint risk. Similarly to Section~\ref{sec:individual_risk}, we first show how the individual risk bounds obtained in the independent assessment we discussed in Section~\ref{sec:independent_individual_bounds} can be used to obtain an easy, albeit conservative, evaluation of the joint risk, and then present a significantly improved estimate based on the results discussed in Section~\ref{sec:joint_individual_bounds}.

\subsection{Evaluation of the joint risk based on the independent assessment of individual risks} \label{sec:joint_independent_bound}
An upper bound on the joint risk $V(\zsN)$ can be easily derived leveraging the sub-additivity property of $\bbP$. Indeed, by definition of individual and joint risks (cf.~\eqref{eq:local-risks} and~\eqref{eq:global-risk}), for any decision $z \in \cZ$, we have
\begin{align}
	V(z)
	&= \bbP \{ \delta\in\Delta :~ z \notin \cZ_1(\delta) \lor \cdots \lor z \notin \cZ_m(\delta) \} \nonumber \\
	&\le \sumim \bbP \{ \delta\in\Delta :~ z \notin \cZ_i(\delta) \} = \sumim V_i(z). \label{eq:sub-additivity}
\end{align}
Therefore, combining~\eqref{eq:sub-additivity} with the independent assessment of the individual risks in~\eqref{eq:individual-independent-bound}, we immediately obtain
\begin{subequations} \label{eq:joint-independent}
\begin{equation} \label{eq:joint-independent-bound}
	\bbPbN \{ V(\zsN) \le \epstub(\bsstar) \} \ge 1 - \beta,
\end{equation}
with
\begin{equation} \label{eq:joint-independent-eps}
	\epstub(\bk) = \sumim \epstub_i(k_i).
\end{equation}
\end{subequations}
Note that $\epstub(\bk)$ is also equal to $\max_{\bv \in \tilde{\cR}(\bk)} |\bv|$, where $\tilde{\cR}(\bk)$ is as in~\eqref{eq:box-region}. That is, the bound on the joint risk is obtained by considering the worst case individual risks corresponding to the upper-most vertex of $\tilde{\cR}(\bk)$ (see the gray regions in Figure~\ref{fig:ex-region-lambda}). These extreme values are excluded when resorting to regions $\cR(\bk)$ as in Theorem~\ref{thm:region_bound} and the discussion in Section~\ref{sec:joint_individual_bounds} supports the intuition that $\epstub(\bk)$ can be overly conservative.

\subsection{Evaluation of the joint risk based on the collective assessment of individual risks} \label{sec:joint_bound}
The very same reasoning we presented in the previous section can be applied to derive an upper bound on the joint risk leveraging the collective assessment of the individual risks presented in Section~\ref{sec:joint_individual_bounds}. Indeed, combining~\eqref{eq:sub-additivity} with~\eqref{eq:region_bound} this time, we immediately have
\begin{equation} \label{eq:joint-coll-bound}
	\bbPbN \Big\{ V(\zsN) \le \max_{\bv \in \cR(\bsstar)} |\bv| \Big\} \ge 1 - \beta,
\end{equation}
where $\cR(\bsstar)$ is any region as per Corollary~\ref{corol:region_bound}, associated with a choice of $\{ \lambda_{\bh} \}_{\bh = \bzero}^{\bH}$ that is feasible for~\eqref{eq:dual-lambda} and satisfies $\sum_{\bh = \bzero}^{\bH} \lambda_{\bh} = 1-\beta$.

The experimental results reported in Section~\ref{sec:joint_individual_bounds} show that for certain choices of $\{ \lambda_{\bh} \}_{\bh = \bzero}^{\bH}$, the evaluation in~\eqref{eq:joint-coll-bound} can greatly improve upon that in~\eqref{eq:joint-independent} since $\max_{\bv \in \cR(\bsstar)} |\bv|$ can be much smaller than $\max_{\bv \in \tilde{\cR}(\bsstar)} |\bv|$. In particular, the counter-diagonal-band regions corresponding to the diagonal allocation of the $\lambda_{\bh}$'s appear well-suited to this purpose since their upper boundaries align closely with the level sets of $|\bv|$, especially when $\bk \ll \bN$ and the region is located nearby the origin. As a matter of fact, in the examples reported in Section~\ref{sec:joint_individual_bounds}, when the scenarios were more or less evenly distributed among the datasets $\cD_1,\ldots,\cD_m$ (green regions in Figure~\ref{fig:ex-region-lambda}, blue regions in Figure~\ref{fig:ex-region-lambda-ki}) the diagonal allocation gave the lowest value for $\max_{\bv \in \cR(\bk)} |\bv|$. Moreover, the diagonal allocation choice gives a remarkable computational benefit, thanks to the almost-explicit characterization of $\cR(\bk)$ given by~\eqref{eq:region_diag} in Proposition~\ref{prop:region-bound-diag}. Proposition~\ref{prop:joint-bound-diag} below leverages this result to give an evaluation of the joint risk in which $\max_{\bv \in \cR(\bk)} |\bv|$ is analytically calculated for all $\bk = \bzero,\dots,\bN$. Since only the upper boundary of $\cR(\bk)$ matters for computing $\max_{\bv \in \cR(\bk)} |\bv|$, we directly consider $\bH = \bN$, which provides control on the upper boundary while neglecting the lower one.
\begin{prop} \label{prop:joint-bound-diag}
	Let $\bN$ be any given multi-index in $\bbN^m$ and $\beta \in (0,1)$. Under Assumptions~\ref{ass:property} and~\ref{ass:non-degeneracy}, $\zsN = \cM(\bcD)$ satisfies
	\begin{equation} \label{eq:joint-bound_diag}
		\bbPbN\{ V(\zsN) \le \epsub(\bsstar) \} \ge 1 - \beta,
	\end{equation}
	where $\bsstar \in \bbN_0^m$ is the complexity of $\zsN$ and
	\begin{equation} \label{eq:eps-bound-diag}
		\epsub(\bk) = \min\{ m ( 1 - \sqrt[m]{\tbub_{\bk}} ), 1\},
	\end{equation}
	where $\tbub_{\bk}$ is defined as in~\eqref{eq:psi_zero_ub} with $\bH = \bN$.
	\sqed
\end{prop}

Table~\ref{tab:joint-collective-vs-indep} offers an initial comparison between the joint risk certification in Proposition~\ref{prop:joint-bound-diag} and that in~\eqref{eq:joint-independent}. Further striking evidence of the advantage of using $\epsub(\cdot)$ instead of $\epstub(\cdot)$ will be provided through the analysis in the next Sections~\ref{sec:apriori_joint_bound} and~\ref{sec:scalability_apriori_bounds}.
\begin{table}[t]
	\centering
	\caption{A comparison between $\epsub(\bk)$ and $\epstub(\bk)$ for some values of $m$, $\bN$, and $\bk$. All values are computed for $\beta = 10^{-7}$.}
	\setlength{\tabcolsep}{1em}			
	\renewcommand{\arraystretch}{1.4}	
	\begin{tabular}{cccc|cc}
		\hline
		$m$ 	& $\bN$ 				& $\bk$ 			& $|\bk|$ 	& $\epsub(\bk)$ & $\epstub(\bk)$ \\
		\hline
		$10$ 	& $1500 \cdot \bone$ 	& $4 \cdot \bone$ 	& $40$ 		& $0.0602$ 		& $0.216$ \\
		$40$ 	& $1500 \cdot \bone$ 	& $1 \cdot \bone$ 	& $40$ 		& $0.0604$ 		& $0.697$ \\
		$25$ 	& $1500 \cdot \bone$ 	& $2 \cdot \bone$ 	& $50$ 		& $0.0699$		& $0.474$ \\
		$25$ 	& $2000 \cdot \bone$ 	& $2 \cdot \bone$ 	& $50$ 		& $0.0525$ 		& $0.357$ \\
		$60$ 	& $1500 \cdot \bone$ 	& $1 \cdot \bone$ 	& $60$ 		& $0.0793$ 		& $1.062$ \\
		$100$ 	& $3000 \cdot \bone$ 	& $1 \cdot \bone$ 	& $100$ 	& $0.0575$ 		& $0.907$ \\
		\hline
	\end{tabular}
	\label{tab:joint-collective-vs-indep}
\end{table}

\subsection{A-priori joint risk assessment} \label{sec:apriori_joint_bound}
The theory developed in the previous sections makes it possible to assess the joint risk of the decision $\zsN$ \emph{a-posteriori}: once the datasets in $\bcD$ are given ($\bN$ is thus fixed), they are used to make the decision $\zsN$, and then $V(\zsN)$ is estimated based on the complexity $\bsstar$, which is computed based on $\bcD$ after obtaining $\zsN$. However, in practice, it is often the case that the user also needs an \emph{a-priori} guarantee that the joint risk $V(\zsN)$ will be small enough, before collecting the data or computing $\zsN$, just based on $\bN$ and $\beta$. This may serve, e.g., to properly size the number of scenarios for all criteria so as to be sure that a desirable risk threshold is met, or to decide whether it is worth proceeding with the design given a certain data collection capacity for each criterion. In general, a-priori certifications are impossible to obtain; they require that some information about $\bsstar$ is a-priori available, before collecting the data and computing $\zsN$.

In this paper, we consider the case in which an upper bound $\ktot$ on the total number of support scenarios is available.
\begin{assum} \label{assum:complexity_bounded}
	For every $\bN$, it holds that $|\bsstar| \le \ktot$, where $\bsstar$ is the complexity of $\zsN$. \sqed
\end{assum}
Assumption~\ref{assum:complexity_bounded} holds true, for example, when the decision $\zsN$ is the optimal solution of an optimization problem like~\eqref{eq:multiagent-opt} that is convex. In this situation, the total number of support scenarios is provably upper bounded by the dimension of the optimization space $\cZ$ (see, e.g.,~\cite{CalCam:05}).

When $|\bsstar| \le \ktot$, \emph{a-priori} upper bounds on the joint risk $V(\zsN)$ can be easily derived by considering the worst possible occurrence of $\bsstar$. Specifically, when we rely on an independent assessment of individual risks (cf.~\eqref{eq:joint-independent-bound} and~\eqref{eq:joint-independent-eps}) we have
\begin{subequations} \label{eq:apriori-indepenednet-joint-bound}
\begin{equation} \label{eq:apriori-joint-independent-bound}
	\bbPbN \{ V(\zsN) \le \epstubap(\ktot) \} \ge 1 - \beta,
\end{equation}
with
\begin{equation} \label{eq:apriori_independent_eps}
	\epstubap(\ktot) = \max_{ \bk : |\bk| \le \ktot} \epstub(\bk) = \max_{ \bk : |\bk| \le \ktot} \sumim \epstub_i(k_i),
\end{equation}
\end{subequations}
while resorting to the collective assessment of individual risks (cf.~\eqref{eq:joint-coll-bound}) the result becomes
\begin{equation} \label{eq:apriori-joint-bound}
	\bbPbN \Big\{ V(\zsN) \le \max_{ \bk : |\bk| \le \ktot} \max_{\bv \in \cR(\bk)} |\bv| \Big\} \ge 1 - \beta.
\end{equation}
The upper bounds on $V(\zsN)$ in~\eqref{eq:apriori-indepenednet-joint-bound} and~\eqref{eq:apriori-joint-bound} are \emph{a-priori} because they do not need to collect $\bcD$, obtain $\zsN = \cM(\bcD)$, and consequently compute $\bsstar$. Additionally, the usage of~\eqref{eq:apriori-indepenednet-joint-bound} and~\eqref{eq:apriori-joint-bound} can be reversed: one can solve for $\bN$ under the constraint that the bound does not exceed a threshold, thus enabling the user to appropriately select $\bN$ to ensure that $V(\zsN)$ is sufficiently small before computing $\zsN$.

Clearly, solving the maximization in~\eqref{eq:apriori-joint-bound} adds yet another layer of computational complexity as the user would need to evaluate $\max_{\bv \in \cR(\bk)} |\bv|$ for all $\bk$ such that $|\bk| \le \ktot$. Fortunately, the diagonal allocation choice for $\{ \lambda_{\bh} \}_{\bh=\bzero}^{\bH}$ provides yet another computational benefit, as the maximum over $\bk$ in~\eqref{eq:apriori-joint-bound} can be explicitly characterized as a function of $\ktot$. Specifically, when the size of the datasets in $\bcD$ are homogeneous (i.e., $\bN = N \bone$), the maximum over $\bk$ is provably achieved when $|\bk| = \ktot$ and all $\ktot$ support scenarios belong to a single dataset, say $\cD_1$. Then, the risk for the case of non-homogeneous datasets (generic $\bN$) can be bounded by the risk in the homogeneous dataset case using $\Nlb \bone$ with $\Nlb = \min\{\bN\}$. The previous discussion constitutes the core of the proof of the following proposition, which formalizes the resulting a-priori probabilistic bound.

\begin{prop} \label{prop:apriori-joint-bound-diag}
	Let $\bN$ be any given multi-index in $\bbN^m$ and $\beta \in (0,1)$. Under Assumptions~\ref{ass:property}, \ref{ass:non-degeneracy} and~\ref{assum:complexity_bounded}, $\zsN = \cM(\bcD)$ satisfies
	\begin{equation} \label{eq:apriori-joint-bound-diag}
		\bbPbN\{ V(\zsN) \le \epsubap(\ktot) \} \ge 1 - \beta,
	\end{equation}
	where
	\begin{subequations} \label{eq:apriori-eps-and-tktot-bound-diag}
	\begin{align}
		& \epsubap(\ktot)
		= \min\{ m (1 - \sqrt[m]{\tktot}), 1 \}, \label{eq:apriori-eps-bound-diag} \\
		& \tktot
		=
		\begin{cases}
			\text{unique zero of $\psi_{\ktot,\Nlb,\Nlb}$} &\ktot < \Nlb \\ 
			0 &\text{otherwise}
		\end{cases},
		\label{eq:psi_single_zero_ub} \\
		& \psi_{\ktot,\Nlb,\Nlb}(t) = 1-\frac{\beta}{\Nlb}\sum_{j = 1}^{\Nlb - \ktot} \frac{\binom{\Nlb - j}{\ktot}}{\binom{\Nlb}{\ktot}} t^{-j}, \label{eq:psi_single_def}
	\end{align}
	\end{subequations}
	and $\Nlb = \min\{\bN\}$.\footnote{Note that $\psi_{\ktot,\Nlb,\Nlb}(t)$ is the function in~\eqref{eq:psi_def} when $\bN$, $\bH$ and $\bk$ are scalars and $\bH = \bN = \Nlb$, while $\bk = \ktot$.} \sqed
\end{prop}
Proposition~\ref{prop:apriori-joint-bound-diag} introduces a certain level of conservativeness when the datasets are non-homogeneous, the support scenarios are not $\ktot$, or even when they are $\ktot$ but they do not all belong to the same dataset for a single criterion. The degree of conservativeness introduced depends on how far $\bN$ is from $\Nlb \bone$ (which can be assessed a-priori) and how far $\bsstar$ is from having $\ktot$ support scenarios belonging to a single dataset (which cannot be assessed a-priori).

\begin{rem}[Dataset sizing] \label{rem:apriori_usage}
	We make explicit the usage of Proposition~\ref{prop:apriori-joint-bound-diag} to size the datasets for the various criteria to ensure with confidence $1-\beta$ that the joint risk is no bigger than a user chosen level $\eps \in (0,1)$. For given $m$ and $\ktot$, in view of~\eqref{eq:apriori-joint-bound-diag}, it is enough to take $\bN = \Nlb \cdot \bone$ and select $\Nlb$ so that $\epsubap(\ktot) \le \eps$. This amounts to find $\Nlb$, with $\Nlb > \ktot$, such that $m (1 - \sqrt[m]{\tktot}) \le \eps$, or equivalently $\tktot \ge (1 - \frac{\eps}{m})^m$, where $\tktot$ is the unique zero in $(0,1)$ of $\psi_{\ktot,\Nlb,\Nlb}$ (see~\eqref{eq:psi_single_def}). A suitable value of $\Nlb$ can always be found, because for $\Nlb$ large enough we can make $\tktot$ as close as $1$ as desired.\footnote{To see this note that $\psi_{\ktot,\Nlb,\Nlb} \le 1 - \beta \Nlb^{-1} {\Nlb \choose \ktot}^{-1} t^{\ktot - \Nlb}$, since $\psi_{\ktot,\Nlb,\Nlb}$ consists of the right-hand side with the addition of negative terms. This gives $1 \ge \tktot \ge \big(\beta \Nlb^{-1} {\Nlb \choose \ktot}^{-1} \big)^{\frac{1}{\Nlb -\ktot}}$, being this latter the solution to $1 - \beta \Nlb^{-1} {\Nlb \choose \ktot}^{-1} t^{\ktot - \Nlb} = 0$. Taking the logarithm, we have that $0 \ge \ln(\tktot) \ge \ln\big(\beta \Nlb^{-1} {\Nlb \choose \ktot}^{-1} \big) / (\Nlb -\ktot)$ and the last term tends to $0$ as $\Nlb \to \infty$. This shows that $\tktot \to 1$ as $\Nlb \to \infty$.} Since $\tktot$ can be efficiently computed, then $\Nlb$ can be easily obtained as well, e.g., by trial and error. \sqed
\end{rem}

\begin{figure*}[ht]
	\begin{subfigure}{0.47\linewidth}
		\includegraphics[width=\linewidth]{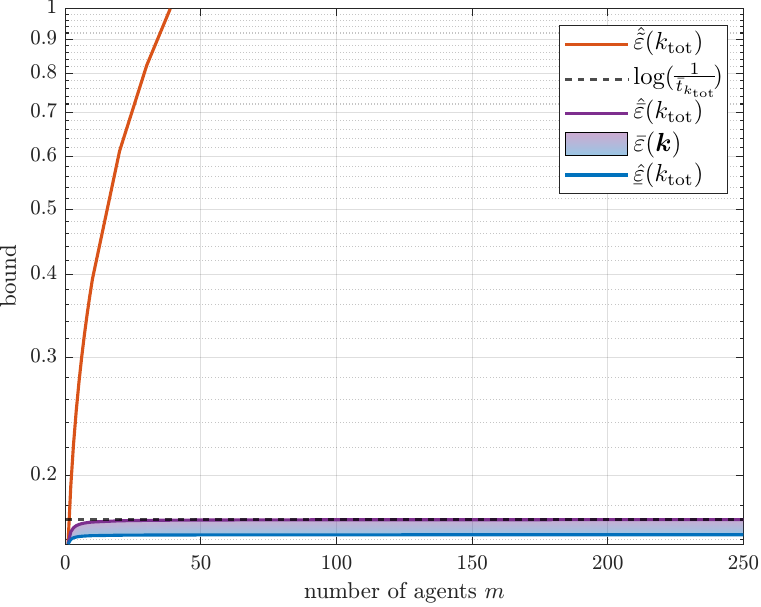}
		\caption{\centering A-priori bounds comparison}
		\label{fig:comparison-all}
	\end{subfigure}
	\hfill
	\begin{subfigure}{0.486\linewidth}
		\includegraphics[width=\linewidth]{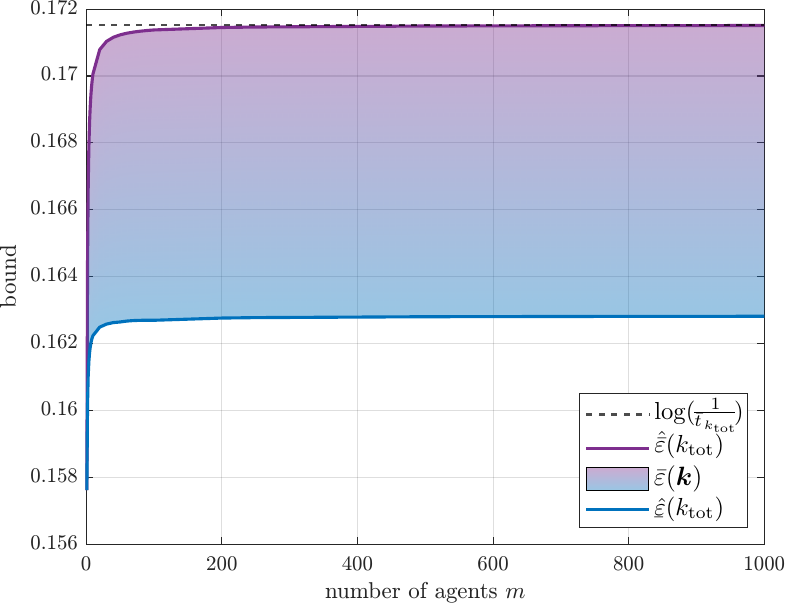}
		\caption{\centering Proposed bounds on a wider range of $m$ }
		\label{fig:comparison-zoom}
	\end{subfigure}
	\caption{Comparison between the different a-priori joint risk bounds discussed in Section~\ref{sec:scalability_apriori_bounds}, as a function of $m$.}
	\label{fig:comparison}
\end{figure*}

\subsection{Scalability of the a-priori bounds as $m$ grows} \label{sec:scalability_apriori_bounds}
In many practical cases, we may have a large number $m$ of appropriateness criteria. It is therefore interesting to study how well the developed risk bounds scale with $m$. Given their practical relevance, we limit the scalability analysis to the a-priori bounds developed in Section~\ref{sec:apriori_joint_bound}. Therefore, we assume that the complexity is bounded by $\ktot$ and, moreover, that $\ktot$ remains the same irrespective of $m$. This is e.g. the case of convex optimization when each criterion is associated with a new type of constraint, but the decision variables and their dimension are fixed irrespective of the criteria considered (in multi-agent optimization, this corresponds to a consensus problem with a common decision vector). For simplicity, we consider a setup in which $N_1 = \Nlb$ for some $\Nlb \ge \ktot$ and $N_i \ge \Nlb$ $\forall i \in \bbN$, so that $\min_{i \in \{1,\dots,m\}} N_i = \Nlb$ remains constant as $m$ varies. Moreover, we also assume that $N_i \le \Nub$ $\forall i \in \bbN$. This corresponds to the common situation in which for each criterion one can secure at least a minimum amount of scenarios (i.e., data), but at the same time there is not an unlimited abundance of them and $N_i$ cannot grow unbounded as $i$ increases.

From~\eqref{eq:apriori_independent_eps}, it is easy to see that in the present setup the a-priori bound $\epstubap(\ktot)$ obtained leveraging the independent assessment of individual risks scales linearly with $m$. Indeed, no matter how $\ktot$ is distributed among the $k_i$, we always have $\epstub_i(k_i) \ge \epstub_i(0) > c$, where $c$ is a strictly positive constant.\footnote{Provably, see e.g. \cite{CamGa2023}, $\epstub_i(0)$ is always positive and converging to $0$ as $N_i \to \infty$; $\epstub_i(0) > c > 0$ follows from condition $N_i \leq \Nub$.} Hence, the summation in~\eqref{eq:apriori_independent_eps} will keep increasing as $m$ increases, faster than $c \cdot m$ and eventually exceeding $1$, which renders the joint risk bound in~\eqref{eq:apriori-joint-independent-bound} useless. This is in line with the results in~\cite[Fig.~2]{falsone_ScenariobasedApproachMultiagent_2020} where a linear increase of the bound as a function of $m$ is reported.

If we now consider the bound $\epsubap(\ktot)$ in Proposition~\ref{prop:apriori-joint-bound-diag}, obtained by leveraging the collective assessment of the individual risks, we can first note that the function $\min\{ m (1 - \sqrt[m]{t}), 1 \}$, $t \in [0,1]$, appearing in the right-hand side of~\eqref{eq:apriori-eps-bound-diag} converges point-wise for $m \to \infty$, with the limit providing an upper bound for the function for every $m$. This underlies the following result.
\begin{prop} \label{prop:apriori-uniform-joint-bound-diag}
	Let $\bN$ be any given multi-index in $\bbN^m$ and $\beta \in (0,1)$. Under Assumptions~\ref{ass:property}, \ref{ass:non-degeneracy} and~\ref{assum:complexity_bounded}, $\zsN = \cM(\bcD)$ satisfies
	\begin{equation} \label{eq:apriori-uniform-joint_bound_diag}
		\bbPbN\{ V(\zsN) \le \min\{ \log(\textstyle \frac{1}{\tktot}), 1 \} \} \ge 1 - \beta,
	\end{equation}
	where $\tktot$ is defined as in~\eqref{eq:psi_single_zero_ub} and with the convention that if $\tktot = 0$ then $\min\{ \log(\textstyle \frac{1}{\tktot}), 1 \} = 1$.
	\sqed
\end{prop}
In the setting where $\ktot$ and $\Nlb$ do not depend on $m$, it is clear from~\eqref{eq:psi_single_def} that $\psi_{\ktot,\Nlb,\Nlb}(t)$, and hence $\tktot$ in~\eqref{eq:psi_single_zero_ub}, are also not affected by the value of $m$. Thus, Proposition~\ref{prop:apriori-uniform-joint-bound-diag} provides a probabilistic bound on the joint risk $V(\zsN)$ that holds uniformly with respect to the number $m$ of appropriateness criteria. Despite being interesting per-se, this also shows that the bound on the joint risk provided by Proposition~\ref{prop:apriori-joint-bound-diag} eventually reaches a limit as $m$ grows (in contrast with the linear trend of $\epstubap(\ktot)$). This result is a, not immediate, consequence of Theorem~\ref{thm:region_bound}, and reveals a clear advantage of the collective assessment of individual risks over the independent one.
\begin{rem}
	Likewise Proposition~\ref{prop:apriori-joint-bound-diag}, also Proposition~\ref{prop:apriori-uniform-joint-bound-diag} lends itself to be used to size the datasets to ensure a desired level of the joint risk. The procedure remains the same as in Remark~\ref{rem:apriori_usage} with the only difference that condition $\tktot \ge e^{-\eps}$ has to be enforced in place of $\tktot \ge (1 - \frac{\eps}{m})^m$. \sqed
\end{rem}

To visually support the discussion on the scalability of the different bounds, a numerical study is in order. We consider a decision problem with a variable number $m$ of appropriateness criteria. We assume that we have homogeneous datasets with $N_i = N = 1000$ samples for each $\cD_i$, $\im$. We select $\beta = 10^{-5}$ and we assume an upper bound on the total number of support scenarios equal to $\ktot = \frac{1}{10} N = 100$. As for the $\epstubap(\ktot)$ bound based on the independent assessment of the individual risks and computed according to~\eqref{eq:apriori_independent_eps}, we set $\beta_i = \frac{\beta}{m}$ and we compute $\epstub_i(k_i)$ based on the zero of the left hand side of~\eqref{eq:dual-standard-scenario-fcn} (cf. Theorem~\ref{thm:one_criterion_bound}) with $H = N$, $\lambda_{N} = 1$, and $\lambda_{h} = -\frac{\beta_i}{N}$ for all $h = 0,\dots,N-1$. As for the $\epsubap(\ktot)$ bound based on the collective assessment of the individual risks and computed according to~\eqref{eq:apriori-eps-and-tktot-bound-diag} (cf. Proposition~\ref{prop:apriori-joint-bound-diag}) we simply set $\Nlb = N$ since we are in a homogeneous set-up. To show the level of conservativeness introduced by considering the worst $\bk$ satisfying $|\bk| \le \ktot$ in the development of Proposition~\ref{prop:apriori-joint-bound-diag}, we also report the value of
\begin{equation} \label{eq:apriori-eps-bound-diag-minbk}
	\epslbap(\ktot) = \min_{ \bk : |\bk| = \ktot} \epsub(\bk),
\end{equation}
which can be viewed as a best-case distribution of the support scenarios across the datasets in $\bcD$, when the total number of support scenarios is kept at $\ktot$. The value of $\epsub(\bk)$ will thus always lie in the interval $[\epslbap(\ktot),\epsubap(\ktot)]$ when $|\bk| = \ktot$.

In Figure~\ref{fig:comparison-all} we show how $\epstubap(\ktot)$ (red line), $\epsubap(\ktot)$ (purple line), and $\epslbap(\ktot)$ (blue line), grow as the number of agents $m$ ranges from $1$ to $250$. If $m = 1$, then all bounds have the same value of approximately $0.158$ and coincide with the standard scenario bound for one criterion of appropriateness. However, as $m$ increases, we clearly see that $\epstubap(\ktot)$ grows much faster than $\epsubap(\ktot)$ and $\epslbap(\ktot)$, exceeding 1 for $m \ge 50$. On the contrary $\epsubap(\ktot)$ and $\epslbap(\ktot)$ increase only slightly and quickly settle to a value of approximately $0.172$ and $0.163$, roughly $9\%$ and $3\%$ bigger than that with $m = 1$, respectively. The shaded region in between $\epsubap(\ktot)$ and $\epslbap(\ktot)$ is where $\epsub(\bk)$ must stay when $|\bk| = \ktot$. In the plot, we also report the (uniform over $m$) upper bound given by Proposition~\ref{prop:apriori-uniform-joint-bound-diag} (black dashed line). Figure~\ref{fig:comparison-zoom} provides a zoom to better appreciate the difference between these bounds for values of $m$ up to 1000. Remarkably, $\epsubap(\ktot)$ and $\epslbap(\ktot)$ are very close to each other, showing that little conservatism is added when considering the worst-case over all $\bk$ such that $|\bk| = \ktot$. Moreover, after approximately $m = 50$, $\epsubap(\ktot)$ is very close to the uniform bound provided by Proposition~\ref{prop:apriori-uniform-joint-bound-diag}, meaning that only a little conservatism is introduced in considering the uniform bound for large values of $m$.

\section{Conclusions} \label{sec:conclusions}
This work extends the scenario approach for data-driven decision making to a multi-criteria setting, where appropriateness of the decision to each criterion is imposed using a different dataset. By adopting a multi-criteria notion of complexity, we were able to provide a collective assessment of the individual risks of obtaining a decision that is not appropriate, with a tighter risk certification with respect to using an independent assessment. Based on such a collective assessment, we derived certificates for the joint risk and also provided an a-priori bound in the case in which complexity is known not to exceed a given threshold. If such a threshold keeps constant as the number of criteria grows, we were able to show that the joint risk is uniformly bounded by a constant, which helps in datasets sizing. Admittedly, our derivations require a non-degeneracy assumption. In the standard single-criterion scenario theory, this assumption can be relaxed while still maintaining the upper bound on the risk. We are then confident that our theory could also be extended to degenerate decision-making problems. However, this requires further investigations and is left for future work.


\section*{Conflict of interest}
The authors declare that they have no conflict of interest.

\bibliographystyle{spmpsci}		
\bibliography{biblio}			

\appendix
\section{Proof of Theorem~\ref{thm:region_bound}} \label{sec:proof-thm-region}
For any $m \in \bbN$, and $\bN \in \bbN^m$, we are interested in finding a lower bound for the probability that the vector of individual risks $\bV(\zsN)$ belongs to the given region $\cR(\bsstar)$, i.e.,
\begin{equation} \label{eq:region_prob}
	\bbPbN\{ \bV(\zsN) \in \cR(\bsstar) \},
\end{equation}
where $\bsstar \in \bbN_0^m$ is the complexity of the decision $\zsN = \cM(\bcD)$.

We start by noting that for any pair of multi-indices $\bk \in \bbN^m$ and $\bk' \in \bbN^m$, the events $\{\bsstar = \bk \}$ and $\{\bsstar = \bk' \}$ are disjoint if and only if $\bk \neq \bk'$. Moreover, since $\bsstar \le \bN$ by definition of support scenarios, we can rewrite~\eqref{eq:region_prob} as
\begin{equation} \label{eq:proof-thm1-PA-1}
	\sum_{\bk=\bzero}^{\bN} \bbPbN\{ \bV(\zsN) \in \cR(\bk) \land \bsstar=\bk\}.
\end{equation}

Let us define the map
\begin{equation*}
	\supp:
	\bigcup_{\bN\in\bbN^m} \Delta_1^{N_1}\times \cdots \times \Delta_m^{N_m}
	\to
	\bigcup_{\bk \in \bbN_0^m} \Delta_1^{k_1} \times \cdots \times \Delta_m^{k_m}
\end{equation*}
that for each $\bcD$ returns the list of lists of support scenarios for the decision $\zsN = \cM(\bcD)$ in each dataset $\cD_i$, $\im$. Then, we can rewrite~\eqref{eq:proof-thm1-PA-1} as
\begin{align}
	\sum_{\bk=\bzero}^{\bN} \bbPbN \Big\{ &\bigcup_{I_{\bk} \in \cI_{\bk}} \{\bV(\zsN) \in \cR(\bk) \land \supp(\bcD)=I_{\bk}\}\Big\} \nonumber \\
	&= \sum_{\bk=\bzero}^{\bN} \sum_{I_{\bk} \in \cI_{\bk}} \bbPbN \big\{ \bV(\zsN) \in \cR(\bk) \land \supp(\bcD)=I_{\bk} \big\}
\end{align}
where $\cI_{\bk}$ is the set of all lists $I_{\bk} = ( \scen{1}{j_{1,1}} ,\dots, \scen{1}{j_{1,k_1}} \allowbreak,\dots,\scen{m}{j_{m,1}}, \dots, \scen{m}{j_{m,k_m}} )$ obtained by selecting the $j_{i,1}\text{-th},\dots,j_{i,k_i}\text{-th}$ samples from dataset $\cD_i$ in $\bcD$, for all $\im$; the equality holds since the events inside the union forms a partition of the event $\{ \bV(\zsN) \in \cR(\bk) \land \bsstar = \bk\}$.

Owing to the i.i.d. setup -- which is captured by the product structure of $\bbPbN$ -- and under the permutation invariance property of $\cM$ granted by Assumption~\ref{ass:property}\ref{property-permutation invariance}, for a given $\bk$ all events $\{\bV(\zsN)\in\cR(\bk) \land \supp(\bcD) = I_{\bk}\}$ have the same probability regardless of how $I_{\bk}$ is chosen from $\cI_{\bk}$. This gives
\begin{equation} \label{eq:proof-VinR-PtE}
	\bbPbN\{\bV(\zsN) \in \cR(\bsstar)\} = \sum_{\bk=\bzero}^{\bN} \binom{\bN}{\bk} \bbPbN\{ \tilde{E} \},
\end{equation}
where the binomial coefficient $\binom{\bN}{\bk} = \Pi_{i=1}^m {N_i \choose k_i}$ counts the $I_{\bk}$'s in $\cI_{\bk}$ (that is, the combinations of $k_i$ elements out of the $N_i$ scenarios in each $\cD_i$) and $\tilde{E}$ is the event in which the $k_i$ support scenarios in each dataset $\cD_i$ are the first $k_i$ scenarios in $\cD_i$, i.e.,
\begin{equation*}
	\tilde{E} = \{ \bV(\cM(\bcD)) \in \cR(\bk) \land \supp(\bcD) = \bSk \}
\end{equation*}
with $\bSk = (\Ski{1},\dots,\Ski{m})$ and $\Ski{i} = ( \sceni{1},\dots,\sceni{k_i} )$ for all $\im$.

Instead of working with $\bbPbN\{\tilde{E}\}$ directly, we first show that $\tilde{E}$ is equal to
\begin{equation*}
	E = \big\{ \bV(\cM(\bSk))\in\cR(\bk) \land\supp(\bSk) = \bSk \land \cM(\bSk)\in\cZ_i(\delta), \delta \in \cD_i \setminus \Ski{i}, \im \big\},
\end{equation*}
up to a probability zero set. Then, $\bbPbN\{E\}$ will be evaluated in place of $\bbPbN\{\tilde{E}\}$.

We start by showing that $\tilde{E} \subseteq E$ up to a zero probability set. If $\bcD \in \tilde{E}$, then $\bV(\cM(\bcD)) \in \cR(\bk)$ and $\supp(\bcD) = \bSk$. Under Assumption~\ref{ass:non-degeneracy}, $\cM(\bcD) = \cM(\supp(\bcD))$ up to a zero probability set. Hence, with the exception of the zero-probability realizations of $\bcD$ for which the equality does not occur, the following implications hold true. First, we immediately have
\begin{subequations}
\begin{align}
	\supp(\bcD) = \bSk &\implies \cM(\bcD) = \cM(\bcS_{\bk}), \label{eq:SsuppD_implies_MD=MS} \\
	\bV(\cM(\bcD)) \in \cR(\bk) &\implies \bV(\cM(\bcS_{\bk})) \in \cR(\bk). \label{eq:MDinRk_implies_MSinRk}
\end{align}
\end{subequations}
For the sake of contradiction let us suppose that there exists $i \in \{1,\dots,m\}$ and $\delta \in \cD_i \setminus \Ski{i}$ such that $\cM(\bSk)\not\in\cZ_i(\delta)$. Under Assumption~\ref{ass:property}\ref{property-responsiveness} we would have $\cM(\bcD) \neq \cM(\bSk)$, which contradicts the right-hand side of ~\eqref{eq:SsuppD_implies_MD=MS}. Therefore,
\begin{equation} \label{eq:MD=MS_implies_MSappropriate}
	\text{\eqref{eq:SsuppD_implies_MD=MS}} \implies \cM(\bSk)\in\cZ_i(\delta) ,~ \delta \in \cD_i \setminus \Ski{i} ,~ \im.
\end{equation}
It remains to show that $\supp(\bSk) = \bSk$. Let $\supp(\bSk) = (\tSki{1},\dots,\tSki{m})$. By Definition~\ref{def:support-scenario}, $\tSki{i} \subseteq \Ski{i}$, for all $\im$. Again for the sake of contradiction let us suppose that there exists $i \in \{1,\dots,m\}$ such that $\tSki{i} \subset \Ski{i}$, and let $\delta \in \Ski{i} \setminus \tSki{i}$. Since $\delta$ is not a support scenario for $\bSk$,
\begin{equation} \label{eq:MSkmd=MD}
	\cM(\Ski{1},\dots,\Ski{i} \setminus (\delta),\dots,\Ski{m}) = \cM(\bSk) = \cM(\supp(\bcD)) = \cM(\bcD),
\end{equation}
where the second equality is by definition of $\bSk$ and the third equality holds owing to Assumption~\ref{ass:non-degeneracy}. However, since $\Ski{j} \cup (\cD_j \setminus \Ski{j}) = \cD_j$ for all $j \neq i$ and $(\Ski{i} \setminus (\delta)) \cup (\cD_i \setminus \Ski{i}) = \cD_i \setminus (\delta)$, by~\eqref{eq:MD=MS_implies_MSappropriate} and Assumption~\ref{ass:property}\ref{property-stability}, we would have
\begin{equation*}
	\cM(\cD_1,\dots,\cD_i \setminus (\delta),\dots,\cD_m) = \cM(\Ski{1},\dots,\Ski{i} \setminus (\delta),\dots,\Ski{m}) = \cM(\bcD),
\end{equation*}
the second equality being due to~\eqref{eq:MSkmd=MD}. This shows that $\delta$ would not be a support scenario for $\bcD$, i.e., $\delta \not\in \Ski{i}$, which is a contradiction. Hence, $\supp(\bSk) = \bSk$, which, together with~\eqref{eq:MDinRk_implies_MSinRk} and~\eqref{eq:MD=MS_implies_MSappropriate} shows that $\bcD \in E$ whenever $\cM(\bcD) = \cM(\supp(\bcD))$. Since this latter condition holds true up to a probability zero set by Assumption~\ref{ass:non-degeneracy}, it remains proven that $\tilde{E} \subseteq E$ up to a probability zero set.

Now, we prove that $E \subseteq \tilde{E}$ up to a probability zero set. If $\bcD \in E$, then $\bV(\cM(\bSk))\in\cR(\bk)$, $\supp(\bSk) = \bSk$, and $\cM(\bSk)\in\cZ_i(\delta)$ for all $\delta \in \cD_i \setminus \Ski{i}$ and $\im$. From this last condition together with Assumption~\ref{ass:property}\ref{property-stability}, using $\Ski{i} \cup (\cD_i \setminus \Ski{i}) = \cD_i$ for all $\im$, we immediately have
\begin{equation} \label{eq:MSappropriate_implies_MD=MS}
	\cM(\bSk) = \cM(\Ski{1},\dots,\Ski{m}) = \cM(\ms) = \cM(\bcD),
\end{equation}
and, hence,
\begin{equation} \label{eq:MSinRk_implies_MDinRk}
	\bV(\cM(\bcS_{\bk})) \in \cR(\bk) \implies \bV(\cM(\bcD)) \in \cR(\bk).
\end{equation}
If we repeat the same reasoning leading to~\eqref{eq:MSappropriate_implies_MD=MS} but with $\Ski{i} \cup ((\cD_i \setminus \Ski{i}) \setminus (\delta)) = \cD_i \setminus (\delta)$ for some $i \in \{1,\dots,m\}$ and $\delta \in \cD_i \setminus \Ski{i}$, we obtain
\begin{equation} \label{eq:MSappropriate_implies_suppD_subset_Sk}
	\cM(\bcD) = \cM(\bSk) = \cM(\Ski{1},\dots,\Ski{i},\dots,\Ski{m}) = \cM(\cD_1,\dots,\cD_i \setminus (\delta),\dots,\cD_m),
\end{equation}
where the first equality is due to~\eqref{eq:MSappropriate_implies_MD=MS}. This shows that any $\delta \in \cD_i \setminus \Ski{i}$ cannot be a support scenario; hence, $\supp(\bcD) = (\tcD_1,\dots,\tcD_m)$ is such that $\tcD_i \subseteq \Ski{i}$ for all $\im$. For the sake of contradiction let us suppose that there exists $i \in \{1,\dots,m\}$ such that $\tcD_i \subset \Ski{i}$, and take $\delta \in \Ski{i} \setminus \tcD_i$. Since $\delta$ is not a support scenario for $\cM(\cD)$
\begin{equation} \label{eq:delta_not_support}
	\cM(\cD_1,\dots,\cD_i \setminus (\delta),\dots,\cD_m) = \cM(\bcD),
\end{equation}
and, by contraposition of Assumption~\ref{ass:property}\ref{property-responsiveness},
\begin{equation} \label{eq:delta_not_support_implies_appropriateness}
	\cM(\bcD) = \cM(\cD_1,\dots,\cD_i \setminus (\delta),\dots,\cD_m) \in \cZ_i(\delta).
\end{equation}
Note that~\eqref{eq:delta_not_support_implies_appropriateness} holds for all $\delta\in \Ski{i} \setminus \tcD_i$.
Consider now a realization of $\bcD$ such that $\cM(\supp(\bcD)) = \cM(\bcD)$, which always occur except for for zero-probability cases according to Assumption~\ref{ass:non-degeneracy}. From~\eqref{eq:delta_not_support_implies_appropriateness} together with Assumption~\ref{ass:property}\ref{property-stability}, using $\tcD_i \cup ((\Ski{i} \setminus \tcD_i) \setminus (\delta)) = \Ski{i} \setminus (\delta)$ for some $i \in \{1,\dots,m\}$ and $\delta \in \Ski{i} \setminus \tcD_i$ and $\tcD_j \cup (\Ski{j} \setminus \tcD_j) = \Ski{j}$ for all $j \neq i$, we immediately have
\begin{equation} \label{eq:delta_not_support_implies_contradiction}
	\cM(\bcD) = \cM(\supp(\bcD)) = \cM(\tcD_1,\dots,\tcD_m) = \cM(\Ski{1},\dots,\Ski{i} \setminus (\delta),\dots,\Ski{m}) \neq \cM(\bSk),
\end{equation}
where the last inequality holds because $\delta \in \Ski{i}$ and $\supp(\bSk) = \bSk$, i.e., $\delta$ is a support scenario for $\bSk$. However,~\eqref{eq:delta_not_support_implies_contradiction} contradicts~\eqref{eq:MSappropriate_implies_MD=MS}. Hence, when $\cM(\supp(\bcD)) = \cM(\bcD)$, $\supp(\bcD) = \bSk$, which, together with~\eqref{eq:MSinRk_implies_MDinRk}, shows $\bcD \in \tilde{E}$. Thanks to Assumption~\ref{ass:non-degeneracy}, this proves that $E \subseteq \tilde{E}$ up to a probability zero set.

Since we have shown both $\tilde{E} \subseteq E$ and $E \subseteq \tilde{E}$, it holds that $\tilde{E} = E$ up to a probability zero set. Therefore, $\bbPbN\{\tilde{E}\} = \bbPbN\{E\}$, and substituting in~\eqref{eq:proof-VinR-PtE} yields
\begin{equation} \label{eq:proof-VinR-PE}
	\bbPbN\{\bV(\zsN) \in \cR(\bsstar)\} = \sum_{\bk=\bzero}^{\bN} \binom{\bN}{\bk} \bbPbN\{E\}.
\end{equation}

If we now rewrite $E$ as $E = E_r \cap E_s \cap E_a$ with
\begin{subequations}
\begin{align}
	&E_r = \{ \bV(\cM(\bSk)) \in \cR(\bk) \} \label{eq:Erisk} \\
	&E_s = \{ \supp(\bSk) = \bSk \} \label{eq:Esupp} \\
	&E_a = \bigcap_{i = 1}^m \bigcap_{\delta \in \cD_i \setminus \Ski{i}} \underbrace{\{ \cM(\bSk)\in\cZ_i(\delta) \}}_{E_{a,i,\delta}}, \label{eq:Eappr}
\end{align}
\end{subequations}
and noting that $E_r$ and $E_s$ depend on $\bSk$ only, the product structure of $\bbPbN\{E\}$ enables elaborating $\bbPbN\{E\}$ as follows:
\begin{align*}
	\bbPbN\{E\}	&= \int_{\Delta^{\bN}} \ind_E(\bcD) \; \bbPbN\{ \dd\bcD \} \\
	&= \int_{\Delta^{\bk}} \Bigg( \int_{\Delta^{\bN-\bk}} \ind_{E_a}(\bcD \setminus \bSk) \; \bbP^{|\bN-\bk|}\{ \dd(\bcD \setminus \bSk) \} \Bigg) \ind_{E_r \cap E_s}(\bSk) \; \bbP^{|\bk|}\{ \dd\bSk \},
\end{align*}
where we used $\Delta^{\balpha} = \Delta_1^{\alpha_1} \times \cdots \times \Delta_m^{\alpha_m}$ for any $\balpha \in \bbN_0^m$ and $\bcD \setminus \bSk = (\cD_1 \setminus \Ski{1}, \dots, \cD_m \setminus \Ski{m})$ as shorthands. In turn, the inner integral can be rewritten as
\begin{align*}
	\int_{\Delta^{\bN-\bk}} \ind_{E_a}(\bcD \setminus \bSk) \; \bbP^{|\bN-\bk|}\{ \dd(\bcD \setminus \bSk) \}
	&= \prodim \prod_{\delta \in \cD_i \setminus \Ski{i}} \int_{\Delta} \ind_{E_{a,i,\delta}}(\delta) \; \bbP\{ \dd \delta \} \\
	&= \prodim \prod_{\delta \in \cD_i \setminus \Ski{i}} (1-V_i(\cM(\bSk))) \\
	&= \prodim (1-V_i(\cM(\bSk)))^{N_i - k_i} \\
	&= (\bone-\bV(\cM(\bSk)))^{\bN - \bk}
\end{align*}
where we used~\eqref{eq:Eappr} in the first equality and the definition of individual risk (cf.~\eqref{eq:local-risks}) in the second equality. We thus have
\begin{equation} \label{eq:PE_ext_int_only}
	\bbPbN\{E\} = \int_{\Delta^{\bk}} (\bone - \bV(\cM(\bSk)))^{\bN - \bk} \ind_{E_r \cap E_s}(\bSk) \; \bbPN\{ \dd\bSk \}.
\end{equation}
For all $\bk \in \bbN_0^m$ and $\bv\in [0,1]^m$, define the generalized distribution function (see \cite[Section II.3]{shiryaev_Probability1_2016})
\begin{equation*}
	F_{\bk}(\bv) = \int_{\bV(\cM(\bSk)) \le \bv} \ind_{E_s}(\bSk) \; \bbP^{|\bk|}\{\dd\bSk\} .
\end{equation*}
Clearly, functions $\{F_{\bk}\}_{\bk\in\bbN_0^m}$ depend on the decision map $\cM$ and may differ for different decision-making problems, but we here omit the dependence from $\cM$ to ease the notation. Then, by a change of variables, $\bbPbN\{E\}$ in~\eqref{eq:PE_ext_int_only} can be expressed in terms of $F_{\bk}(\bv)$ as
\begin{equation}
	\bbPbN\big\{E\big\}
	= \int_{[0,1]^m} (\bone-\bv)^{\bN-\bk} \ind_{\cR(\bk)}(\bv) \; \dd F_{\bk}(\bv)
	= \int_{\cR(\bk)} (\bone-\bv)^{\bN-\bk} \; \dd F_{\bk}(\bv), \label{eq:pbdFk}
\end{equation}
where we used~\eqref{eq:Erisk}. Substituting~\eqref{eq:pbdFk} in~\eqref{eq:proof-VinR-PE} yields
\begin{equation} \label{eq:PVdFk}
	\bbPbN\{\bV(\zsN) \in \cR(\bsstar)\} = \sum_{\bk=\bzero}^{\bN} \binom{\bN}{\bk} \int_{\cR(\bk)} (\bone-\bv)^{\bN-\bk} \; \dd F_{\bk}(\bv),
\end{equation}
which provides a characterization of $\bbPbN\{\bV(\zsN) \in \cR(\bsstar)\}$ in terms of the ($\cM$-dependent) family $\{F_{\bk}\}_{\bk\in\bbN_0^m}$ of generalized distribution functions.

Relation~\eqref{eq:PVdFk} must hold irrespective of the specific choices for $\bN$ and $\cR(\bk)$. In particular, it must hold for any multi-index $\bh \in \bbN_0^m$ in place of $\bN$ and for $\cR(\bk) = [0,1]^m$ for all $\bk = \bzero,\dots,\bh$, yielding
\begin{equation} \label{eq:moment_cond}
	\sum_{\bk=\bzero}^{\bh} \binom{\bh}{\bk} \int_{[0,1]^m} (\bone-\bv)^{\bh-\bk} \dd F_{\bk}(\bv) = \bbP^{|\bh|}\{\bV(\zsN) \in [0,1]^m \} = 1,
\end{equation}
where the second equality is due to $\bV(\zsN)$ being a vector of probabilities. Equation~\eqref{eq:moment_cond} is a simple condition that $\{F_{\bk}\}_{\bk\in\bbN_0^m}$ must satisfy irrespective of $\cM$; it is actually sufficient to derive an informative lower bound for $\bbPbN\{\bV(\zsN) \in \cR(\bsstar)\}$, by minimizing the right hand side of~\eqref{eq:PVdFk} subject to~\eqref{eq:moment_cond}. More precisely, it holds that
\begin{equation} \label{eq:P>=eta}
	\bbPbN\{\bV(\zsN)\in \cR(\bk)\} \ge \eta
\end{equation}
where
\begin{subequations}\label{eq:primal-infinite-dimensional}
	\begin{align}
		\eta = \inf_{\{F_{\bk}\}_{\bk\in\bbN_0^m}} \quad
		&\sum_{\bk=\bzero}^{\bN} \binom{\bN}{\bk} \int_{\cR(\bk)} (\bone-\bv)^{\bN-\bk} \; \dd F_{\bk}(\bv) \label{eq:primal-infinite-dimensional-cost} \\
		\text{s.t.} \quad
		&\sum_{\bk = \bzero}^{\bh} \binom{\bh}{\bk} \int_{[0,1]^m} (\bone-\bv)^{\bh-\bk} \; \dd F_{\bk}(\bv) = 1
			&&\bh \in \bbN_0^m \label{eq:primal-infinite-dimensional-constraint} \\
		&F_{\bk} \in \cF
			&&\bk \in \bbN_0^m,
	\end{align}
\end{subequations}
and $\cF$ denotes the positive cone of generalized distribution functions. In problem~\eqref{eq:primal-infinite-dimensional}, the cost function depends on $\{F_{\bk}\}_{\bk=\bzero}^{\bN}$, while the constraints depends on $\{F_{\bk}\}_{\bk\in\bbN_0^m}$. For some $\bH \ge \bN$, we can then lower bound $\eta$ in~\eqref{eq:primal-infinite-dimensional} with
\begin{subequations} \label{eq:primal-truncated}
	\begin{align}
		\eta_{\bH} = \inf_{\{F_{\bk}\}_{\bk=\bzero}^{\bH}} \quad
		&\sum_{\bk=\bzero}^{\bN} \binom{\bN}{\bk} \int_{\cR(\bk)} (\bone-\bv)^{\bN-\bk} \; \dd F_{\bk}(\bv) \label{eq:primal-truncated-cost} \\
		\text{s.t.} \quad
		&\sum_{\bk = \bzero}^{\bh} \binom{\bh}{\bk} \int_{[0,1]^m} (\bone-\bv)^{\bh-\bk} \; \dd F_{\bk}(\bv) = 1
			&&\bh = \bzero,\dots,\bH \label{eq:primal-truncated-moment-constraint} \\
		&F_{\bk} \in \cF
			&&\bk = \bzero,\dots,\bH,
	\end{align}
\end{subequations}
because, by considering the functions $\{F_{\bk}\}_{\bk=\bzero}^{\bH}$ only with $\bH \ge \bN$, we have kept the cost function unchanged, while enlarging the feasible set.

We will now lower bound $\eta_{\bH}$ leveraging duality theory. Let $\bF = \{F_{\bk}\}_{\bk=\bzero}^{\bH} \in \cF^{(\bH+\bone)^{\bone}}$ and consider the collection $\blambda = \{ \lambda_{\bh} \}_{\bh = \bzero}^{\bH}$ of Lagrange multipliers, each one corresponding to a constraint in~\eqref{eq:primal-truncated-moment-constraint}. The dual problem of~\eqref{eq:primal-truncated} is
\begin{equation} \label{eq:dual-truncated}
	\gamma = \sup_{\blambda} \inf_{\bF\in\cF^{(\bH+\bone)^{\bone}}} \cL(\bF,\blambda),
\end{equation}
with
\begin{equation*}
	\cL(\bF,\blambda) = \sum_{\bk=\bzero}^{\bN} \binom{\bN}{\bk} \int_{\cR(\bk)} \! (\bone-\bv)^{\bN-\bk} \; \dd F_{\bk}(\bv) + \sum_{\bh = \bzero}^{\bH} \lambda_{\bh} \Big( 1 - \sum_{\bk = \bzero}^{\bh} \binom{\bh}{\bk} \int_{[0,1]^m} \! (\bone-\bv)^{\bh-\bk} \; \dd F_{\bk}(\bv) \Big)
\end{equation*}
being the Lagrangian. By weak duality, it holds that $\eta_{\bH} \ge \gamma$.

Moving the binomial coefficient inside the integral, the first term in the Lagrangian $\cL(\bF,\blambda)$ can be rewritten as
\begin{equation*}
	\sum_{\bk=\bzero}^{\bH} \int_{[0,1]^m} \binom{\bN}{\bk} (\bone-\bv)^{\bN-\bk} \ind_{\cR(\bk)}(\bv) \ind_{\cN}(\bk) \dd F_{\bk},
\end{equation*}
where $\cN = \{ \balpha \in \bbN_0^m : \balpha \le \bN \}$, and the sum is extended up to $\bH$ because the indicator function $\ind_{\cN}(\bk)$ sets to zero all terms in the sum for which $\bk \not\le \bN$. As for the second term in $\cL(\bF,\blambda)$, by expanding the product, exchanging the order of the summations, and bringing the inner sum, the binomial coefficients, and dual variables into the integral, we obtain
\begin{equation*}
	\sum_{\bh = \bzero}^{\bH} \lambda_{\bh} - \sum_{\bk = \bzero}^{\bH} \int_{[0,1]^m} \sum_{\bh = \bk}^{\bH} \lambda_{\bh} \binom{\bh}{\bk} (\bone-\bv)^{\bh-\bk} \; \dd F_{\bk}(\bv).
\end{equation*}
Given the reformulations of the two terms, the Lagrangian can be rewritten as
\begin{align*}
	\cL(\bF,\blambda) = &\sum_{\bh = \bzero}^{\bH} \lambda_{\bh} \\
	&+ \sum_{\bk=\bzero}^{\bH} \int_{[0,1]^m} \Bigg[ \binom{\bN}{\bk} (\bone-\bv)^{\bN-\bk} \ind_{\cR(\bk)}(\bv) \ind_{\cN}(\bk) - \sum_{\bh = \bk}^{\bH} \lambda_{\bh} \binom{\bh}{\bk} (\bone-\bv)^{\bh-\bk} \Bigg] \dd F_{\bk}.
\end{align*}
Whenever for some $\bk$ the term in square brackets is negative for some $\bar{\bv} \in [0,1]^m$, the inner minimization in~\eqref{eq:dual-truncated} brings $\cL(\bF,\blambda)$ to $-\infty$ by considering $F_{\bk}(\bv)$ that has an arbitrarily large mass concentrated in $\bar{\bv}$. This translates into the following constraint for the outer maximization in~\eqref{eq:dual-truncated}
\begin{equation} \label{eq:dual_con}
	\sum_{\bh = \bk}^{\bH} \lambda_{\bh} \binom{\bh}{\bk} (\bone-\bv)^{\bh-\bk} \le \binom{\bN}{\bk} (\bone-\bv)^{\bN-\bk} \ind_{\cR(\bk)}(\bv) \ind_{\cN}(\bk),
\end{equation}
which must be enforced for all $\bk = \bzero,\dots,\bH$ and all $\bv \in [0,1]^m$. Under this additional constraint, the inner minimization of $\cL(\bF,\blambda)$ forces the integral term to zero, i.e.,
\begin{equation*}
	\inf_{\stackrel{\bF\in\cF^{(\bH+\bone)^{\bone}}}{\text{s.t.~\eqref{eq:dual_con}}}} \cL(\bF,\blambda) = \sum_{\bh = \bzero}^{\bH} \lambda_{\bh},
\end{equation*}
and the dual in~\eqref{eq:dual-truncated} becomes
\begin{subequations} \label{eq:dual-truncated-2}
\begin{align}
	\gamma = \sup_{\{\lambda_{\bh}\}_{\bh=\bzero}^{\bH}} \quad
	&\sum_{\bh=0}^{\bH} \lambda_{\bh} \label{eq:dual-truncated-2-cost} \\
	\mathrm{s.t.} \quad
	&\begin{aligned}
		&\sum_{\bh=\bk}^{\bH}\lambda_{\bh} \binom{\bh}{\bk}(\bone-\bv)^{ \bh -\bk} \le \binom{\bN}{\bk} (\bone-\bv)^{\bN-\bk} \ind_{\cR(\bk)}(\bv) \ind_{\cN}(\bk) \\
		&\qquad \bk = \bzero,\dots,\bH, \; \bv \in [0,1]^m.
	\end{aligned} \label{eq:dual-truncated-2-constraint}
\end{align}
\end{subequations}

We now show that~\eqref{eq:dual-lambda} is a tightened version of~\eqref{eq:dual-truncated-2} so that $\gamma$ in~\eqref{eq:dual-truncated-2-cost} is lower bounded by $\gamma^\star$ in~\eqref{eq:dual-lambda-cost}. Comparing~\eqref{eq:dual-lambda-cost} and~\eqref{eq:dual-truncated-2-cost} we clearly see that the two problems have the same decision variables and cost function. Then, it suffices to show that any feasible point of~\eqref{eq:dual-lambda} satisfies~\eqref{eq:dual-truncated-2-constraint} and is thus also feasible for~\eqref{eq:dual-truncated-2}. If that is the case, then~\eqref{eq:dual-lambda} has a feasibility domain no larger than that of~\eqref{eq:dual-truncated-2}, which implies the sought relation between the optimal values.

Consider the inequality in~\eqref{eq:dual-lambda-constraint}, which is imposed for $\bk = \bzero,\dots,\bN$ and $\bv \in [0,1]^m$ such that $(\bone-\bv)^{\bN-\bk} > 0$. In this range of values, by multiplying both sides of~\eqref{eq:dual-lambda-constraint} by $\binom{\bN}{\bk} (\bone-\bv)^{\bN-\bk} > 0$, we obtain
\begin{equation*} 
	\sum_{\bh=\bk}^{\bH} \lambda_{\bh} \binom{\bh}{\bk} (\bone-\bv)^{\bh-\bk} \le \binom{\bN}{\bk} (\bone-\bv)^{\bN-\bk} \ind_{\cR(\bk)}(\bv),
\end{equation*}
which is the same inequality in~\eqref{eq:dual-truncated-2-constraint} (note that $\ind_{\cN}(\bk) = 1$ when $\bk \le \bN$). Thus, \eqref{eq:dual-lambda-constraint} implies~\eqref{eq:dual-truncated-2-constraint} for $\bk = \bzero,\dots,\bN$ and $\bv \in [0,1]^m$ such that $(\bone-\bv)^{\bN-\bk} > 0$. For the same range of $\bk$, but for $\bv \in [0,1]^m$ such that $(\bone-\bv)^{\bN-\bk} = 0$,~\eqref{eq:dual-truncated-2-constraint} is instead implied by~\eqref{eq:dual-lambda-lcond}. As a matter of fact, by~\eqref{eq:dual-lambda-lcond}, we have $\lambda_{\bh} \le 0$ for all $\bh \neq \bN$, yielding
\begin{equation*}
	\sum_{\substack{\bh = \bk \\ \bh \neq \bN}}^{\bH} \lambda_{\bh} \binom{\bh}{\bk} (\bone-\bv)^{\bh-\bk} \le 0.
\end{equation*}
When $(\bone-\bv)^{\bN-\bk} = 0$, this inequality can be in turn rewritten as
\begin{equation*}
	\lambda_{\bN} \binom{\bN}{\bk} (\bone-\bv)^{\bN-\bk} + \sum_{\substack{\bh = \bk \\ \bh \neq \bN}}^{\bH} \lambda_{\bh} \binom{\bh}{\bk} (\bone-\bv)^{\bh-\bk} \le \binom{\bN}{\bk} (\bone-\bv)^{\bN-\bk} \ind_{\cR(\bk)}(\bv),
\end{equation*}
which corresponds to~\eqref{eq:dual-truncated-2-constraint} for $\bk = \bzero,\dots,\bN$. Thus, \eqref{eq:dual-lambda-constraint} and~\eqref{eq:dual-lambda-lcond} together imply~\eqref{eq:dual-truncated-2-constraint} for $\bk = \bzero,\dots,\bN$ and for all $\bv \in [0,1]^m$. Finally, for all the remaining $\bk = \bzero,\ldots,\bH$ such that $\bk \not\le \bN$, it holds that $\bh \neq \bN$ for all $\bh=\bk,\ldots,\bH$, and~\eqref{eq:dual-lambda-lcond} implies
\begin{equation*}
	\sum_{\bh=\bk}^{\bH} \lambda_{\bh} \binom{\bh}{\bk} (\bone-\bv)^{\bh-\bk} \le 0,
\end{equation*}
for all $\bv \in [0,1]^m$, which is the inequality in~\eqref{eq:dual-truncated-2-constraint} when $\bk \not\le \bN$ (note that $\ind_{\cN}(\bk) = 0$ when $\bk \not\le \bN$). In conclusion, we have shown that~\eqref{eq:dual-lambda-constraint} and~\eqref{eq:dual-lambda-lcond} imply~\eqref{eq:dual-truncated-2-constraint}; therefore, any feasible point of~\eqref{eq:dual-lambda} is also feasible for~\eqref{eq:dual-truncated-2}, and $\gamma \ge \gamma^\star$.

Wrapping up, from~\eqref{eq:primal-infinite-dimensional}, \eqref{eq:primal-truncated}, \eqref{eq:dual-truncated}, and~\eqref{eq:dual-truncated-2}, and the discussions thereabout, we obtained $\eta \ge \eta_{\bH} \ge \gamma \ge \gamma^\star$. Using this relation with~\eqref{eq:P>=eta} gives
\begin{equation*}
	\bbPbN\{\bV(\zsN) \in \cR(\bsstar)\} \ge \gamma^\star,
\end{equation*}
which concludes the proof. \qed

\section{Auxiliary results} \label{sec:proof-aux}
In this second appendix we present some auxiliary lemmas regarding properties of the function $\psi_{\bk,\bN,\bH}$ defined in~\eqref{eq:psi_def}, which will ease the proofs of the remaining theoretical results. For ease of reference, we recall here the definition of $\psi_{\bk,\bN,\bH}$, rewritten by splitting the summation over $j \in \cJ_{\bk} = \{\max\{\bk-\bN\}, \dots,\min\{\bH-\bN\}\} \subset \bbZ$ into two parts
\begin{align}
	\psi_{\bk,\bN,\bH}(t)
	= 1 &- \frac{\beta}{|\cJ_{\bzero}|-1} \sum_{j \in \cJ_{\bk} \setminus \{0\}} \frac{\binom{\bN+j\bone}{\bk}}{\binom{\bN}{\bk}} t^j \nonumber \\
	= 1 &- \frac{\beta}{|\cJ_{\bzero}|-1} \sum_{j = \max\{ \bk-\bN \} }^{-1} \frac{\binom{\bN+j\bone}{\bk}}{\binom{\bN}{\bk}} t^j
		 - \frac{\beta}{|\cJ_{\bzero}|-1} \sum_{j = 1}^{ \min\{\bH-\bN\} } \frac{\binom{\bN+j\bone}{\bk}}{\binom{\bN}{\bk}} t^j \nonumber \\
	= 1 &- \frac{\beta}{|\cJ_{\bzero}|-1} \sum_{j = 1}^{ \min\{\bN-\bk\} } \frac{\binom{\bN-j\bone}{\bk}}{\binom{\bN}{\bk}} t^{-j}
		 - \frac{\beta}{|\cJ_{\bzero}|-1} \sum_{j = 1}^{ \min\{\bH-\bN\} } \frac{\binom{\bN+j\bone}{\bk}}{\binom{\bN}{\bk}} t^j, \label{eq:psi_two_sums}
\end{align}
where the third equality arises from changing the index of the first summation. It is understood that $\sum_{j = 1}^{ \min\{\bN-\bk\} }$ is $0$ (i.e., absent) when $\min\{\bN-\bk\} = 0$, and similarly for $\sum_{j = 1}^{ \min\{\bH-\bN\} }$ when $\min\{\bH-\bN\}=0$.

We also restate here some useful identities regarding ratios of binomial coefficients for ease of reference.
\begin{align}
	\frac{\binom{\bN-j\bone}{\bk}}{\binom{\bN}{\bk}}
	&= \frac{\prodim \frac{(N_i-j)!}{k_i!(N_i-j-k_i)!}}{\prodim \frac{N_i!}{k_i!(N_i-k_i)!}}
		= \prodim \frac{(N_i-j)! (N_i-k_i)!}{(N_i-j-k_i)!N_i!} \nonumber \\
	&=\prodim \frac{(N_i-j)(N_i-j-1) \cdots (N_i-j-k_i+1)}{N_i (N_i-1) \cdots (N_i-k_i+1)}
	=\prodim \prod_{\ell = 0}^{k_i-1} \frac{N_i-j-\ell}{N_i-\ell}, \label{eq:bin_ratio_-j_j}
\end{align}
where the third equality is obtained from the right-hand side of the second by simplifying the common terms first in $\frac{(N_i-j)!}{(N_i-j-k_i)!}$ and then in $\frac{(N_i-k_i)!}{N_i!}$. (It is understood that $\prod_{\ell = 0}^{ k_i-1 }$ is $1$ when $k_i = 0$.) Analogously, by simplifying first $\frac{(N_i-k_i)!}{(N_i-j-k_i)!}$ and then $\frac{(N_i-j)!}{N_i!}$, we obtain also
\begin{equation} \label{eq:bin_ratio_-j_ki}
	\frac{\binom{\bN-j\bone}{\bk}}{\binom{\bN}{\bk}} = \prodim \frac{(N_i-k_i) \cdots (N_i-k_i-j+1)}{N_i \cdots (N_i-j+1)} =\prodim \prod_{\ell = 0}^{j-1} \frac{N_i-k_i-\ell}{N_i-\ell}.
\end{equation}
Similarly,
\begin{align}
	\frac{\binom{\bN+j\bone}{\bk}}{\binom{\bN}{\bk}}
	&= \frac{\prodim \frac{(N_i+j)!}{k_i!(N_i+j-k_i)!}}{\prodim \frac{N_i!}{k_i!(N_i-k_i)!}}
		= \prodim \frac{(N_i+j)! (N_i-k_i)!}{(N_i+j-k_i)!N_i!} \nonumber \\
	&=\prodim \frac{(N_i+j) \cdots (N_i+1)}{(N_i-k_i+j) \cdots (N_i-k_i+1)}
	=\prodim \prod_{\ell = 1}^{j} \frac{N_i+\ell}{N_i-k_i+\ell}, \label{eq:bin_ratio_+j_ki}
\end{align}
where the third equality is obtained from the right-hand side of the second by simplifying the common terms first in $\frac{(N_i+j)!}{(N_i+j-k_i)!}$ and then in $\frac{(N_i-k_i)!}{N_i!}$.

\begin{lem} \label{lem:psi-continous-concave-increasing}
	Let $\beta\in(0,1)$, $\bk \in \bbN_0^m$, and $\bN, \bH \in \bbN^m$ be such that $\bk \le \bN \le \bH$. Then, for $t > 0$ ($t \ge 0$ when $\bk \not< \bN$), it holds that $\psi_{\bk,\bN,\bH}(t)$ is continuous, concave, upper-bounded by 1, non-decreasing when $\bH \not> \bN$, and non-increasing when $\bk \not< \bN$. \sqed
\end{lem}

\begin{proof}
	Function $\psi_{\bk,\bN,\bH}(t)$ is clearly continuous and differentiable over $(0,+\infty)$ and the same applies over $[0,+\infty)$ when $\bk \not< \bN$, because $\bk \not< \bN$ implies that $k_i = N_i$ for some $i \in \{1,\ldots,m\}$, and hence $\min\{\bN-\bk\}=0$ and there are no negative-power terms in~\eqref{eq:psi_two_sums}. Given that all terms in~\eqref{eq:psi_two_sums} are non-positive except for the first, which is equal to $1$, we can readily conclude that $\psi_{\bk,\bN,\bH}(t) \le 1$ for all $t > 0$ ($t \ge 0$ when $\bk \not< \bN$).

	Differentiating the expression in~\eqref{eq:psi_two_sums}, we obtain
	\begin{equation*}
		\frac{\dd \psi_{\bk,\bN,\bH}}{\dd t}(t) = \frac{\beta}{|\cJ_{\bzero}|-1} \sum_{j = 1}^{ \min\{\bN-\bk\} } \frac{\binom{\bN-j\bone}{\bk}}{\binom{\bN}{\bk}} j \; t^{-(j+1)} - \frac{\beta}{|\cJ_{\bzero}|-1} \sum_{j = 1}^{ \min\{\bH-\bN\} } \frac{\binom{\bN+j\bone}{\bk}}{\binom{\bN}{\bk}} j \; t^{j-1}.
	\end{equation*}
	Since each term resulting from the two sums is a non-increasing function of $t$ when $t > 0$, it holds that $\frac{\dd \psi_{\bk,\bN,\bH}}{\dd t}(t)$ is also non-increasing, and hence that $\psi_{\bk,\bN,\bH}(t)$ is concave over $t > 0$ ($t \ge 0$ when $\bk \not< \bN$).

	If $H_i = N_i$ for some $i \in \{1,\dots,m\}$, then $\min\{\bH-\bN\} = 0$ and the second sum in~\eqref{eq:psi_two_sums} is zero. Each remaining term is a non-decreasing function of $t$, and $\psi_{\bk,\bN,\bH}(t)$ is non-decreasing as well. Similarly, if $k_i = N_i$ for some $i \in \{1,\dots,m\}$, then $\min\{\bN-\bk\} = 0$ and the first sum in~\eqref{eq:psi_two_sums} is zero. This time, each remaining term is non-increasing, which means that also $\psi_{\bk,\bN,\bH}(t)$ is non-increasing.\footnote{When both $\min\{\bH-\bN\} = 0$ and $\min\{\bN-\bk\} = 0$, $\psi_{\bk,\bN,\bH}(t) = 1$ $\forall t$, which is both non-decreasing and non-increasing.}
	\qed
\end{proof}

\begin{lem} \label{lem:zeros-psi}
	Let $\beta\in(0,1)$, $\bk \in \bbN_0^m$, and $\bN, \bH \in \bbN^m$ be such that $\bk \le \bN \le \bH$. Then $\psi_{\bk,\bN,\bH}(t)$ has
	\begin{itemize}
		\item a unique zero $\tbub_{\bk} \in (0,\hat{t})$ if $\bk < \bN$;
		\item no zeros in $[0,\hat{t}]$ if $\bk \not< \bN$;
		\item a unique zero $\tblb_{\bk} \in (\hat{t},+\infty)$ if $\bH > \bN$;
		\item no zeros in $[\hat{t},+\infty)$ if $\bH \not> \bN$;
	\end{itemize}
	where $\hat{t} = (\bone-\frac{\bk}{\bN})^{\bone} \in [0,1]$ and $\psi_{\bk,\bN,\bH}(\hat{t}) > 0$. \sqed
\end{lem}

\begin{proof}
	Let us start by evaluating $\psi_{\bk,\bN,\bH}(t)$ in~\eqref{eq:psi_two_sums} at $\hat{t} = (\bone-\frac{\bk}{\bN})^{\bone} = (\frac{\bN - \bk}{\bN})^{\bone} \in [0,1]$.
	If $\bk \not < \bN$, then $\min\{\bN-\bk\} = 0$, so that the first sum in~\eqref{eq:psi_two_sums} is absent, and $\hat{t} = 0$, making also null the second term. This gives $\psi_{\bk,\bN,\bH}(\hat{t}) = 1$, which is positive. When instead $\bk < \bN$, we have that
	\begin{align}
		\psi_{\bk,\bN,\bH}(\hat{t})
		= 1 &- \frac{\beta}{|\cJ_{\bzero}|-1} \sum_{j = 1}^{ \min\{\bN-\bk\} } \frac{\binom{\bN-j\bone}{\bk}}{\binom{\bN}{\bk}} \Big[ \Big(\frac{\bN - \bk}{\bN}\Big)^{\bone} \Big]^{-j} \nonumber \\
		&- \frac{\beta}{|\cJ_{\bzero}|-1} \sum_{j = 1}^{ \min\{\bH-\bN\} } \frac{\binom{\bN+j\bone}{\bk}}{\binom{\bN}{\bk}} \Big[ \Big(\frac{\bN - \bk}{\bN}\Big)^{\bone} \Big]^j. \label{eq:psi_two_sums_N-k_N}
	\end{align}
	Using~\eqref{eq:bin_ratio_-j_ki} and~\eqref{eq:bin_ratio_+j_ki} we can express each term in the two sums as
	\begin{subequations} \label{eq:bin_ratio_<1}
	\begin{equation}
		\frac{\binom{\bN-j\bone}{\bk}}{\binom{\bN}{\bk}} \Big[ \Big(\frac{\bN - \bk}{\bN}\Big)^{\bone} \Big]^{-j}
		=\prodim \prod_{\ell = 0}^{j-1} \frac{N_i-k_i-\ell}{N_i-\ell} \frac{N_i}{N_i-k_i}
		\le 1,
	\end{equation}
	and
	\begin{equation}
		\frac{\binom{\bN+j\bone}{\bk}}{\binom{\bN}{\bk}} \Big[ \Big(\frac{\bN - \bk}{\bN}\Big)^{\bone} \Big]^j
		=\prodim \prod_{\ell = 1}^{j} \frac{N_i+\ell}{N_i-k_i+\ell} \frac{N_i-k_i}{N_i}
		\le 1,
	\end{equation}
	\end{subequations}
	where the inequalities $\le 1$ are, respectively, due to
	\begin{align*}
		(N_i-k_i-\ell) N_i & \le (N_i-\ell)(N_i-k_i) \\
		(N_i+\ell)(N_i-k_i) & \le (N_i-k_i+\ell) N_i
	\end{align*}
	for $\ell,k_i \ge 0$, as it is straightforward to verify. Using the bounds~\eqref{eq:bin_ratio_<1} in~\eqref{eq:psi_two_sums_N-k_N} yields
	\begin{align*}
		\psi_{\bk,\bN,\bH}(\hat{t})
		&\ge 1 - \beta \frac{\min\{\bN-\bk\} + \min\{\bH-\bN\}}{\min\{\bH-\bN\} - \max\{-\bN\}} \\
		&= 1 - \beta \frac{\min\{\bH-\bN\} + \min\{\bN-\bk\}}{\min\{\bH-\bN\} + \min\{\bN\}} \\
		&\ge 1 - \beta > 0.
	\end{align*}
	Hence, $\psi_{\bk,\bN,\bH}(\hat{t}) > 0$ also in this case.

	We now discuss the zeros of $\psi_{\bk,\bN,\bH}(t)$ for the four cases: $\bk < \bN$, $\bk \not< \bN$, $\bH > \bN$, and $\bH \not> \bN$.

	If $\bk < \bN$, then $\min\{\bN-\bk\} \ge 1$, and the first sum in~\eqref{eq:psi_two_sums} is non-empty, yielding
	\begin{equation} \label{eq:psi_lim_0+}
		\lim_{t \to 0^+} \psi_{\bk,\bN,\bH}(t) = -\infty,
	\end{equation}
	irrespective of $\bH \ge \bN$ (the limit of the second sum for $t \to 0^+$ is zero for any $\bH \ge \bN$). By~\eqref{eq:psi_lim_0+} and $\psi_{\bk,\bN,\bH}(\hat{t}) > 0$ together with continuity and concavity on $t > 0$ granted by Lemma~\ref{lem:psi-continous-concave-increasing}, there exists a unique zero $\tbub_{\bk} \in (0,\hat{t})$.

	If $\bk \not< \bN$, then we have already observed that $\hat{t} = 0$ and $\psi_{\bk,\bN,\bH}(\hat{t}) = 1$. Therefore, there are no zeros in $[0,\hat{t}]$ trivially.

	If $\bH > \bN$, then $\min\{\bH-\bN\} \ge 1$ and the second sum in~\eqref{eq:psi_two_sums} is non-empty, yielding
	\begin{equation} \label{eq:psi_lim_inf}
		\lim_{t \to +\infty} \psi_{\bk,\bN,\bH}(t) = -\infty,
	\end{equation}
	irrespective of $\bk \le \bN$ (the limit of the first summation for $t \to +\infty$ is zero for any $\bk \le \bN$). By~\eqref{eq:psi_lim_inf} and $\psi_{\bk,\bN,\bH}(\hat{t}) > 0$ together with continuity and concavity on $t > 0$ granted by Lemma~\ref{lem:psi-continous-concave-increasing}, there exists a unique zero $\tblb_{\bk} \in (\hat{t},+\infty)$.

	If $\bH \not> \bN$, then $\min\{\bH-\bN\} = 0$ and the second sum in~\eqref{eq:psi_two_sums} is absent. This gives $\lim_{t \to +\infty} \psi_{\bk,\bN,\bH}(t) = 1$ (irrespective of $\bk \le \bN$). This, together with $\psi_{\bk,\bN,\bH}(\hat{t}) > 0$, and the monotonicity on $t > 0$ granted by Lemma~\ref{lem:psi-continous-concave-increasing}, implies that there are no zeros in $[\hat{t},+\infty)$.

	This concludes the proof.
	\qed
\end{proof}

\begin{lem} \label{lem:psiN_le_psiNlb}
	Let $\beta\in(0,1)$, $\bk \in \bbN_0^m$, and $\bN \in \bbN^m$ be such that $\bk \le \Nlb \bone$ where $\Nlb = \min\{\bN\}$. Then, we have
	\begin{equation} \label{eq:psiN_le_psiNlb}
		\psi_{\bk,\bN,\bN}(t) \le \psi_{\bk,\Nlb\bone,\Nlb\bone}(t),
	\end{equation}
	for all $t > 0$. \sqed
\end{lem}

\begin{proof}
	Let us start by considering $\psi_{\bk,\bN,\bH}(t)$ in~\eqref{eq:psi_two_sums} with $\bH = \bN$, for which it holds
	\begin{align*}
		\psi_{\bk,\bN,\bN}(t)
		&= 1 - \frac{\beta}{\Nlb} \sum_{j = 1}^{ \min\{\bN-\bk\} } \frac{\binom{\bN-j\bone}{\bk}}{\binom{\bN}{\bk}} t^{-j} \\
		&= 1 - \frac{\beta}{\Nlb} \sum_{j = 1}^{ \min\{\bN-\bk\} } \prodim \prod_{\ell = 0}^{j-1} \frac{N_i-k_i-\ell}{N_i-\ell} t^{-j} \\
		&\le 1 - \frac{\beta}{\Nlb} \sum_{j = 1}^{ \min\{\bN-\bk\} } \prodim \prod_{\ell = 0}^{j-1} \frac{\Nlb-k_i-\ell}{\Nlb-\ell} t^{-j} \\
		&= 1 - \frac{\beta}{\Nlb} \sum_{j = 1}^{ \min\{\bN-\bk\} } \frac{\binom{\Nlb\bone-j\bone}{\bk}}{\binom{\Nlb\bone}{\bk}} t^{-j},
	\end{align*}
	where the second equality follows from~\eqref{eq:bin_ratio_-j_ki}, the inequality from
	\begin{equation*}
		\frac{N_i-k_i-\ell}{N_i-\ell} \ge \frac{\Nlb-k_i-\ell}{\Nlb-\ell}
	\end{equation*}
	as an effect of $N_i \ge \Nlb$ for all $\im$, and the last equality from~\eqref{eq:bin_ratio_-j_ki} again.

	Since $N_i \ge \Nlb$, we also have $\min\{\bN-\bk\} \ge \min\{\Nlb\bone-\bk\}$ and, hence,
	\begin{align*}
		\psi_{\bk,\bN,\bN}(t) &\le 1 - \frac{\beta}{\Nlb} \sum_{j = 1}^{ \min\{\bN-\bk\} } \frac{\binom{\Nlb\bone-j\bone}{\bk}}{\binom{\Nlb\bone}{\bk}} t^{-j} \\
		&\le 1 - \frac{\beta}{\Nlb} \sum_{j = 1}^{ \min\{\Nlb\bone-\bk\} } \frac{\binom{\Nlb\bone-j\bone}{\bk}}{\binom{\Nlb\bone}{\bk}} t^{-j} \\
		&= \psi_{\bk,\Nlb\bone,\Nlb\bone}(t),
	\end{align*}
	where the second inequality comes from dropping non-positive terms and the last equality is by the very definition of $\psi_{\bk,\bN,\bH}(t)$ in~\eqref{eq:psi_two_sums}.

	This concludes the proof.
	\hqed
\end{proof}

\begin{lem} \label{lem:psi_best_worst_supp_distr}
	Let $\beta\in(0,1)$, $\bk \in \bbN_0^m$, and $N \in \bbN$ be such that $\knorm := |\bk|\le N$. Then, for any $\bkub,\bklb \in \bbN_0^m$ such that $|\bkub|=|\bklb|=\knorm$ and
	\begin{align*}
		&\bar{k}_i \in \{\knorm,0\} &&\hspace{-8em}\im, \\
		&\ubar{k}_i \in \{\lfloor\tfrac{\knorm}{m}\rfloor,\lceil\tfrac{\knorm}{m}\rceil\} &&\hspace{-8em}\im,
	\end{align*}
	it holds that
	\begin{equation} \label{eq:psi_best_worst_supp_distr}
		\psi_{\bklb,N\bone,N\bone}(t) \le \psi_{\bk,N\bone,N\bone}(t) \le \psi_{\bkub,N\bone,N\bone}(t),
	\end{equation}
	for all $t > 0$. \sqed
\end{lem}

\begin{proof}
	Note that $\knorm \le N$ implies that $\ubar{\bk} \le N \bone$, $\bk \le N \bone$, and $\bar{\bk} \le N \bone$, so that all functions in~\eqref{eq:psi_best_worst_supp_distr} are well-defined. To ease the notation, let $\bk^+ = \max\{\bk\}$, $\bk^- = \min\{\bk\}$, and, accordingly, $\bklb^+ = \lceil\tfrac{\knorm}{m}\rceil$ and $\bklb^- = \lfloor\tfrac{\knorm}{m}\rfloor$.

	We start by showing the upper-bound on $\psi_{\bk,N\bone,N\bone}(t)$. According to~\eqref{eq:psi_two_sums}, for $t>0$ we have that
	\begin{align}
		\psi_{\bk,N\bone,N\bone}(t)
		&= 1 - \frac{\beta}{N} \sum_{j = 1}^{ \min\{N\bone-\bk\} } \frac{\binom{N\bone-j\bone}{\bk}}{\binom{N\bone}{\bk}} t^{-j} \nonumber \\
		&\le 1 - \frac{\beta}{N} \sum_{j = 1}^{ N - \knorm } \frac{\binom{N\bone-j\bone}{\bk}}{\binom{N\bone}{\bk}} t^{-j} \nonumber \\
		&= 1 - \frac{\beta}{N} \sum_{j = 1}^{ N - \knorm } \prodim \prod_{\ell = 0}^{k_i-1} \frac{N-j-\ell}{N-\ell} t^{-j} \label{eq:psi_le_ub_1}
	\end{align}
	where the inequality follows from dropping non-positive terms since $\min\{N\bone-\bk\} \ge N-\knorm$ (note that $k_i \le \knorm$ for all $\im$), while the last equality is due to~\eqref{eq:bin_ratio_-j_j}.

	Focusing on the two-product summand in the last expression, a change of the order of the products and letting $\kappa_{\ell} = |\{ i \in \{1,\dots,m\}: k_i \ge \ell+1 \}|$ gives
	\begin{align}
		\prodim \prod_{\ell = 0}^{k_i-1} \frac{N-j-\ell}{N-\ell}
		&= \prod_{\ell = 0}^{\bk^+-1} \prod_{\substack{i \in \{1,\dots,m\}: \\ k_i \ge \ell+1}} \frac{N-j-\ell}{N-\ell} \nonumber \\
		&= \prod_{\ell = 0}^{\bk^+-1} \bigg( \frac{N-j-\ell}{N-\ell} \bigg)^{\kappa_{\ell}} \nonumber \\
		&= \prod_{\ell = 0}^{\bk^+-1} \frac{N-j-\ell}{N-\ell} \prod_{\ell = 0}^{\bk^+-1} \bigg( \frac{N-j-\ell}{N-\ell} \bigg)^{\kappa_{\ell}-1} \nonumber \\
		\displaybreak[0]
		&\ge \prod_{\ell = 0}^{\bk^+-1} \frac{N-j-\ell}{N-\ell} \prod_{\ell' = \bk^+}^{|\bk|-1} \frac{N-j-\ell'}{N-\ell'} \nonumber \\
		&= \prod_{\ell = 0}^{|\bk|-1} \frac{N-j-\ell}{N-\ell} \nonumber \\
		&= \prod_{\ell = 0}^{\knorm-1} \frac{N-j-\ell}{N-\ell}, \label{eq:prod_i_ell_gt_prod_K}
	\end{align}
	where: the second equality follows from the fact that there are $\kappa_{\ell}$ equal terms in the inner product; the third equality because $\kappa_{\ell} > 1$ for $\ell=0,\ldots,\bk^+-1$; the inequality is due to
	\begin{equation*}
		\frac{N-j-\ell}{N-\ell} \ge \frac{N-j-\ell'}{N-\ell'}
	\end{equation*}
	for any $\ell' \ge \ell$ (and $j \ge 0$) together with the fact that the second products on the two sides of the inequality have the same number of terms since $\sum_{\ell = 0}^{\bk^+-1} (\kappa_\ell-1) = \sumim k_i - \bk^+ = |\bk| - \bk^+$; the fourth equality simply recollects the two products into a single one, and the last equality is due to $|\bk| = \knorm$ by definition.

	Using~\eqref{eq:prod_i_ell_gt_prod_K} in~\eqref{eq:psi_le_ub_1} yields
	\begin{align*}
		\psi_{\bk,N\bone,N\bone}(t)
		&\le 1 - \frac{\beta}{N} \sum_{j = 1}^{ N - \knorm } \prod_{\ell = 0}^{\knorm-1} \frac{N-j-\ell}{N-\ell} t^{-j} \\
		&= 1 - \frac{\beta}{N} \sum_{j = 1}^{ N - \knorm } \frac{\binom{N\bone-j\bone}{\bkub}}{\binom{N\bone}{\bkub}} t^{-j} \\
		&= \psi_{\bkub,N\bone,N\bone}(t),
	\end{align*}
	where the first equality follows from~\eqref{eq:bin_ratio_-j_j} applied for $\bN = N\bone$ and $\bkub$, and the second equality from $N-K = \min\{N\bone - \bkub\}$ and~\eqref{eq:psi_two_sums}.

	This proves the upper bound in~\eqref{eq:psi_best_worst_supp_distr}.

	We now proceed by showing the lower-bound on $\psi_{\bk,N\bone,N\bone}(t)$. According to~\eqref{eq:psi_two_sums}, for $t>0$ we have that
	\begin{align}
		\psi_{\bklb,N\bone,N\bone}(t)
		&= 1 - \frac{\beta}{N} \sum_{j = 1}^{ \min\{N\bone-\bklb\} } \frac{\binom{N\bone-j\bone}{\bklb}}{\binom{N\bone}{\bklb}} t^{-j} \nonumber \\
		&= 1 - \frac{\beta}{N} \sum_{j = 1}^{ N-\bklb^+ } \frac{\binom{N\bone-j\bone}{\bklb}}{\binom{N\bone}{\bklb}} t^{-j} \nonumber \\
		&\le 1 - \frac{\beta}{N} \sum_{j = 1}^{ N-\bk^+ } \frac{\binom{N\bone-j\bone}{\bklb}}{\binom{N\bone}{\bklb}} t^{-j} \nonumber \\
		&= 1 - \frac{\beta}{N} \sum_{j = 1}^{ N-\bk^+ } \prodim \prod_{\ell = 0}^{\ubar{k}_i-1} \frac{N-j-\ell}{N-\ell} t^{-j} \label{eq:psi_ge_ub_1}
	\end{align}
	where the second equality derives from $\min\{N - \bklb\} = N-\bklb^+$, the inequality from $N-\bklb^+ \ge N - \bk^+$ (indeed, $\bklb^+ = \ceil{\frac{\knorm}{m}} = \ceil{\frac{|\bk|}{m}} \le \ceil{\frac{m \bk^+}{m}} = \bk^+$), and the last equality from~\eqref{eq:bin_ratio_-j_j}.

	Focusing on the two-product summand in the last expression, we have that
	\begingroup
	\allowdisplaybreaks
	\begin{align}
		\prodim \prod_{\ell = 0}^{\ubar{k}_i-1} \frac{N-j-\ell}{N-\ell}
		&= \prod_{\substack{i \in \{1,\dots,m\} \\ k_i > \ubar{k}_i}} \prod_{\ell = 0}^{\ubar{k}_i-1} \frac{N-j-\ell}{N-\ell} \quad \times \nonumber \\
		&\hspace{1.17em} \prod_{\substack{i \in \{1,\dots,m\} \\ k_i = \ubar{k}_i}} \prod_{\ell = 0}^{k_i-1} \frac{N-j-\ell}{N-\ell} \quad \times \nonumber \\
		&\hspace{1.17em} \prod_{\substack{i \in \{1,\dots,m\} \\ k_i < \ubar{k}_i}} \prod_{\ell = 0}^{k_i-1} \frac{N-j-\ell}{N-\ell} \prod_{\ell' = k_i}^{\ubar{k}_i-1} \frac{N-j-\ell'}{N-\ell'} \nonumber \\
		&\ge \prod_{\substack{i \in \{1,\dots,m\} \\ k_i > \ubar{k}_i}} \prod_{\ell = 0}^{\ubar{k}_i-1} \frac{N-j-\ell}{N-\ell} \prod_{\ell'' = \ubar{k}_i}^{k_i-1} \frac{N-j-\ell''}{N-\ell''} \quad \times \nonumber \\
		&\hspace{1.17em} \prod_{\substack{i \in \{1,\dots,m\} \\ k_i = \ubar{k}_i}} \prod_{\ell = 0}^{k_i-1} \frac{N-j-\ell}{N-\ell} \quad \times \nonumber \\
		&\hspace{1.17em} \prod_{\substack{i \in \{1,\dots,m\} \\ k_i < \ubar{k}_i}} \prod_{\ell = 0}^{k_i-1} \frac{N-j-\ell}{N-\ell} \nonumber \\
		&= \prodim \prod_{\ell = 0}^{k_i-1} \frac{N-j-\ell}{N-\ell}, \label{eq:prod_i_uki_gt_prod_ki}
	\end{align}
	\endgroup
	where in the first equality the product over $i$ has been split into three cases, and the inequality is due to:
	\begin{itemize}
		\item[a.] $\frac{N-j-\ell'}{N-\ell'} \ge \frac{N-j-\ell''}{N-\ell''}$ for any pair of indices $\ell',\ell''$ appearing in the products, since it holds that $\ell'' \ge \floor{\frac{\knorm}{m}} \ge \ceil{\frac{\knorm}{m}} - 1 \ge \ell'$;
		\item[b.] the fact that the products on the two sides of the inequality have the same total number of terms since $\sumim \ubar{k}_i = |\bklb| = \knorm = |\bk| = \sumim k_i$ by assumption.
	\end{itemize}

	Using~\eqref{eq:prod_i_uki_gt_prod_ki} in~\eqref{eq:psi_ge_ub_1} yields
	\begin{align*}
		\psi_{\bklb,N\bone,N\bone}(t)
		&\le 1 - \frac{\beta}{N} \sum_{j = 1}^{ N-\bk^+ } \prodim \prod_{\ell = 0}^{k_i-1} \frac{N-j-\ell}{N-\ell} t^{-j} \\
		&= 1 - \frac{\beta}{N} \sum_{j = 1}^{ N-\bk^+ } \frac{\binom{N\bone-j\bone}{\bk}}{\binom{N\bone}{\bk}} t^{-j} \\
		&= \psi_{\bk,N\bone,N\bone}(t),
	\end{align*}
	where the second equality derives from~\eqref{eq:bin_ratio_-j_j} and the last equality from $N - \bk^+ = \min\{N\bone-\bk\}$ and~\eqref{eq:psi_two_sums}.

	This proves the lower bound in~\eqref{eq:psi_best_worst_supp_distr} and concludes the proof.
	\qed
\end{proof}

\section{Proof of Propositions~\ref{prop:region-bound-diag}--\ref{prop:apriori-uniform-joint-bound-diag}} \label{sec:proof-props}

\subsection{Proof of Proposition~\ref{prop:region-bound-diag}} \label{sec:proof-diag-region}
For ease of reference, we recall that, according to~\eqref{eq:region_diag}, throughout this section we consider
\begin{equation} \label{eq:region_diag_proof_prop}
	\cR(\bk) = \{ \bv \in [0,1]^m :~ (\bone-\bv)^{\bone} \in [\tbub_{\bk}, \tblb_{\bk}] \}.
\end{equation}
with $\tbub_{\bk}$ and $\tblb_{\bk}$ defined as in~\eqref{eq:psi_zeros}.

Under Assumptions~\ref{ass:property} and~\ref{ass:non-degeneracy}, we can invoke Theorem~\ref{thm:region_bound} to obtain
\begin{equation} \label{eq:region_bound_proof_prop}
	\bbPbN\{ \bV(\zsN) \in \cR(\bsstar) \} \ge \gamma^\star,
\end{equation}
where $\bsstar \in \bbN_0^m$ is the complexity of $\zsN$ and $\gamma^\star$ is the optimal cost of~\eqref{eq:dual-lambda}. Consider choosing $\{\lambda_{\bh}\}_{\bh = \bzero}^{\bH}$ according to the diagonal allocation in~\eqref{eq:lambda-allocation-diagonal}. We show that this choice of $\{\lambda_{\bh}\}_{\bh = \bzero}^{\bH}$ is feasible for~\eqref{eq:dual-lambda}, a fact that later will be used to bound $\gamma^\star$ and conclude the proof.

By~\eqref{eq:lambda-allocation-diagonal} we have that
\begin{align*}
	1 = \lambda_{\bh} & &&\hspace{-6em} \bh = \bN \\
	\left.
	\begin{aligned}
		&&\textstyle -\frac{\beta}{|\cJ_{\bzero}|-1}	\\
		&&0
	\end{aligned}
	\right\}
	= \lambda_{\bh} &\le 0 &&\hspace{-6em} \bh \neq \bN,
\end{align*}
which shows that constraint~\eqref{eq:dual-lambda-lcond} is trivially satisfied (note that $|\cJ_{\bzero}| \ge 2$ for any $\bN \in \bbN^m$).

Let us now focus on the left hand side of~\eqref{eq:dual-lambda-constraint}, which, evaluated at the $\{\lambda_{\bh}\}_{\bh = \bzero}^{\bH}$ set according to the diagonal allocation in~\eqref{eq:lambda-allocation-diagonal} becomes as follows:
\begin{align}
	\sum_{\bh=\bk}^{\bH} \lambda_{\bh} \frac{\binom{\bh}{\bk}}{\binom{\bN}{\bk}}(\bone-\bv)^{\bh-\bN}
	&= \sum_{\substack{\exists j: \; \bh=\bN+j\bone \\ \bk \le \bh \le \bH}} \lambda_{\bN+j\bone} \frac{\binom{\bN+j\bone}{\bk}}{\binom{\bN}{\bk}}(\bone-\bv)^{\bN+j\bone-\bN} \nonumber \\
	&= \sum_{j = \max\{\bk-\bN\}}^{\min\{\bH-\bN\}} \lambda_{\bN+j\bone} \frac{\binom{\bN+j\bone}{\bk}}{\binom{\bN}{\bk}}(\bone-\bv)^{j\bone} \nonumber \\
	&= 1 - \frac{\beta}{|\cJ_{\bzero}|-1} \sum_{j \in \cJ_{\bk}\setminus \{0\}} \frac{\binom{\bN+j\bone}{\bk}}{\binom{\bN}{\bk}} \big((\bone-\bv)^{\bone}\big)^j \nonumber \\
	&= \psi_{\bk,\bN,\bH} \big((\bone-\bv)^{\bone}\big), \label{eq:dual-lambda-constraint-diag}
\end{align}
where $\cJ_{\bk} =\{\max\{\bk-\bN\}, \dots,\min\{\bH-\bN\}\}$; the first equality holds because $\lambda_{\bh} = 0$ for all $\bh \neq \bN + j\bone$, the second equality makes the range of $j$ explicit, in the third equality we simply take the term $j = 0$ outside the sum, and the last equality is by the definition of $\psi_{\bk,\bN,\bH}(t)$ in~\eqref{eq:psi_def}. Using~\eqref{eq:dual-lambda-constraint-diag} and taking into account the structure of $\cR(\bk)$ in~\eqref{eq:region_diag_proof_prop}, we have that
\begin{equation}
	\eqref{eq:dual-lambda-constraint}
	\iff
	\begin{cases}
		\psi_{\bk,\bN,\bH} \big((\bone-\bv)^{\bone}\big) \le \ind_{[\tbub_{\bk},\tblb_{\bk}]}\big((\bone-\bv)^{\bone}\big) \\
		\bv \in
		\begin{cases}
			[0,1)^m, \quad \text{if } \bk < \bN \\
			[0,1]^m, \quad \text{if } \bk \not< \bN
		\end{cases}
	\end{cases}
	\hspace{-1em}\iff
	\begin{cases}
		\psi_{\bk,\bN,\bH}(t) \le \ind_{[\tbub_{\bk}, \tblb_{\bk}]}(t) \\
		t \in
		\begin{cases}
			(0,1], \quad \text{if } \bk < \bN \\
			[0,1], \quad \text{if } \bk \not< \bN
		\end{cases}
	\end{cases}
	\label{eq:dual-lambda-constraint-psi}
\end{equation}
the second double implication being due to the substitution $t = (\bone-\bv)^{\bone}$.

By Lemma~\ref{lem:psi-continous-concave-increasing}, $\psi_{\bk,\bN,\bH}(t)$ is continuous and satisfies $\psi_{\bk,\bN,\bH}(t) \le 1$ for all $t > 0$ (for all $t \ge 0$ when $\bk \not< \bN$). Therefore, checking if~\eqref{eq:dual-lambda-constraint-psi} is satisfied amounts to study the sign of $\psi_{\bk,\bN,\bH}(t)$ on $(0,1]$ (on $[0,1]$ when $\bk \not< \bN)$.

Start by noticing that, according to Lemma~\ref{lem:zeros-psi}, $\psi_{\bk,\bN,\bH}(\hat{t}) > 0$ with $\hat{t} = (\bone-\frac{\bk}{\bN})^{\bone} \in [0,1]$. We now discuss four cases: $\bk < \bN < \bH$, $\bk \not< \bN < \bH$, $\bk < \bN \not< \bH$, and $\bk \not< \bN \not< \bH$.

If $\bk < \bN < \bH$, then, by Lemma~\ref{lem:zeros-psi}, there is a unique zero in $(0,\hat{t})$ and a unique zero in $(\hat{t},+\infty)$ and, by~\eqref{eq:psi_zeros}, $\tbub_{\bk}$ and $\tblb_{\bk}$ are these two zeros respectively. Since $\psi_{\bk,\bN,\bH}(\hat{t}) > 0$ and $\hat{t} \in [\tbub_{\bk},\tblb_{\bk}]$, then $\psi_{\bk,\bN,\bH}(t) \ge 0$ for all $t \in [\tbub_{\bk},\tblb_{\bk}]$ and $\psi_{\bk,\bN,\bH}(t) < 0$ for all $t \not\in [\tbub_{\bk},\tblb_{\bk}]$, which satisfies~\eqref{eq:dual-lambda-constraint-psi} and, hence,~\eqref{eq:dual-lambda-constraint}.

If $\bk \not< \bN < \bH$, then, by Lemma~\ref{lem:zeros-psi}, there are no zeros in $[0,\hat{t}]$ and a unique zero in $(\hat{t},+\infty)$ and, by~\eqref{eq:psi_zeros}, $\tbub_{\bk} = 0$ and $\tblb_{\bk}$ is the zero. Since $\psi_{\bk,\bN,\bH}(\hat{t}) > 0$ and $\hat{t} \in [0,\tblb_{\bk}]$, then $\psi_{\bk,\bN,\bH}(t) \ge 0$ for all $t \in [0,\tblb_{\bk}] = [\tbub_{\bk},\tblb_{\bk}]$ and $\psi_{\bk,\bN,\bH}(t) < 0$ for all $t \not\in [0,\tblb_{\bk}] = [\tbub_{\bk},\tblb_{\bk}]$, which satisfies~\eqref{eq:dual-lambda-constraint-psi} and, hence,~\eqref{eq:dual-lambda-constraint}.

If $\bk < \bN \not< \bH$, then, by Lemma~\ref{lem:zeros-psi}, there is a unique zero in $(0,\hat{t})$ and no zeros in $[\hat{t},+\infty)$ and, by~\eqref{eq:psi_zeros}, $\tbub_{\bk}$ is the zero and $\tblb_{\bk} = 1$. Since $\psi_{\bk,\bN,\bH}(\hat{t}) > 0$ and $\hat{t} \in [\tbub_{\bk},1]$, then $\psi_{\bk,\bN,\bH}(t) \ge 0$ for all $t \in [\tbub_{\bk},1] = [\tbub_{\bk},\tblb_{\bk}]$ and $\psi_{\bk,\bN,\bH}(t) < 0$ for all $t \not\in [\tbub_{\bk},1] = [\tbub_{\bk},\tblb_{\bk}]$, which satisfies~\eqref{eq:dual-lambda-constraint-psi} and, hence,~\eqref{eq:dual-lambda-constraint}.

Finally, if $\bk \not< \bN \not< \bH$, then, by Lemma~\ref{lem:psi-continous-concave-increasing}, $\psi_{\bk,\bN,\bH}(t)$ must be both non-increasing and non-decreasing; thus, it is constant and equal to $\psi_{\bk,\bN,\bH}(\hat{t}) > 0$, $\forall t \ge 0$ (indeed, we have $\psi_{\bk,\bN,\bH}(t) = 1$, $\forall t$). By~\eqref{eq:psi_zeros}, $\tbub_{\bk} = 0$ and $\tblb_{\bk} = 1$, which readily satisfies~\eqref{eq:dual-lambda-constraint-psi} and, hence,~\eqref{eq:dual-lambda-constraint}.

We have proven that selecting $\{\lambda_{\bh}\}_{\bh = \bzero}^{\bH}$ according to the diagonal allocation in~\eqref{eq:lambda-allocation-diagonal} is feasible for problem~\eqref{eq:dual-lambda} given the selected $\cR(\bk)$, $\bk = \bzero,\dots,\bN$. From~\eqref{eq:dual-lambda-cost}, we obtain
\begin{equation}
	\gamma^\star
	\ge \sum_{\bh=\bzero}^{\bH} \lambda_{\bh}
	= \sum_{\substack{\exists j: \; \bh=\bN+j\bone \\ \bzero \le \bh \le \bH}} \lambda_{\bN+j\bone}
	= \sum_{j = \max\{-\bN\}}^{\min\{\bH-\bN\}} \lambda_{\bN+j\bone}
	= 1 - \sum_{j \in \cJ_{\bzero}\setminus \{0\}} \frac{\beta}{|\cJ_{\bzero}|-1}
	= 1 - \beta
	\label{eq:gamma_star_lb_proof_prop}
\end{equation}
where the rationale behind these equalities closely follows that used to obtain~\eqref{eq:dual-lambda-constraint-diag}. Using~\eqref{eq:gamma_star_lb_proof_prop} in~\eqref{eq:region_bound_proof_prop} yields~\eqref{eq:region_bound_diag} and concludes the proof. \qed

\subsection{Proof of Proposition~\ref{prop:joint-bound-diag}} \label{sec:proof-diag-joint}
Set $\bH = \bN$. Under Assumptions~\ref{ass:property} and~\ref{ass:non-degeneracy}, we can readily invoke Proposition~\ref{prop:region-bound-diag} to claim
\begin{equation} \label{eq:region_bound_diag_ub}
	\bbPbN\{ \bV(\zsN) \in \cR(\bsstar) \} \ge 1 - \beta,
\end{equation}
where $\bsstar \in \bbN_0^m$ is the complexity of $\zsN$ and
\begin{equation}
	\cR(\bk)
	= \{ \bv \in [0,1]^m :~ (\bone-\bv)^{\bone} \in [\tbub_{\bk}, 1] \}
	= \{ \bv \in [0,1]^m :~ (\bone-\bv)^{\bone} \ge \tbub_{\bk} \}, \label{eq:region_diag_joint_ub}
\end{equation}
$\tbub_{\bk}$ being defined according to~\eqref{eq:psi_zero_ub} in Proposition~\ref{prop:region-bound-diag} and $\bk=\bzero,\ldots,N$ (the requirement $(\bone-\bv)^{\bone} \le 1$ can be ignored since it is trivially satisfied by any $\bv \in [0,1]^m$). By~\eqref{eq:sub-additivity} and~\eqref{eq:region_bound_diag_ub}, we immediately have (see also~\eqref{eq:joint-coll-bound})
\begin{equation} \label{eq:region_bound_diag_joint_ub}
	\bbPbN \Big\{ V(\zsN) \le \max_{\bv \in \cR(\bsstar)} |\bv| \Big\} \ge 1 - \beta,
\end{equation}
where $\max_{\bv \in \cR(\bsstar)} |\bv|$ has to be explicitly computed for the specific regions $\cR(\bk)$ in~\eqref{eq:region_diag_joint_ub}. This requires to solve
\begin{align*}
	\max_{\bv \in [0,1]^m} \quad
	& |\bv| \\
	\text{s.t.} \quad
	& (\bone-\bv)^{\bone} \ge \tbub_{\bk},
\end{align*}
or, more explicitly,\footnote{Note that $|\bv| = |v_1| + \cdots + |v_m| = v_1 + \cdots + v_m$ for $\bv \in [0,1]^m$.}
\begin{subequations} \label{eq:maxR}
\begin{align}
	\max_{\bv \in [0,1]^m} \quad \label{eq:maxR_cost}
	& \sum_{i=1}^m v_i \\
	\text{s.t.} \quad
	& \prod_{i=1}^m (1-v_i) \ge \tbub_{\bk}, \label{eq:maxR_constraint}
\end{align}
\end{subequations}
for all $\bk = \bzero,\dots,\bN$.

For any $\bv$ feasible for~\eqref{eq:maxR}, we have that
\begin{equation}
	\sum_{i=1}^m v_i = m - \sum_{i=1}^m (1-v_i)
	\le m - m \sqrt[m]{\prod_{i=1}^m (1-v_i)}
	\le m - m \sqrt[m]{\tbub_{\bk}}
	= m \left(1- \sqrt[m]{\tbub_{\bk}}\right) \label{eq:AM-GM-bound}
\end{equation}
where the first inequality derives from the celebrated AM-GM inequality $\sqrt[m]{\prod_{i=1}^m (1-v_i)} \le \frac{1}{m} \sum_{i=1}^m (1-v_i)$ (note that $(1-v_i) \ge 0$, $i=1,\ldots,m$, since $\bv \in [0,1]^m$), and the second inequality follows from~\eqref{eq:maxR_constraint}.

The point $\bv^\star = (1 - \sqrt[m]{\tbub_{\bk}}) \bone$ is feasible because $\prod_{i=1}^m (1-v^\star_i) = \prod_{i=1}^m \sqrt[m]{\tbub_{\bk}} = \tbub_{\bk}$, and moreover it attains $\sum_{i=1}^m v^\star_i = m(1 - \sqrt[m]{\tbub_{\bk}})$, which is maximal in view of~\eqref{eq:AM-GM-bound}. Hence, for all $\bk = \bzero,\dots,\bN$ we have that $\max_{\bv \in \cR(\bk)} |\bv| = m (1 - \sqrt[m]{\tbub_{\bk}})$. Using this result with~\eqref{eq:region_bound_diag_joint_ub} yields exactly~\eqref{eq:joint-bound_diag} and~\eqref{eq:eps-bound-diag}. The bound is capped at $1$ whenever $m (1 - \sqrt[m]{\tbub_{\bk}})$ exceeds $1$ because $V(\zsN) \in [0,1]$ by definition. This concludes the proof.
\qed

\subsection{Proof of Proposition~\ref{prop:apriori-joint-bound-diag}} \label{sec:proof-apriori-joint-diag}
Under Assumptions~\ref{ass:property} and~\ref{ass:non-degeneracy}, by Proposition~\ref{prop:joint-bound-diag} we have that
\begin{equation} \label{eq:joint-bound_diag_proof_apriori}
	\bbPbN\{ V(\zsN) \le \epsub(\bsstar) \} \ge 1 - \beta,
\end{equation}
where $\bsstar \in \bbN_0^m$ is the complexity of $\zsN$ and
\begin{equation} \label{eq:eps-bound-diag_proof_apriori}
	\epsub(\bk) = \min\{ m ( 1 - \sqrt[m]{\tbub_{\bk}} ), 1\},
\end{equation}
with $\tbub_{\bk}$ being either the unique zero of $\psi_{\bk,\bN,\bN}$ (when $\bk < \bN$) or 1 (otherwise).

Consider $\bk \in \bbN_o^m$ such that $\knorm := |\bk| \le \ktot$, and start by noticing that, when $\ktot < \Nlb = \min\{\bN\}$, an application of Lemmas~\ref{lem:psiN_le_psiNlb} and~\ref{lem:psi_best_worst_supp_distr}, respectively, gives
\begin{equation*}
	\psi_{\bk,\bN,\bN}(t) \le \psi_{\bk,\Nlb\bone,\Nlb\bone}(t) \le \psi_{\bkub,\Nlb\bone,\Nlb\bone}(t),
\end{equation*}
where $\bkub \in \bbN_0^m$ is as in Lemma~\ref{lem:psi_best_worst_supp_distr}, i.e., it must be $|\bkub| = \knorm$ and $\bar{k}_i \in \{K,0\}$. Then, for $t > 0$, we have
\begin{align}
	\psi_{\bk,\bN,\bN}(t)
	&\le \psi_{\bkub,\Nlb\bone,\Nlb\bone}(t) \nonumber \\
	&= 1 - \frac{\beta}{\Nlb} \sum_{j = 1}^{ \Nlb - \knorm } \frac{\binom{\Nlb-j}{\knorm}}{\binom{\Nlb}{\knorm}} t^{-j} \nonumber \\
	\displaybreak[0]
	&\le 1 - \frac{\beta}{\Nlb} \sum_{j = 1}^{ \Nlb - \ktot } \frac{\binom{\Nlb-j}{\knorm}}{\binom{\Nlb}{\knorm}} t^{-j} \nonumber \\
	&= 1 - \frac{\beta}{\Nlb} \sum_{j = 1}^{ \Nlb - \ktot } \prod_{\ell = 0}^{j-1} \frac{\Nlb-\knorm-\ell}{\Nlb-\ell} t^{-j} \nonumber \\
	&\le 1 - \frac{\beta}{\Nlb} \sum_{j = 1}^{ \Nlb - \ktot } \prod_{\ell = 0}^{j-1} \frac{\Nlb-\ktot-\ell}{\Nlb-\ell} t^{-j} \nonumber \\
	&= \psi_{\ktot,\Nlb,\Nlb}(t), \label{eq:psi_ordering_apriori}
\end{align}
where the second inequality follows from dropping non-positive terms, the third equality is due to~\eqref{eq:bin_ratio_-j_ki} with $m = 1$, the third inequality is due to $\Nlb-\knorm-\ell \ge \Nlb-\ktot-\ell$, and the last term is by definition~\eqref{eq:psi_two_sums} and coincides with~\eqref{eq:psi_single_def}.

Still considering $\bk$ such that $|\bk| = \knorm \le \ktot$, we now discuss three cases: $\ktot < \Nlb$ (which implies $\bk < \bN$), $\ktot \geq \Nlb$ and $\bk < \bN$, and $\bk \not< \bN$ (which means that $\ktot \ge \Nlb$). If $\ktot < \Nlb$ and $\bk < \bN$, then, by Lemmas~\ref{lem:psi-continous-concave-increasing} and~\ref{lem:zeros-psi}, both $\psi_{\bk,\bN,\bN}(t)$ and $\psi_{\ktot,\Nlb,\Nlb}(t)$ are non-decreasing and each admits a unique zero in $[0,1]$. In this case, $\tbub_{\bk}$ and $\tktot$ coincide with the two zeros (cf.~\eqref{eq:psi_zero_ub} with $\bH = \bN$ and~\eqref{eq:psi_single_zero_ub}, respectively). Therefore, by~\eqref{eq:psi_ordering_apriori}, it follows that $\tbub_{\bk} \ge \tktot$.
If $\ktot \geq \Nlb$ and $\bk < \bN$, then, by Lemmas~\ref{lem:psi-continous-concave-increasing} and~\ref{lem:zeros-psi}, $\psi_{\bk,\bN,\bN}(t)$ is non-decreasing and has a unique zero in $[0,1]$ corresponding to $\tbub_{\bk}$ (cf.~\eqref{eq:psi_zero_ub} with $\bH = \bN$), while $\tktot = 0$ by definition (cf.~\eqref{eq:psi_single_zero_ub}). Trivially, it holds that $\tbub_{\bk} \ge \tktot$.
Finally, if $\bk \not< \bN$ and $\ktot \ge \Nlb$, then $\tbub_{\bk} = \tktot = 0$ by definition (cf.~\eqref{eq:psi_zero_ub} with $\bH = \bN$ and~\eqref{eq:psi_single_zero_ub}), which again gives straightforwardly that $\tbub_{\bk} \ge \tktot$.

Therefore, in all cases, we have $\tbub_{\bk} \ge \tktot$, which implies $m(1-\sqrt[m]{\tbub_{\bk}}) \le m(1-\sqrt[m]{\tktot})$, and shows that $\epsub(\bk) \le \epsubap(\ktot)$ ($\epsubap(\ktot)$ being defined as in~\eqref{eq:apriori-eps-bound-diag}) for every $\bk$ such that $|\bk| \le \ktot$. Since by Assumption~\ref{assum:complexity_bounded} it always holds that $|\bsstar| \le \ktot$, we also have that $\epsub(\bsstar) \le \epsubap(\ktot)$, and using this in~\eqref{eq:joint-bound_diag_proof_apriori} immediately gives
\begin{equation*}
	\bbPbN\{ V(\zsN) \le \epsubap(\ktot) \} \ge 1 - \beta.
\end{equation*}
This concludes the proof.
\qed

\subsection{Proof of Proposition~\ref{prop:apriori-uniform-joint-bound-diag}} \label{sec:proof-apriori-joint-uniform}
Under Assumptions~\ref{ass:property}, \ref{ass:non-degeneracy}, and~\ref{assum:complexity_bounded}, we can readily invoke Proposition~\ref{prop:apriori-joint-bound-diag} to claim
\begin{equation} \label{eq:apriori-joint-bound-diag_proof}
	\bbPbN\{ V(\zsN) \le \epsubap(\ktot) \} \ge 1 - \beta,
\end{equation}
where $\epsubap(\ktot) = \min\{ m (1 - \sqrt[m]{\tktot}), 1 \}$ and $\tktot$ are defined according to~\eqref{eq:apriori-eps-and-tktot-bound-diag}.

Consider the inequality $t^x \ge 1 + x \log(t)$. Substituting $x = \frac{1}{m} > 1$ and $t = \tktot$, and rearranging some terms, yields
\begin{equation*}
	m (1 - \sqrt[m]{\tktot}) = \frac{1 - t^x}{x} \le -\log(t) = \log(\tfrac{1}{\tktot}),
\end{equation*}
which, together with~\eqref{eq:apriori-joint-bound-diag_proof} proves~\eqref{eq:apriori-uniform-joint_bound_diag}.
\qed

\section{MATLAB code} \label{appendix:matlab}
The following code implements the bisection method to compute $\tbub_{\bk}$ and $\tblb_{\bk}$ as in~\eqref{eq:psi_zeros}. \\[.25em]
\textsf{\underline{Input} :} $\bk$, $\bN$, $\bH$, $\beta$ with $\bk \le \bN \le \bH$ \\[.25em]
\textsf{\underline{Output} :} $\tbub_{\bk}$, $\tblb_{\bk}$
\clearpage
\begin{lstlisting}[basicstyle=\ttfamily\small, escapeinside={[*@}{@*]}]
function [tUk,tLk] = find_tk(k,N,H,beta)

TOL = 1e-10                           % Tolerance for bisection
k = k(:); N = N(:); H = H(:);         % Force column vectors
m = length(N);                        % Number of criteria
j_neg = 1:min(N-k);                   % Indices for terms t^-j
j_pos = 1:min(H-N);                   % Indices for terms t^j
Jtot = min(H-N)+min(N)-1;             % Jtot = |J_0|-1

% Logarithm of binomial coefficient ratios (bcr)
% for terms multiplying t^-j (cf. [*@(\ref{eq:psi_two_sums})@*] and [*@(\ref{eq:bin_ratio_-j_ki})@*])
bcr_neg = sum(cumsum( log(N-k-(j_neg-1)) - log(N-(j_neg-1)) ,2),1);

% Logarithm of binomial coefficient ratios (bcr)
% for terms multiplying t^j (cf. [*@(\ref{eq:psi_two_sums})@*] and [*@(\ref{eq:bin_ratio_+j_ki})@*])
bcr_pos = sum(cumsum( log(N+j_pos) - log(N-k+j_pos) ,2),1);

th = prod(1-k./N);                    % \hat{t} in [*@(\ref{eq:psi_zeros})@*]
psi = @(t) ...                        % psi_{k,N,H}(t) in [*@(\ref{eq:psi_two_sums})@*]
        1 - (beta/Jtot)*( ...
            sum(exp(bcr_neg - j_neg*log(t))) ...
            + sum(exp(bcr_pos + j_pos*log(t))) ...
        );

% Find tUk if k < N according to [*@(\ref{eq:psi_zero_ub})@*] via bisection
if all(k < N)
    t1 = 0;
    t2 = th;
    while t2-t1 > TOL
        tb = (t1+t2)/2;
        if psi(tb) >= 0
            t2 = tb;
        else
            t1 = tb;
        end
    end
    tUk = t1;
else
    tUk = 0;
end

% Find tLk if H > N according to [*@(\ref{eq:psi_zero_lb})@*] via bisection
if (nargout == 2) && all(H > N)
    t1 = th;
    t2 = 1;
    while t2-t1 > TOL
        tb = (t1+t2)/2;
        if psi(tb) >= 0
            t1 = tb;
        else
            t2 = tb;
        end
    end
    tLk = t2;
else
    tLk = 1;
end
\end{lstlisting}

\end{document}